\documentclass[10pt]{article}
\usepackage{balance}
\usepackage[top=2.5cm, bottom=2.9cm, left=2.5cm,
right=2.5cm]{geometry}
\usepackage{wrapfig} 
\usepackage{dblfloatfix}
\usepackage[export]{adjustbox} 
\usepackage{microtype}
\usepackage{graphicx} 
\usepackage{pgfplots}
\pgfplotsset{compat=1.18}
\usepackage{tikz}
\usetikzlibrary{calc}
\usepgfplotslibrary{groupplots}
\usepackage{algpseudocode}
\usepackage{subcaption}
\usepackage{booktabs} 
\usepackage{placeins}
\usepackage{hyperref}
\usepackage{bm}
\usepackage{enumitem,amsthm,amsmath,amsfonts,amssymb}
\usepackage{indentfirst}
\usepackage{float}
\usepackage{natbib}
\usepackage{mathtools}
\usepackage{algorithm}
\usepackage{xcolor}
\usepackage{graphicx}
\usepackage{subcaption}
\usepackage{tikz}
\usepackage{multirow}
\usepackage{multicol}
\usepackage{hyperref}
\usepackage{soul}
\usepackage{mathrsfs, fancyhdr}
\usepackage{amscd}
\usepackage{booktabs} 
 \def\polylog{\mathrm{polylog}}
\newcommand{\Z}{\mathbb{Z}}

\newcommand{\E}{\mathbb{E}}

\newcommand{\R}{\mathbb{R}}

 \newcommand{\bDelta}{\Delta}

\def\bw{\mathbf{w}}
\def\O{\mathcal{O}}

\def\X{\mathcal{X}}
\def\Y{\mathcal{Y}}

\def\M{\mathcal{M}}
\def\bA{\mathbf{A}}
\def\Z{\mathcal{Z}}

\def\L{\mathcal{L}}
\def\bx{\mathbf{x}}

\def\bw{\mathbf{w}}
\def\bW{\mathbf{W}}
\def\bu{\mathbf{u}}

\def\A{\mathcal{A}}

\def\bh{\mathbf{h}}

\def\ba{\mathbf{a}}

\def\bU{\mathbf{U}}

\def\bu{\mathbf{u}}

\def\0{\mathbf{0}}

\def\bv{\mathbf{v}}

\def\bv{\mathbf{v}}

\def\proj{\text{Proj}}
\def\bB{\mathbf{B}}

\def\bw{\mathbf{w}}

\def\O{\mathcal{O}}

\def\M{\mathcal{M}}
\def\A{\mathcal{A}}
\def\Z{\mathcal{Z}}
\def\R{\mathbb{R}}

\def\X{\mathcal{X}}

\def\Y{\mathcal{Y}}
\def\Z{\mathcal{Z}}

\def\bD{\mathbf{D}}
\def\bV{\mathbf{V}}
\def\bfI{\mathbf{I}}

\def\K{\mathcal{K}}

\def\bc{\mathbf{c}}

\usepackage{amsmath}
\usepackage{amssymb}
\usepackage{mathtools}
\usepackage{amsthm}
\usepackage{multirow}

\usepackage[capitalize,noabbrev]{cleveref}

\theoremstyle{plain}
\newtheorem{theorem}{Theorem}[section]

\newtheorem{lemma}[theorem]{Lemma}
\newtheorem{corollary}[theorem]{Corollary}
\theoremstyle{definition}
\newtheorem{definition}[theorem]{Definition}
\newtheorem{assumption}[theorem]{Assumption}
\theoremstyle{remark}
\newtheorem{remark}[theorem]{Remark}

\usepackage[textsize=tiny]{todonotes}

\begin{document}

\title{Optimization Risk Bounds for Kolmogorov–Arnold Networks Trained by DP-SGD with  Correlated Noise}
\author{Puyu Wang$^1$ \quad Jan Schuchardt$^2$ \quad Nikita Kalinin$^3$ \quad 
Marcio Monteiro$^1$ \quad Junyu Zhou$^4$\\ \quad  Sophie Fellenz$^1$ \quad Christoph Lampert$^3$ \quad Marius Kloft$^1$ \\ 
\smallskip \\
$^{1}$ RPTU Kaiserslautern-Landau, Kaiserslautern, Germany\\
$^{2}$ Machine Learning Research, Morgan Stanley\\
$^{3}$  Institute of Science and Technology, Klosterneuburg, Austria\\
$^{4}$ Catholic University of Eichstätt-Ingolstadt, Ingolstadt, Germany}

\maketitle

\begin{abstract}
The theoretical understanding of differentially private stochastic gradient descent (DP-SGD) with temporally correlated noise remains limited, particularly for non-convex neural network training. As a first step, we study two-layer Kolmogorov–Arnold Networks (KANs), a recently introduced architecture with learnable spline-based edge functions. We establish the first optimization risk bounds for clipped mini-batch DP-SGD with correlated noise in this setting, with explicit dependence on temporal correlation, clipping, mini-batch sampling, and network width. Existing arguments fail for three reasons: temporal dependence breaks the conditional-centering step; projection obstructs the cross-iteration cancellation of correlated perturbations; and active clipping breaks the empirical-gradient structure. We address these issues through shifted and auxiliary dynamics, a weighted empirical loss, and a high-probability localization argument. Our bound shows that temporal correlation reduces the leading noise terms, the clipping threshold enters essentially through an effective step size, and private training admits an explicit network width range. Experiments on synthetic data and MNIST support the predicted optimization effect of temporal correlation. As an application, we derive population risk guarantees via algorithmic stability. Our framework recovers non-private mini-batch SGD, independent-noise DP-SGD, and their full-batch counterparts as special cases. 
\end{abstract}

 \section{Introduction}

Differentially private stochastic gradient descent (DP-SGD) \citep{song2013stochastic} is the standard approach to training machine learning models with formal privacy guarantees \citep{dwork2006differential}. Standard DP-SGD adds fresh Gaussian noise at every iteration, whose accumulation can substantially degrade optimization and predictive performance. Recent correlated noise mechanisms instead introduce temporal correlations across perturbations, allowing consecutive noise terms to partially cancel and improve privacy--utility tradeoffs. Recent studies have demonstrated substantial empirical gains over independent noise training \citep{choquette2023correlated,choquette-choo2023amplified,kalinin2026dp}.

Despite this empirical progress, optimization theory for correlated noise private training remains limited, particularly for neural networks. \cite{koloskova2023gradient} provide stationarity guarantees for smooth non-convex optimization with linearly correlated noise, but do not address optimization risk for clipped mini-batch neural network training with local network geometry. Optimization risk guarantees in this regime therefore remain largely unexplored.

We investigate this problem through Kolmogorov-Arnold Networks (KANs) \citep{liu2025kan}, a structured alternative to multilayer perceptrons that uses learnable univariate functions on edges. KANs have shown promising performance in scientific applications \citep{cherednichenko2025kolmogorov,li2025kolmogorov,patra2025physics,shukla2024comprehensive,wang2025kolmogorov,vaca2024kolmogorov}. The closest KAN theory, \cite{wang2026optimization}, analyzes full-batch GD and its independent noise DP counterpart. Extending this setting to clipped mini-batch training with temporally correlated noise introduces new analytical difficulties from temporal dependence, localization, and clipping. Our analysis provides a first step toward understanding correlated noise DP training for non-convex neural networks more broadly.

\begin{wrapfigure}{l}{0.42\textwidth}
    \centering
    \vspace{-1mm}
    \includegraphics[width=0.95\linewidth]{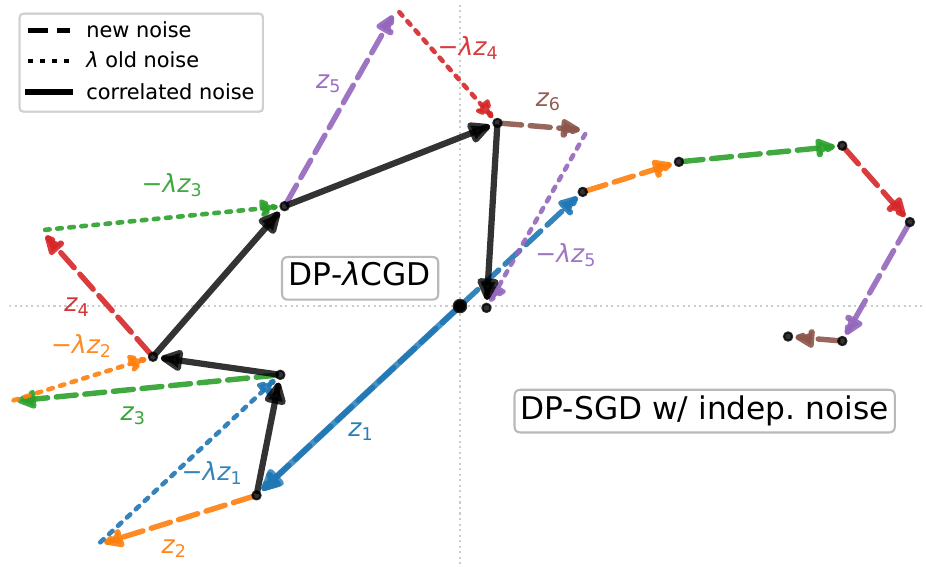}
    \caption{DP-$\lambda$CGD subtracts a $\lambda$-fraction of the previous noise, inducing partial cancellation across consecutive iterations. The figure is reproduced from \cite{kalinin2026dp} with the authors' permission.}
    \label{fig:correlated-noise}
    \vspace{-4mm}
\end{wrapfigure}
Concretely, for two-layer KANs, we consider clipped mini-batch DP-$\lambda$CGD \citep{kalinin2026dp} with correlated perturbations $\xi_t=\kappa(Z_t-\lambda Z_{t-1})$, where $Z_t$ is standard Gaussian noise, $\kappa\ge0$ is the noise multiplier, and $\lambda=0$ recovers standard DP-SGD. Our main focus is optimization risk: we characterize how temporal correlation, clipping, mini-batch sampling, and network width enter the optimization bound, and use this result as the input to a stability-based population risk analysis.

This analysis poses three main technical challenges: \emph{(i)} temporal dependence breaks the conditional-centering arguments underlying standard stochastic optimization analyses; \emph{(ii)} the projection used for localization obstructs the cross-iteration cancellation structure of correlated noise (Figure~\ref{fig:correlated-noise}); and \emph{(iii)} active clipping makes the averaged update no longer the gradient of the empirical loss, preventing direct application of the local comparator argument. Consequently, the analysis does not follow from a direct combination of the full-batch KAN argument of \cite{wang2026optimization} and existing correlated noise optimization analyses such as \cite{koloskova2023gradient}. 

Our main contributions are summarized as follows.

\begin{itemize}[leftmargin=1.2em,itemsep=2pt,topsep=1pt,parsep=0pt]

\item We provide the \textit{first} optimization risk analysis of clipped mini-batch DP-SGD with temporally correlated noise for non-convex neural networks, taking two-layer KANs as a concrete setting. The analysis combines shifted and auxiliary dynamics with a high-probability localization bootstrap to exploit cross-iteration cancellation while certifying projection inactivity. Active clipping is handled through a weighted empirical loss with frozen clipping coefficients, which restores the gradient structure needed for the local comparator argument.

\item Our optimization bound explicitly characterizes how the \textit{clipping threshold}, \textit{temporal correlation}, \textit{mini-batch sampling}, and \textit{network width} affect optimization risk. In the active clipping regime, the clipping threshold enters essentially through an effective step size, while temporal correlation reduces the leading noise terms through explicit correlation dependent factors. The result also identifies an admissible width range for private KAN training. Experiments on synthetic data and MNIST at fixed noise scales support the predicted optimization effect of temporal correlation.

\item As an application, we combine the optimization bound with algorithmic stability to obtain population risk guarantees. We further study how privacy calibration interacts with the optimization effect of temporal correlation under representative privacy regimes.

\item Our framework recovers non-private mini-batch SGD, independent noise DP-SGD, and the corresponding full-batch GD and DP-GD regimes as special cases. In the full-batch setting, our bounds match the known dependence on $n$, $d$, and $\epsilon$ in \cite{wang2026optimization}.

\end{itemize}

\section{Related Work}\label{sec:relatedwork}
KAN theory has focused on approximation, expressiveness, and optimization \citep{eshtehardianconvergence,gao2025convergence,li2025generalization,liu2025rate,wang2025expressiveness}. The closest prior work \cite{wang2026optimization} studies optimization and population risk for two-layer KANs under full-batch GD and independent-noise DP-GD. Our setting adds mini-batch stochasticity, clipping, and temporally correlated noise, while recovering these regimes as special cases. 

\paragraph{Theory of private neural network (NN) training.} Recent work studies private NN training with independent noise \cite{ding2025understanding,shi2026towards,wang2025optimaldp,xu2026differential,zhang2026understanding}. \cite{wang2025optimaldp} analyzes DP-GD for three-layer MLPs in regression, while \cite{shi2026towards} studies DP-GD for two-layer CNNs. \cite{ding2025understanding,xu2026differential,zhang2026understanding} study DP-SGD for NNs from a feature-learning perspective, including noisy feature dynamics, fairness and robustness degradation, and long-tailed memorization. These works do not cover KANs or temporally correlated noise mechanisms.

\paragraph{Theory of correlated noise DP.} Correlated noise mechanisms have been analyzed from several perspectives \cite{denisov2022improved,koloskova2023gradient,choquette2023correlated}. The work \cite{koloskova2023gradient} studies GD with linearly correlated noise and gives stationarity guarantees for smooth non-convex objectives, while \cite{choquette2023correlated} establishes utility separations between correlated and independent noise for private convex learning. Related work develops privacy accounting and amplification for subsampled or structured matrix mechanisms \cite{choquette-choo2023amplified,choquette2024near,choquette2023privacy,schuchardt2026sampling}. These works establish privacy guarantees or optimization results in different settings, but do not address optimization risk for clipped mini-batch DP-SGD in KANs, where the analysis must additionally control local curvature and clipping. 

A broader overview of related work, including neural network theory and privacy amplification by subsampling, is provided in Appendix~\ref{sec:appen-relatedwork}.

\section{Problem Setting}\label{section:preliminaries}
We now introduce the learning problem, the two-layer KAN architecture, the mini-batch DP-SGD algorithm with correlated noise, and the assumptions underlying our analysis.

\textbf{Notation and learning problem.}
Let $\mathcal{P}$ be a probability distribution on $\mathcal{X}\times\mathcal{Y}$, where $\mathcal{X}\subseteq\{\bx\in\R^d:\|\bx\|_2\le 1\}$ and $\mathcal{Y}=\{-1,+1\}$. For a positive integer $q$, let $[q] =\{1,\ldots,q\}$. We use $\|\cdot\|_2$ for the Euclidean norm, and $\langle\cdot,\cdot\rangle$ for the inner product. Given a training dataset $S=\{(\bx_i,y_i)\}_{i=1}^n$ drawn i.i.d.\ from $\mathcal{P}$, we measure the quality of a classifier $f:\mathcal{X}\to\R$ by the population and empirical risks
$\mathcal{L}(f) = \mathbb{E}_{(\bx,y)\sim\mathcal{P}}\!\left[\ell(yf(\bx))\right]$ and 
$\mathcal{L}_S(f) = \frac{1}{n}\sum_{i=1}^n \ell(y_if(\bx_i))$, respectively, 
where $\ell(z)=\log(1+\exp(-z))$ is the logistic loss.

\subsection{Architecture: Two-layer KANs with B-spline Bases}
Let \(m\) be the hidden width. Following the spline-based two-layer KAN formulation studied in \cite{gao2025convergence,wang2026optimization}, we consider a model with B-spline basis \(\{b_k\}_{k=1}^p\) and a fixed activation function \(\sigma:\mathbb{R}\to\mathbb{R}\). For an input \(\bx=(x_1,\ldots,x_d)\), the network is defined as
\vspace{-1mm}
\[
f_{\bW}(\bx) = \frac{1}{\sqrt m} \sum_{j=1}^m \sum_{k=1}^p c_{j,k}\, b_k(x_{1,j}) \quad \text{with} \quad x_{1,j} = \sigma \Big( \frac{1}{\sqrt d} \sum_{i=1}^d \sum_{k=1}^p w_{i,j,k}\, b_k(x_i) \Big),
\]
where \[\bW=\{w_{i,j,k}\}_{i\in[d],\,j\in[m],\,k\in[p]}\in \R^{mdp}\] denotes the trainable first-layer spline coefficients, \[\bc=\{c_{j,k}\}_{j\in[m],\,k\in[p]}\in\R^{mp}\] denotes the second-layer spline coefficients, and $x_{1,j}$ is the output of the $j$-th hidden unit. Following the first-layer training formulation of \cite{eshtehardianconvergence}, we draw $\bc$ once at initialization from $\mathcal{N}(\mathbf 0,\mathbf I_{mp})$ and keep it fixed throughout training, while optimizing only $\bW$. We further discuss this modeling choice in Section~\ref{sec:discussion}. For notational simplification, we write $\mathcal{L}_S(\bW)=\mathcal{L}_S(f_{\bW})$ and $\mathcal{L}(\bW)=\mathcal{L}(f_{\bW})$. 

\subsection{Algorithm: Mini-batch DP-SGD with Correlated Noise}
To enable training KANs on sensitive datasets, we study a differentially private variant of mini-batch SGD. We call two datasets neighboring if they differ in the contribution of one data record.
\begin{definition}[Differential privacy \citep{dwork2006calibrating}]\label{def:DP}
A randomized algorithm $\A$ satisfies $(\epsilon,\delta)$-DP if, for any two neighboring datasets $S$ and $S'$ and any measurable set $E$ in the output space of $\A$, it holds that $\mathbb{P}(\A(S)\in E)\le e^\epsilon \mathbb{P}(\A(S')\in E)+\delta$. We say $\A$ satisfies $\epsilon$-DP if $\delta=0$.
\end{definition}

The algorithm we analyze is mini-batch DP-SGD with correlated noise. It also subsumes independent noise DP-SGD and non-private mini-batch SGD as special cases. We use DP-$\lambda$CGD \citep{kalinin2026dp}, a simple correlated noise mechanism, as the noise model. After initializing $\bW_0\sim\mathcal{N}(\mathbf 0,\mathbf I_{mdp})$, at each iteration $t$ we draw a mini-batch $\mathcal B_t\subseteq[n]$ of size $B$ uniformly without replacement and form $g_{t,i} =\nabla \ell(y_i f_{\bW_{t-1}}(\bx_i))$, which are clipped at threshold $C_{\mathrm{clip}}$ and averaged: 
$\tilde g_{t,i} = g_{t,i}\min\big\{1,\frac{C_{\mathrm{clip}}}{\|g_{t,i}\|_2}\big\}$ and $v_t = \frac{1}{B}\sum_{i\in\mathcal B_t}\tilde g_{t,i}$.
The noise takes the correlated form $\xi_t = \kappa(Z_t-\lambda Z_{t-1})$ with $Z_0=\mathbf 0$ and $\{Z_t\}_{t\ge 1}\overset{\mathrm{iid}}{\sim}\mathcal{N}(\mathbf 0,\mathbf I_{mdp})$, where $\kappa\ge 0$ is the noise multiplier and $\lambda\in[0,1)$ controls the strength of temporal correlation. The iterate is then updated by a projected step onto $\mathcal K=\mathcal B(\bW_0,R_*)$: $\bW_t = \Pi_{\mathcal K}(\bW_{t-1}-\eta\, \hat v_t)$ with $\hat v_t = v_t + \frac{C_{\mathrm{clip}}}{B}\xi_t$. The full procedure is given in Algorithm~\ref{alg:dp-lambda-minibatch}. We empirically evaluate it on two-layer KANs in Section~\ref{sec:main_opt}, with additional fixed-privacy predictive results provided in Appendix~\ref{app:experiments}. 

\begin{algorithm}[h]
\caption{Mini-batch DP-SGD  with Correlated Gaussian Noise} 
\label{alg:dp-lambda-minibatch}
\begin{algorithmic}[1]
\Require Dataset $S=\{(\bx_i,y_i)\}_{i=1}^n$; total iterations $T$; step size $\eta>0$; batch size $B$; clip threshold $C_{\mathrm{clip}}>0$; correlation $\lambda\in[0,1)$; noise multiplier $\kappa\ge 0$; localization radius $R_*>0$.

\State Initialize $\bW_0\!\sim\!\mathcal{N}(\mathbf 0,\mathbf I_{mdp})$ and $\bc\!\sim\!\mathcal{N}(\mathbf 0,\mathbf I_{mp})$; keep $\bc$ fixed; set $Z_0\!=\!\mathbf 0\in\mathbb R^{mdp}$ and $\mathcal K=\mathcal B(\bW_0,R_*)$.
\For{$t=1,2,\dots,T$}
    \State Sample a mini-batch $\mathcal{B}_t \subseteq [n]$ of size $B$ uniformly without replacement
    \For{each $i\in \mathcal{B}_t$}
        \State $\tilde g_{t,i} \gets g_{t,i}\cdot \min\big\{1,\frac{C_{\mathrm{clip}}}{\|g_{t,i}\|_2}\big\}$ with $g_{t,i} = \nabla \ell(y_i f_{\bW_{t-1}}(\bx_i))$ \Comment{per-example clipping}
    \EndFor
    \State $v_t \gets \frac{1}{B}\sum_{i\in\mathcal{B}_t}\tilde g_{t,i}$ \Comment{clipped mini-batch gradient}
    \State Sample fresh noise $Z_t\sim \mathcal N(\mathbf 0,\mathbf I_{mdp})$
    \State $\xi_t \gets \kappa\,(Z_t-\lambda Z_{t-1})$ \Comment{correlated Gaussian noise}
    \State $\hat v_t \gets v_t + \frac{C_{\mathrm{clip}}}{B}\,\xi_t$
    \State $\bW_t \gets \Pi_{\mathcal K}(\bW_{t-1}-\eta\,\hat v_t)$
\EndFor
\State \Return $\{\bW_t\}_{t=0}^{T}$.
\end{algorithmic}
\end{algorithm}

\paragraph{Special cases.} This formulation recovers the standard baselines as special cases. Setting $\lambda=0$ gives $\xi_t=\kappa Z_t$, recovering \emph{independent-noise DP-SGD}. Setting $\kappa=0$ removes the noise and yields \emph{non-private mini-batch SGD}; taking in addition $C_{\mathrm{clip}}\to\infty$ and $R_*\to\infty$ removes the clipping and projection, recovering plain \emph{mini-batch SGD}.

We impose two standard assumptions used in prior KAN analyses \citep{gao2025convergence,taheri2024sharper,wang2026optimization}. The first bounds the activation and B-spline basis. It is satisfied, e.g., by cubic (or higher-degree) B-splines together with sigmoid or hyperbolic tangent activations. Let $\mathcal I_b$ be an interval containing $[-1,1]\cup\mathrm{range}(\sigma)$.
\begin{assumption}\label{ass:sigma}
Assume $\sigma$ satisfies $|\sigma(u)|\le B_\sigma$, $|\sigma'(u)|\le B'_\sigma$, and $|\sigma''(u)|\le B''_\sigma$ for all $u\in\R$. Further, assume $\{b_k\}_{k=1}^p$ satisfy $|b_k(v)|\le B_b$, $|b_k'(v)|\le B'_b$, and $|b_k''(v)|\le B''_b$ for all $v\in\mathcal I_b$.
\end{assumption}
The second is a margin-style separability condition on the NTK features at initialization \citep{lei2026optimization,taheri2024sharper,wang2026optimization}. It is weaker than the NTK Gram-matrix positive-definiteness commonly used in the literature \citep{arora2019fine,gao2025convergence,nitanda2019gradient}, and  was empirically validated for KANs in \cite{wang2026optimization}.
\begin{assumption}[NTK separability]\label{ass:ntk}
There exists $\gamma\in(0,1]$ and $\bu\in\R^{mdp}$ with $\|\bu\|_2=1$ such that $y_i \langle\nabla f_{\bW_0}(\bx_i),\bu\rangle\geq\gamma$ for all $i\in[n]$.
\end{assumption}

\section{Optimization Risk Bound for DP-SGD with Correlated Noise}\label{sec:main_opt}
This section develops our optimization risk bound for DP-SGD with correlated noise, which forms the technical core of the paper. We first identify the main analytical challenges, then present and interpret the optimization theorem, empirically illustrate its correlation-dependent behavior, and finally give a proof sketch showing how the main difficulties are resolved.

\subsection{Technical Challenges}
To control the optimization risk $\frac1T\sum_{t=0}^{T-1} \L_S(\bW_t)$, a standard approach is to use a comparator argument based on the local curvature of the empirical loss. In particular, for a suitable comparator $\bW^*\in\mathcal K$, one expects an inequality of the form
\[
\L_S(\bW_{t-1})-\L_S(\bW^*)
\lesssim
\big\langle \bar v_t,\bW_{t-1}-\bW^*\big\rangle
+\mathrm{Err}_t,
\]
where $\bar v_t =\frac1n\sum_{i=1}^n \tilde g_{t,i}$ is the full-batch clipped gradient used only in the analysis, and $\mathrm{Err}_t$ contains the local curvature error terms. Thus, bounding the optimization risk reduces to controlling the accumulated first-order term $\sum_t \langle \bar v_t,\bW_{t-1}-\bW^* \rangle$.
This reduction leads to three distinct difficulties: the correlated perturbations break the conditional-centering argument used in the independent noise analysis; the curvature-based comparator inequality needed above is only valid locally for KANs; and under active clipping, $\bar v_t$ is no longer the gradient of the empirical loss. 

A natural first attempt is to follow the independent noise proof, which applies the projected one-step recursion and uses conditional centering to control the perturbation terms. To see the key step, substitute the decomposition of the noisy averaged clipped gradient $\hat v_t$ into the first order term:
\[
\langle \hat v_t,\bW_{t-1}-\bW^*\rangle
=
\langle \bar v_t,\bW_{t-1}-\bW^*\rangle
+\big\langle v_t-\bar v_t,\bW_{t-1}-\bW^*\big\rangle
+\frac{C_{\mathrm{clip}}}{B}
\big\langle \xi_t,\bW_{t-1}-\bW^*\big\rangle .
\]
In the independent noise case, the mini-batch fluctuation $v_t-\bar v_t$ and the Gaussian perturbation $\xi_t=\kappa Z_t$ are both conditionally centered given the past. Hence, after taking conditional expectations in the one-step recursion, the two perturbation terms vanish and only the deterministic descent term remains, while the perturbations contribute only through higher order variance terms. 
For correlated noise, however, this cancellation fails. Since $\xi_t=\kappa(Z_t-\lambda Z_{t-1})$, we have $\mathbb E[\xi_t\mid\mathcal F_{t-1}] = -\kappa\lambda Z_{t-1}$ with \(\mathcal F_t\) the sigma-algebra containing all randomness revealed up to the end of iteration \(t\) .
Hence, ignoring the scaling factor $\frac{C_{\mathrm{clip}}}{B}$, the noise component in the first-order term satisfies
 \[
\mathbb E\big[
\big\langle \xi_t,\bW_{t-1}-\bW^* \big\rangle
\, |\,\mathcal F_{t-1}
\big]
=
-\kappa\lambda
\big\langle Z_{t-1},\bW_{t-1}-\bW^* \big\rangle .
\]
Thus, instead of vanishing after conditioning, the perturbation leaves a first order drift term. This term cannot be removed by conditional centering, because $\bW_{t-1}$ already depends on $Z_{t-1}$. Moreover, the projection prevents us from directly exploiting the cross-iteration cancellation structure of the correlated perturbations. Hence the standard projected one-step argument no longer applies.

The second difficulty comes from the local KAN geometry. The empirical loss is non-convex in the first-layer spline coefficients, and the comparator inequality that converts the first order term into empirical loss is valid only when the iterate and the comparator stay in a localized region around the initialization. Therefore, even after the correlated noise drift is handled through the shifted dynamics, the proof must still certify that the auxiliary trajectory remains in the region where the local KAN curvature and self-boundedness estimates can be applied.

The third difficulty comes from active gradient clipping. When clipping is inactive, $\bar v_t$ coincides with the gradient of the empirical loss, so the local comparator argument can be applied directly. Under active clipping, however, the coefficients multiplying the per-example gradients depend on the current iterate, and $\bar v_t$ is no longer the gradient of $\L_S$. Consequently, the comparator inequality above cannot be applied directly to the empirical loss, and a suitable gradient structure must be recovered while explicitly accounting for the clipping threshold.

\subsection{Main Theorem} 
We now state our main optimization theorem, which addresses the difficulties described above. The theorem characterizes how mini-batch sampling, clipping, temporal correlation, and network width affect the optimization risk, and identifies an admissible width range. Let $C_{\sigma,b}>0$ depend only on $\sigma,\,b$, and let $G_\delta =B_\sigma'B_bB_b'p\big(4\sqrt p+2\sqrt{\log(2/\delta)/m}\big)$ be a uniform per-example gradient bound holding with probability at least $1-\delta$. Define $\rho_{\rm clip}=\min\big\{1,\frac{C_{\rm clip}}{G_\delta}\big\}$ and $c_{\rm eff}=\frac{C_{\rm clip}\kappa}{B\rho_{\rm clip}}$. Moreover, $\widetilde{\O}$ and $\widetilde{\Omega}$ suppress polylogarithmic factors in $m,n,T,\delta^{-1},\rho_{\rm clip}^{-1}$.
\begin{theorem}\label{thm:lambda-positive-ntk-opt} 
Let $\delta\in(0,1)$. Suppose Assumptions~\ref{ass:sigma} and~\ref{ass:ntk} hold.
Let $\{\bW_t\}_{t=0}^{T-1}$ be generated by
Algorithm~\ref{alg:dp-lambda-minibatch} with $\eta_{\rm eff} \le \frac{1}{12C_{\sigma,b}p^3}$ with $\eta_{\rm eff} =\rho_{\rm clip}\eta$ and $C_{\rm clip}>0$. 
Assume $\eta_{\rm eff}\gamma\sqrt{T}
\big( \frac{1}{\sqrt B} + (1-\lambda)c_{\rm eff} \big)
+\eta_{\rm eff}^2\gamma^2
\big( \frac{T}{B}+1 \big)
\lesssim 1$ 
and  $\widetilde{\Omega}(\gamma^{-4})
\le m \le
\widetilde{\O} \big(
(\eta_{\rm eff}^2c_{\rm eff}^2\gamma^2d)^{-1}
\min \big\{ 1,
\big( T((1-\lambda)^2+\lambda^2\eta_{\rm eff}) \big)^{-1}\big\}\big)$. 
With probability at least $1-\delta$ over the initialization
and the algorithmic randomness,
\begin{align*}
\frac1T\sum_{t=0}^{T-1}\mathcal L_S(\bW_t)
\lesssim
\mathrm{polylog}\big(\frac{nT}{\delta\rho_{\rm clip}}\big)
\cdot
\Big[\;& \frac{1}{\gamma^2\eta_{\rm eff} T}
+ \frac{1}{\gamma\sqrt{BT}}
+ \eta_{\rm eff}\big( \frac1B+\frac1T \big)+ \frac{(1-\lambda)c_{\rm eff}}{\gamma \sqrt T}
 \\
& 
+ \left( (1-\lambda)^2+\lambda^2\eta_{\rm eff} \right)
\eta_{\rm eff} c_{\rm eff}^2md \Big]
=:A_{\rm corr}.
\end{align*} 
\end{theorem}
\paragraph{Interpretation.}
The structure of $A_{\rm corr}$ separates the different sources of error. The first three terms arise from optimization and mini-batch sampling, whereas the last two capture the effect of the correlated Gaussian perturbations. In particular, temporal correlation enters these noise-dependent terms through the factors $1-\lambda$ and $(1-\lambda)^2+\lambda^2\eta_{\rm eff}$, respectively, making the effect of temporal correlation explicit in the bound. 

The admissible width range also reflects a tradeoff: $m$ must be large enough for the local KAN curvature argument to apply, but small enough that the private noise does not break localization.

\begin{figure*}[t]
    \centering
    \begin{subfigure}[t]{0.495\linewidth}
        \centering
        \includegraphics[width=0.49\linewidth]{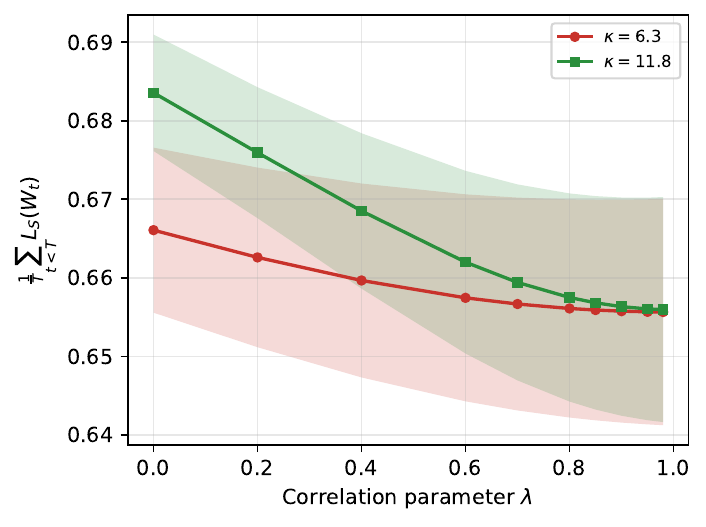}
        \hfill
        \includegraphics[width=0.49\linewidth]{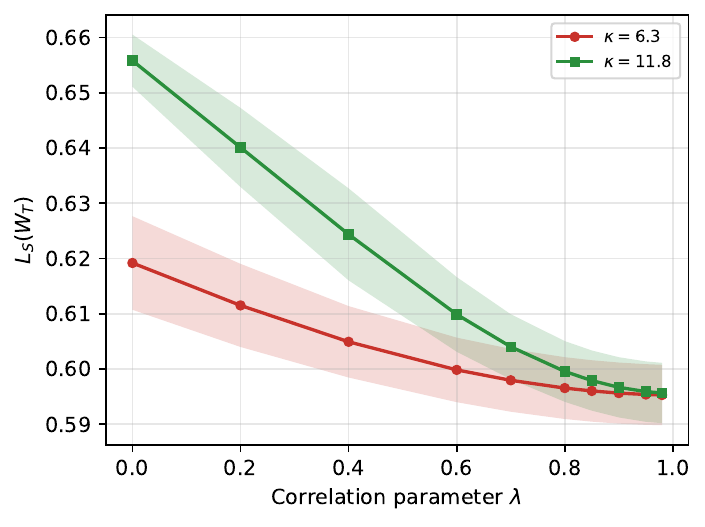}
        \caption{Synthetic logistic data.}
    \end{subfigure}
    \hfill
    \begin{subfigure}[t]{0.495\linewidth}
        \centering
        \includegraphics[width=0.49\linewidth]{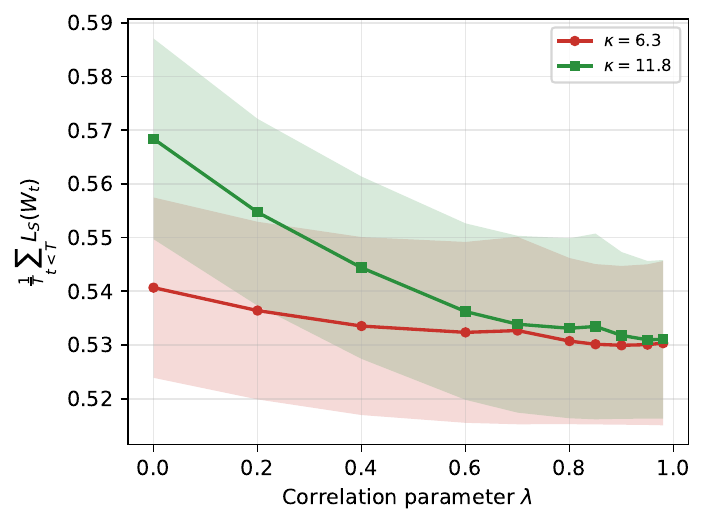}
        \hfill
        \includegraphics[width=0.49\linewidth]{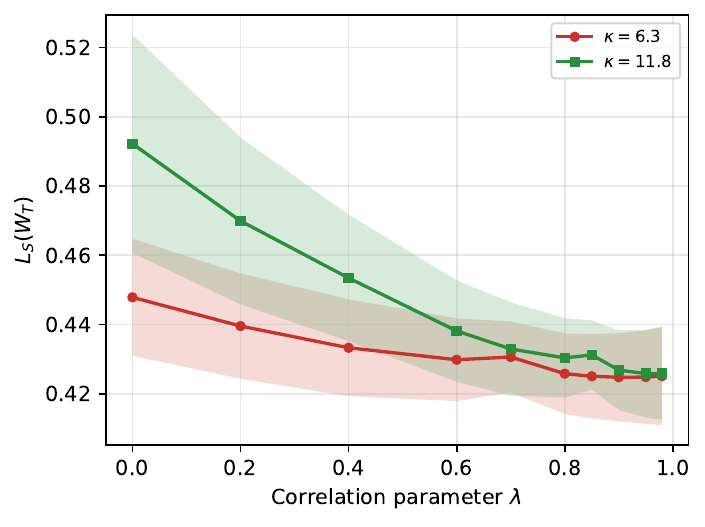}
        \caption{MNIST.}
    \end{subfigure}
    \vspace{-2mm}
    \caption{Optimization effect of temporal correlation at fixed noise multipliers. For each dataset, the left panel reports the trajectory-averaged empirical loss $\frac{1}{T}\sum_{t=0}^{T-1}\mathcal L_S(\bW_t)$, while the right panel reports the final iterate empirical loss $\L_S(\bW_T)$. We hold the noise multiplier $\kappa$ fixed and vary only the correlation parameter $\lambda$, thereby isolating the optimization-side effect of temporal correlation. Curves show the mean over five random seeds, with shaded regions indicating one standard deviation.}
    \label{fig:fixed-kappa-eta1}
    \vspace{-2mm}
\end{figure*}

\begin{remark}[Effect of the clipping threshold $C_{\rm clip}$]\label{rmk:clipping}
The dependence of the optimization bound on $C_{\rm clip}$ is captured by the effective step size $\eta_{\rm eff}=\rho_{\rm clip}\eta$ and the effective noise scale $c_{\rm eff}=C_{\rm clip}\kappa/(B\rho_{\rm clip})$. When $C_{\rm clip}\le G_\delta$, we have $\rho_{\rm clip}=C_{\rm clip}/G_\delta$ and $c_{\rm eff}=G_\delta\kappa/B$, so $c_{\rm eff}$ is independent of the clipping threshold. Consequently, apart from logarithmic factors in $\rho_{\rm clip}^{-1}$, the clipping threshold enters Theorem~\ref{thm:lambda-positive-ntk-opt} only through $\eta_{\rm eff}$. In particular, stronger clipping can be compensated by increasing the nominal step size $\eta$ to keep $\eta_{\rm eff}$ fixed, in which case the optimization guarantee is unchanged up to logarithmic factors. When $C_{\rm clip}\ge G_\delta$, we have $\rho_{\rm clip}=1$ and $\eta_{\rm eff}=\eta$, while $c_{\rm eff}=C_{\rm clip}\kappa/B$. The clipping operator is then inactive on the high-probability event underlying $G_\delta$, and increasing $C_{\rm clip}$ further provides no optimization benefit while enlarging the noise-dependent terms and tightening the admissible upper bound on the width. Thus, $G_\delta$ naturally separates the active clipping regime from the regime in which increasing the clipping threshold only increases the injected noise.
\end{remark}

\begin{remark}[Role of $\lambda$]\label{rmk:role}
Theorem~\ref{thm:lambda-positive-ntk-opt} shows that, for fixed $\kappa$ and $C_{\rm clip}$, the linear noise-fluctuation term scales proportionally to $1-\lambda$, while the quadratic Gaussian term is governed by $(1-\lambda)^2+\lambda^2\eta_{\rm eff}$. Thus, for small $\eta_{\rm eff}$, increasing $\lambda$ substantially reduces the effect of the correlated perturbations, with the quadratic factor minimized at $\lambda=1/(1+\eta_{\rm eff})$, which is close to $1$. This describes the optimization effect of temporal correlation for a fixed noise multiplier $\kappa$. Under DP, however, the required value of $\kappa$ itself depends on $\lambda$, so the resulting privacy--optimization tradeoff depends on the privacy calibration. We illustrate this interaction in Section~\ref{sec:rates} under two representative settings.
\end{remark}

\paragraph{Empirical Illustration.}
To isolate the optimization effect of temporal correlation from privacy calibration, we fix the noise multiplier $\kappa$ and vary only the correlation parameter $\lambda$. We consider synthetic logistic data generated following the setup of \cite{wang2026optimization} and MNIST 3 vs.\ 8, using a two-layer KAN with $m=32$ and $p=8$ cubic B-splines. We set $n=5{,}000$, $B=125$, $T=400$, $C_{\rm clip}=0.5$, and $\eta=1$, and consider two fixed noise multipliers, $\kappa=6.3$ and $\kappa=11.8$. Thus, these experiments focus on the optimization-side effect of temporal correlation rather than comparisons under a common privacy budget.

Figure~\ref{fig:fixed-kappa-eta1} reports both the trajectory-averaged empirical loss $\frac{1}{T}\sum_{t=0}^{T-1}\mathcal L_S(\bW_t)$, which is the quantity controlled by Theorem~\ref{thm:lambda-positive-ntk-opt}, and the empirical loss of the final iterate. On both datasets and for both noise multipliers, the averaged training loss decreases consistently as $\lambda$ increases and levels off under strong correlation. The final-iterate loss exhibits the same qualitative behavior. These results are consistent with the reduction of the correlation-dependent noise terms identified in Theorem~\ref{thm:lambda-positive-ntk-opt}. Additional experimental details and further results are provided in Appendix~\ref{app:experiments}.

\subsection{Proof Sketch of Theorem~\ref{thm:lambda-positive-ntk-opt}}\label{sec:proofsketch} 
We outline the main steps and defer the full proof to Appendix~\ref{appen:lambda1-opt}. The proof uses an auxiliary unprojected trajectory and a shifted iterate to expose the cross-iteration cancellation of correlated noise, and then combines a localization bootstrap with a weighted empirical loss to make this cancellation compatible with the local KAN geometry and active clipping. The proof proceeds in six steps.

\begin{proof}[Proof Sketch]
\textsc{Step 1: Reduce to an auxiliary unprojected trajectory.}
The correlated perturbations break conditional centering, while the projection obstructs their cross-iteration cancellation. We therefore introduce an auxiliary unprojected trajectory $\{\widetilde{\bW}_t\}_{t=0}^{T}$ driven by the same mini-batches and Gaussian variables:
\[
\widetilde{\bW}_t=\widetilde{\bW}_{t-1}-\eta\big(\bar g_t+\widetilde{\bDelta}_t+c_{\rm priv}(Z_t-\lambda Z_{t-1})\big),\qquad c_{\rm priv}=C_{\rm clip}\kappa/B,
\]
where $\bar g_t$ is the full-batch average of the clipped gradients and $\widetilde{\bDelta}_t$ is the centered mini-batch fluctuation. We first analyze this auxiliary process and later show that it remains inside the projection region, so that it coincides with the original iterates.

 \

\textsc{Step 2: Introduce the shifted iterate.}
Applying a comparator recursion directly to the auxiliary trajectory still fails because the correlated perturbation is not conditionally centered, leaving a first order drift term.

To handle this obstruction, we define the shifted iterate $\bU_t=\widetilde{\bW}_t+\eta c_{\rm priv}Z_t$. 
Then the auxiliary update admits the exact reformulation $\bU_t = \bU_{t-1} - \eta\big( \bar g_t + \widetilde{\bDelta}_t + (1-\lambda)c_{\rm priv}Z_{t-1} \big)$. 

This identity is the key cancellation step. The leading correlated perturbation is reduced to a residual term proportional to \(1-\lambda\), which is the source of the improved noise dependence in the final bound.

\

\textsc{Step 3: Recover a local descent inequality under clipping.}
Applying the squared-distance recursion to $\|\bU_t-\bW^*\|_2^2$ produces a descent term involving $\langle \bar g_t,\widetilde{\bW}_{t-1}-\bW^*\rangle$ together with mini-batch and Gaussian fluctuation terms. Under active clipping, $\bar g_t$ is not the gradient of the empirical loss. On the initialization event, the per-example gradient norm is uniformly bounded by $G_\delta$, and we freeze the clipping coefficients at $\widetilde{\bW}_{t-1}$ to define
\[
\mathcal L_{S,t}^{\alpha}(\bW)=\frac1n\sum_{i=1}^n\alpha_{t,i}\ell(y_i f_{\bW}(\bx_i)),\qquad \alpha_{t,i}\in[\rho_{\rm clip},1].
\]
Then $\bar g_t=\nabla\mathcal L_{S,t}^{\alpha}(\widetilde{\bW}_{t-1})$ and $\rho_{\rm clip}\mathcal L_S\le\mathcal L_{S,t}^{\alpha}\le\mathcal L_S$. We stop the recursion when the shifted trajectory leaves the localization region; before this stopping time, the local KAN comparator inequality and self-bounding estimate yield descent in $\mathcal L_{S,t}^{\alpha}$ while controlling $\|\bar g_t\|_2^2$.

\

\textsc{Step 4: Control the stopped stochastic terms.}
The Gaussian terms, the mini-batch quadratic variation, and the martingale fluctuations in the stopped recursion are controlled with high probability using Gaussian and chi-square concentration, a without-replacement variance bound, and martingale concentration. Summing the resulting stopped recursion gives
\[
\frac1T\sum_{t=0}^{T-1}\mathcal L_S(\widetilde{\bW}_t)
\lesssim \frac{\mathcal L_S(\bW^*)}{\rho_{\rm clip}}
+\frac{\|\bW_0-\bW^*\|_2^2}{\eta_{\rm eff}T}
+\eta_{\rm eff}\Big(\frac1B+\frac1T\Big)
+\text{Noise}_{\lambda,\kappa},
\]
as long as the trajectory remains localized, where $\eta_{\rm eff}=\rho_{\rm clip}\eta$ and $\text{Noise}_{\lambda,\kappa}$ collects the $\lambda$-dependent Gaussian fluctuation terms.

\

\textsc{Step 5: Close the localization bootstrap and establish projection inactivity.}
For a comparator satisfying the properties verified in Step~6, the width and step-size conditions in Theorem~\ref{thm:lambda-positive-ntk-opt}, together with the stopped potential bound from Step~4, close the localization bootstrap and yield $\sup_{0\le t\le T}\|\bU_t-\bW^*\|_2\lesssim\tau_\gamma$. 
Since $\widetilde{\bW}_t=\bU_t-\eta c_{\rm priv}Z_t$, the Gaussian maximum-norm bound and the upper bound on $m$ imply $\eta c_{\rm priv}\|Z_t\|_2\lesssim\tau_\gamma$ uniformly in $t$. Combining these bounds with $\|\bW^*-\bW_0\|_2\lesssim\tau_\gamma$ gives $\sup_{0\le t\le T}\|\widetilde{\bW}_t-\bW_0\|_2\lesssim\tau_\gamma$. Choosing $R_*=C_R\tau_\gamma$ with $C_R$ sufficiently large therefore shows that projection is inactive, and hence $\bW_t=\widetilde{\bW}_t$ for all $t\le T$.

\

\textsc{Step 6: Construct the comparator from NTK separability.}
It remains to verify that a comparator satisfying the required localization and loss properties exists. Assumption~\ref{ass:ntk} yields a comparator $\bW^*=\bW_0+\tau_{\rm ntk}\bu$ such that, with high probability,
\[
\|\bW^*-\bW_0\|_2^2\lesssim \tau_\gamma^2,
\qquad
\mathcal L_S(\bW^*)\le \frac{\rho_{\rm clip}}{T},
\]
where $\tau_\gamma^2\asymp \big(\log^2(T/\rho_{\rm clip})+\log(n/\delta)\big)/\gamma^2$. Substituting this comparator into the bound from Step~4 and using the projection-inactivity result from Step~5 gives $\frac1T\sum_{t=0}^{T-1}\mathcal L_S(\bW_t)\lesssim A_{\rm corr}$. A union bound over the required high-probability events completes the proof. 
\end{proof}

\section{Applications}\label{sec:rates}

This section applies the optimization result from Section~\ref{sec:main_opt} to derive population risk guarantees via algorithmic stability. We then illustrate, through a finite-horizon example, how privacy calibration can preserve the optimization-side benefit of temporal correlation. Finally, we derive the corresponding guarantees for independent noise DP-SGD, non-private mini-batch SGD, and full-batch training.

\subsection{Population Risk of DP-SGD with Correlated Noise}\label{subsec:rate-corre}

The following theorem converts the optimization quantity $A_{\rm corr}$ from Theorem~\ref{thm:lambda-positive-ntk-opt} into a population risk guarantee through algorithmic stability, without requiring a particular privacy calibration of the noise scale. The detailed proof is given in Appendix~\ref{appen:lambda1-gen}. Denote the algorithmic randomness by $\A=\{{\cal B}_0,\ldots,{\cal B}_{T-1},Z_0,\ldots,Z_{T-1}\}$. We suppress logarithmic factors in the bound below.

\begin{theorem}[Population risk bound]\label{thm:lambda1-gen}
Suppose Assumption~\ref{ass:sigma} holds and the parameter conditions of Theorem~\ref{thm:lambda-positive-ntk-opt} are satisfied. Assume further that, conditional on the initialization, an i.i.d.\ sample $S\sim\mathcal P^n$ satisfies Assumption~\ref{ass:ntk} with margin $\gamma$ with probability at least $1-\delta_{\rm sep}$, where $\delta_{\rm sep}\le\delta/8$.   Let $\{\bW_t\}_{t=0}^T$ be generated by Algorithm~\ref{alg:dp-lambda-minibatch} with $C_{\mathrm{clip}}>0$ and $\eta_{\rm eff}\le \frac{1}{12C_{\sigma,b}p^3}$. Assume $m\gtrsim R_*^2\log(m/\delta)\eta^2
\bigl(
TA_{\mathrm{corr}}
+(\log(n/\delta)+R_*)
\bigl(\sqrt{T/B}+\log(1/\delta)\bigr)
\bigr)^2$. 
Then, with probability at least $1-\delta$ over the initialization,
\[
\frac1T\sum_{t=0}^{T-1}\mathbb E_{S,\A}[\mathcal L(\bW_t)]
\lesssim
\Bigl(1+\frac{\eta T}{n}\Bigr)
\Bigl(A_{\mathrm{corr}}+(\sqrt m+R_*)\delta\Bigr).
\]
\end{theorem}

Theorem~\ref{thm:lambda1-gen} shows that the dependence on temporal correlation and clipping obtained in the optimization analysis is preserved at the population level. In particular, when $\eta T/n=\O(1)$ and $\delta$ is sufficiently small, the averaged population risk inherits the leading order of $A_{\rm corr}$ up to constant and logarithmic factors.

\paragraph{A finite-horizon example where privacy calibration preserves the optimization gain.}
Consider full-batch training $B=n$ with any fixed finite horizon $T$, $C_{\rm clip}=G_\delta$, and the same privacy budget $(\epsilon,\delta)$ for the correlated and independent mechanisms. Under a finite-horizon matched calibration, the linear and quadratic privacy-dependent terms relative to independent noise are multiplied by
\[
r_1(\lambda)=\Big(\frac1T\sum_{t=1}^T(1-\lambda^t)^2\Big)^{1/2},
\qquad
r_2(\lambda)=\big((1-\lambda)^2+\eta\lambda^2\big)\frac1T\sum_{t=1}^T\Big(\frac{1-\lambda^t}{1-\lambda}\Big)^2.
\]
For every $\lambda\in(0,1)$, $r_1(\lambda)<1$, while $r_2(\lambda)=1-\frac{2\lambda}{T}+\O(\lambda^2)$ as $\lambda\to0^+$.  
Hence, for every fixed finite $T$, there exists $\lambda_T\in(0,1)$ such that $r_1(\lambda_T)<1$ and $r_2(\lambda_T)<1$. Under the common admissibility conditions of Theorem~\ref{thm:lambda1-gen}, temporal correlation therefore yields a strictly smaller population risk bound than independent noise DP-SGD under the same privacy budget. The detailed privacy calculation and resulting risk bound are given in Appendix~\ref{appen:finite-horizon}.

\paragraph{Conservative worst-case benchmark.}
Appendix~\ref{appen:subsampling} analyzes a conservative worst-case mini-batch privacy calibration. Even under this calibration, the resulting population risk bound still retains, up to logarithmic and model-dependent factors, the optimal dependence on $n$, $d$, and $\epsilon$ known for  Lipschitz convex DP stochastic convex optimization ~\citep{bassily2019private}. 

\subsection{Recovery of Standard Regimes}\label{subsec:standard-regimes}

Our general framework recovers independent noise DP-SGD, non-private mini-batch SGD, and the corresponding full-batch regimes as special cases. For the representative rates below, we take $C_{\rm clip}=G_\delta$ for simplicity; other clipping thresholds can be handled through the general clipping analysis in Theorem~\ref{thm:lambda-positive-ntk-opt}.

\paragraph{Independent noise DP-SGD.}
Setting $\lambda=0$ recovers standard DP-SGD with independent Gaussian noise. Under the standard privacy calibration $\kappa^2\asymp B^2T\log(1/\delta)/(n^2\epsilon^2)$, Algorithm~\ref{alg:dp-lambda-minibatch} satisfies $(\epsilon,\delta)$-DP. Under a representative parameter regime, the resulting averaged optimization and population risks are both of order $\O\bigl(\frac{1}{\sqrt n}+\frac{\sqrt d}{n\epsilon}\bigr)$. The complete optimization and generalization analyses, together with the precise parameter choices, are given in Appendix~\ref{appen:lambda0}.

\paragraph{Non-private mini-batch SGD.}
Setting $\kappa=0$ removes the Gaussian perturbation and recovers non-private mini-batch SGD. Under a representative parameter regime, both the averaged optimization and population risks are of order $\O(1/n)$. The corresponding general result and its proof are given in Appendix~\ref{appen:SGD}.

\paragraph{Full-batch GD and DP-GD.}
Taking $B=n$ recovers the corresponding full-batch regimes. In the fixed-second-layer setting, the independent noise DP-GD specialization yields a population risk rate $\O\bigl(\frac{\sqrt d}{n\epsilon}\bigr)$ under $m\asymp\polylog(n/\delta)/\gamma^4$, matching the established sample and privacy dependence in \cite{wang2026optimization}. Similarly, the non-private full-batch specialization yields a population risk rate $\O(1/n)$, matching the rate in \cite{wang2026optimization}. 

\section{Discussion and Limitations}\label{sec:discussion}
\paragraph{On fixing the second layer.}
In this work, we draw the second-layer coefficients $\bc$ at initialization and keep them fixed throughout training, while optimizing only the first-layer parameters $\bW$. This is a common theoretical formulation for shallow neural networks. In the KAN literature, \cite{eshtehardianconvergence} explicitly study two-layer KANs with first-layer training. The same modeling choice is also common in the theory of fully connected neural networks \cite{lei2026optimization,taheri2024generalization,lei2022stability} fix the output-layer weights and optimize only the hidden-layer parameters. We adopt this formulation for analytical simplicity and to avoid introducing an additional trainable parameter block into the main analysis.

More importantly, fixing the second layer is not responsible for the leading-order rates obtained in this paper. \cite{wang2026optimization} analyze the corresponding two-layer KAN with both layers trained and obtain the same leading optimization and population risk rates. Their joint gradient and Hessian bounds, including the mixed first--second-layer blocks, provide the local estimates needed in our analysis. Thus, replacing $\bW$ by the full parameter vector and incorporating these estimates would mainly require additional bookkeeping for the second-layer trajectory  and is not expected to change the leading-order bounds.  We therefore fix the second layer for analytical simplicity.

\paragraph{Limitation.}  
Our analysis focuses on binary classification with the logistic loss, while corresponding optimization and population risk guarantees for regression remain open. A second limitation is the restriction to two-layer KANs. Extending the analysis to deep KANs is tied to the broader open problem of understanding the optimization and generalization of deep KANs, and is therefore beyond a routine extension of the present arguments. Finally, our results are developed specifically for KANs. Extending the correlated noise optimization and population risk analysis to other neural network architectures, such as MLPs and CNNs, would require architecture-specific control of their local training dynamics and remains an interesting direction for future work.

\section{Conclusion}\label{sec:conclu}
We established optimization risk bounds for two-layer KANs trained by clipped mini-batch DP-SGD with correlated noise, explicitly characterizing the effects of temporal correlation, clipping, mini-batch sampling, and network width. Combining this analysis with algorithmic stability yields corresponding population risk guarantees. We further showed that the optimization benefit of temporal correlation can survive privacy calibration in a finite-horizon full-batch regime, while even a conservative worst-case mini-batch calibration retains, up to logarithmic and model-dependent factors, the optimal dependence on $n$, $d$, and $\epsilon$ known for Lipschitz convex DP-SCO. Our framework also recovers standard private and non-private training regimes as special cases.

\section*{Acknowledgment}
Part of this work was conducted within the DFG SPP 2298 (ID 464252197). PW acknowledges support by the Alexander-von-Humboldt Foundation through a Humboldt Research Fellowship. MK and SF acknowledge support by the DFG through FOR 5359 (ID 459419731), TRR 375 (ID 511263698), and SPP 2331 (ID 441958259, 553345933, 466468799), by the Carl-Zeiss Foundation through the initiative AI-Care, and by the BMFTR award 01IS24071A.
NK is supported in part by the Austrian Science Fund (FWF) [10.55776/COE12].
\bibliography{learning}
\bibliographystyle{plain}

\newpage
\appendix
\onecolumn

\begin{center}
   {\Large \bf  Appendix}
\end{center}

\section{Further Related Work}\label{sec:appen-relatedwork}

This appendix expands on the related work referenced in Section~\ref{sec:relatedwork}, covering generalization theory for neural networks (Appendix~\ref{sec:appen-relatedMLP}) and privacy amplification by subsampling for correlated-noise mechanisms (Appendix~\ref{sec:appen-subsampling-relatedwork}).

\subsection{Population risk bounds for neural networks}\label{sec:appen-relatedMLP}

Most existing optimization and generalization results for neural networks have been developed for \emph{fully connected} multilayer perceptrons (MLPs), rather than architectures with edge-wise functional parametrizations such as KANs.
Broadly speaking, the theoretical literature on MLPs can be organized into three main directions.

A first line of work is based on the neural tangent kernel (NTK) viewpoint \cite{jacot2018neural}.
In this regime, a sufficiently wide network behaves approximately like its linearization around initialization, which makes it possible to analyze gradient-based training through an associated kernel.
This perspective has led to influential convergence guarantees for overparameterized MLPs and, in the lazy-training regime, to sharp statistical and excess risk bounds in a variety of settings \cite{allen2019convergence,arora2019fine,cao2019generalization,chen2021much,nguyen2024many,nitanda2021optimal,zou2020gradient}.

A second line of work studies generalization through uniform convergence approach.
Typical tools include Rademacher complexity, covering numbers, norm-based complexity measures, and margin-based arguments \cite{bartlett2017spectrally,frei2023random,ji2019polylogarithmic,lei2026optimization,li2025optimal,nitanda2019gradient,shi2026towards,zhou2024generalization}.
These approaches yield broad generalization guarantees for MLPs, often under relatively mild assumptions on the training algorithm itself.

More recently, there has been growing interest in algorithm-dependent generalization analyses based on algorithmic stability \cite{lei2022stability,richards2021stability,taheri2024generalization,taheri2024sharper,wang2025generalization}.
This line of work studies how perturbations in the training data propagate through the optimization dynamics, and consequently yields generalization and excess risk bounds that are directly tied to the learning algorithm.
This viewpoint is especially relevant to the present paper, since our goal is likewise to derive algorithm-dependent risk guarantees, but now in the more structured setting of KANs and under minibatch, clipped, and differentially private training.

\subsection{Privacy amplification by subsampling for correlated noise}\label{sec:appen-subsampling-relatedwork}

Prior analyses of correlated-noise mechanisms with privacy amplification cover banded matrix mechanisms, Poisson-subsampled matrix mechanisms, and Balls-in-Bins schemes \cite{choquette2024near,choquette-choo2023amplified,choquette2023privacy,schuchardt2026sampling}. We complement this line by deriving an analytical noise-multiplier bound for DP-$\lambda$CGD under mini-batches drawn uniformly without replacement, which underlies the privacy guarantee for Algorithm~\ref{alg:dp-lambda-minibatch}.

\section{Useful Lemmas}
In this section, we introduce some useful lemmas that will be used in the proofs. 
\begin{lemma}[\cite{wang2026optimization}]\label{lem:block_matrix}
    Let $s\in\mathbb{N}$ and $\bA_i\in\R^{m_i\times n_i}$ for $i\in[s]$.
    Define the diagonal block matrix $$\bB = \begin{bmatrix}
        \bA_1 & \cdots & \mathbf{0}\\
        \vdots & \ddots & \vdots\\
        \mathbf{0} & \cdots & \bA_s
    \end{bmatrix} \in \R^{\sum_{i=1}^sm_i \times \sum_{i=1}^sn_i}$$
    Then, it holds that
    \begin{align*}
        \|\bB\|_2 = \max_{i\in[s]} \|\bA_i\|_2.
    \end{align*}
\end{lemma}
\begin{lemma}\label{lem:finite-pop-var}
For any $n\ge 2$, let \(x_1,\dots,x_n\in\R^D\) be deterministic vectors, and let $\bar x:=\frac1n\sum_{i=1}^n x_i.$
Suppose that \(\mathcal B\subseteq[n]\) is sampled uniformly without replacement among all subsets of cardinality \(B\). Define $\bar x_{\mathcal B} =\frac1B\sum_{i\in\mathcal B} x_i.$
Then, $\mathbb E_{\mathcal B}[\bar x_{\mathcal B}]=\bar x$ and 
\[ \mathbb E_{\mathcal B}\big[\|\bar x_{\mathcal B}-\bar x\|_2^2\big] = \frac{n-B}{B(n-1)}\cdot \frac1n\sum_{i=1}^n \|x_i-\bar x\|_2^2.\]
Furthermore, if \(\|x_i\|_2\le G\) for all \(i\in[n]\), then
\[\mathbb E_{\mathcal B}\big[\|\bar x_{\mathcal B}-\bar x\|_2^2\big]\le\frac{G^2}{B}.\]
\end{lemma}
\begin{proof}
The proof is standard, we include it for completeness. 
The unbiasedness follows from
\[\mathbb E_{\mathcal B}[\bar x_{\mathcal B}]=\frac1B\sum_{i=1}^n \mathbb P(i\in\mathcal B)\,x_i=\frac1B\sum_{i=1}^n \frac{B}{n}x_i=\bar x.\]
To prove the second identity, let $x_i' =x_i-\bar x $ for any $ i\in[n].$
Then it holds \(\sum_{i=1}^n x_i'=0\)  and $\bar x_{\mathcal B}-\bar x=\frac1B\sum_{i\in\mathcal B} x_i'.$
Hence,
\begin{align*}
\mathbb{E}_{\mathcal B} \big[\|\bar x_{\mathcal B} - \bar x \|_2^2\big]
&=\frac{1}{B^2}\,\mathbb E_{\mathcal B}\Big[\Big\|\sum_{i\in\mathcal B} x_i'\Big\|_2^2\Big] 
=\frac{1}{B^2}\Big(\sum_{i=1}^n \mathbb P(i\in\mathcal B)\|x_i'\|_2^2 + \sum_{i\neq j} \mathbb P(i,j\in\mathcal B) \langle x_i',x_j'\rangle\Big).
\end{align*}
Since \(\mathcal B\) is sampled uniformly without replacement with \(|\mathcal B|=B\), we know
\[\mathbb P(i\in\mathcal B)=\frac{B}{n},\qquad\mathbb P(i,j\in\mathcal B)=\frac{B(B-1)}{n(n-1)}\quad (i\neq j).\]
Moreover, from \(\sum_{i=1}^n x_i'=0\) we get
\[\sum_{i\neq j}\langle x_i',x_j'\rangle=\Big\|\sum_{i=1}^n x_i'\Big\|_2^2-\sum_{i=1}^n \|x_i'\|_2^2=-\sum_{i=1}^n \|x_i'\|_2^2.\]
Therefore,
\begin{align*}
\mathbb E_{\mathcal B}\big[\|\bar x_{\mathcal B}-\bar x\|_2^2\big]
&=\frac{1}{B^2}\Big(\frac{B}{n}-\frac{B(B-1)}{n(n-1)}\Big)\sum_{i=1}^n \|x_i'\|_2^2 
=\frac{n-B}{B(n-1)}\cdot \frac1n\sum_{i=1}^n \|x_i-\bar x\|_2^2.
\end{align*}
Further,  if \(\|x_i\|_2\le G\) for all \(i\in[n]\), then 
\[\frac1n\sum_{i=1}^n \|x_i-\bar x\|_2^2=\frac1n\sum_{i=1}^n \|x_i\|_2^2-\|\bar x\|_2^2\le \frac1n \sum_{i=1}^n \|x_i\|_2^2\le G^2.\]
Hence,
\[\mathbb E_{\mathcal B}\big[\|\bar x_{\mathcal B}-\bar x\|_2^2\big]\le\frac{n-B}{B(n-1)}\,G^2\le\frac{G^2}{B}.\]
This completes the proof.
\end{proof}

\begin{lemma}[Chernoff bound for adapted bounded variables]
\label{lem:adapted-chernoff}
Let $\{X_t\}_{t=1}^T$ be a nonnegative adapted process with respect to
$\{\mathscr F_t\}_{t=0}^T$. Suppose $0\le X_t\le L$  a.s.   and $\sum_{t=1}^T \mathbb E[X_t\mid\mathscr F_{t-1}]
\le V$ a.s.  
Then, for every $\delta\in(0,1)$, with probability at least $1-\delta$, 
\[ \sum_{t=1}^T X_t \le 2V+2L\log\Big(\frac1\delta\Big) .\]
\end{lemma}
\begin{lemma}[Sub-Gaussianity of sampling without replacement]
\label{lem:finite-pop-subgaussian}
Let $a_1,\ldots,a_n\in\R$ be deterministic and let $\bar a=\frac1n\sum_{i=1}^n a_i$. 
Let $\mathcal B\subset[n]$ be sampled uniformly without replacement with
$|\mathcal B|=B$. If $|a_i|\le R$ for all $i\in[n]$, 
then, for every $\theta\in\R$,
\[
\mathbb E\Big[
\exp\Big(
\theta\Big(
\frac1B\sum_{i\in\mathcal B}a_i-\bar a
\Big)
\Big)
\Big]
\le
\exp\left(
\frac{\theta^2R^2}{2B}
\right).
\]
\end{lemma} 
\begin{proof}
Let $I_1,\ldots,I_B$ be sampled uniformly without replacement from $[n]$,
and let $J_1,\ldots,J_B$ be i.i.d. uniform on $[n]$.
By Hoeffding's comparison theorem for sampling without replacement
\citep[Theorem~4]{hoeffding1963probability},
\[
\mathbb E\Big[
\exp\Big(
\theta\Big(
\frac1B\sum_{s=1}^B a_{I_s}-\bar a
\Big)
\Big)
\Big]
\le
\mathbb E\Big[
\exp\Big(
\frac{\theta}{B}\sum_{s=1}^B(a_{J_s}-\bar a)
\Big)
\Big].
\]
Since $a_{J_s}-\bar a$ has mean zero and lies in an interval of length at
most $2R$, Hoeffding's lemma yields
\[
\mathbb E\Big[
\exp\Big(
\frac{\theta}{B}(a_{J_s}-\bar a)
\Big)
\Big]
\le
\exp\left(
\frac{\theta^2R^2}{2B^2}
\right).
\]
Independence of $J_1,\ldots,J_B$ completes the proof.
\end{proof}

Our convergence analysis relies on the following standard mirror descent inequality (see, e.g., \cite{wang2026optimization}).  
 \begin{lemma}[Mirror descent inequality]\label{lem:MD}
Let $\K\subseteq \R^{ m dp}$ be a nonempty closed convex set, and let $\proj_\K(\cdot)$ denote the Euclidean projection onto $\K$.
Fix any $\eta>0$ and any vector $g\in\R^{mdp}$. Define
\begin{equation*} 
    \bW^{+} = \proj_\K\big(\bW - \eta g\big).
\end{equation*}
Then, for any comparator $\bW^*\in \K$, the following inequality holds
\begin{equation*} 
    \big\langle g,\ \bW-\bW^*\big\rangle \le \frac{1}{2\eta}\Big(\|\bW-\bW^*\|_2^2-\|\bW^{+}-\bW^*\|_2^2\Big)+\frac{\eta}{2}\|g\|_2^2.
\end{equation*}
\end{lemma}

Denote $[\bW_1,\bW_2] = \{\alpha\bW_1 + (1-\alpha) \bW_2 : \alpha\in[0,1]\}$ as the line segment between $\bW_1$ and $\bW_2$.
\begin{lemma}[Local quasi-convexity property \cite{taheri2024generalization}]\label{lem:quasi-convexity}
    Suppose $s\in\mathbb N$ and $G:\R^s \rightarrow \R$ be a second-order differentiable function satisfying $ \lambda_{\min}(\nabla^2 G(\bW)) \ge -\kappa G(\bW) $. Let $\bW_1, \bW_2\in \R^s$ be two arbitrary points with distance $\|\bW_1-\bW_2\|_2\le D\le \sqrt{2/\kappa}$. Let $\tau:=(1-D^2\kappa/2)^{-1}$.
    Then, 
    \[ \max_{\mathcal{V} \in [\bW_1,\bW_2]} G(\mathcal{V}) \le \tau \max\big\{ G(\bW_1) , G(\bW_2)  \big\}. \]
\end{lemma}

For any scalar $v\in\R$, we denote $\bh(v) = [b_1(v), \ldots, b_{p}(v)]^\top\in\R^{p}$.
For $s\in\mathbb{N}$ and a vector $\bu = [u_1,\ldots,u_s]^\top \in \R^{s}$, we denote $\bh(\bu) = [\bh(u_1)^\top,\ldots,\bh(u_s)^\top]^\top\in\R^{sp}.$
For each hidden unit $j\in[m]$, the first-layer spline coefficients $\{w_{i,j,k}\}_{i\in[d],\,k\in[p]}$ are arranged into a vector $\bw_j\in\R^{dp}$ according to a fixed ordering of $(i,k)$. We write $\bW=(\bw_1,\ldots,\bw_m)^\top\in\R^{m\times dp}$. The second-layer spline coefficients $\{c_{j,k}\}_{j\in[m],\,k\in[p]}$ are collected as $\bc=(\bc_1,\ldots,\bc_m)\in\R^{mp}$ with $\bc_j=(c_{j,1},\ldots,c_{j,p})\in\R^p $. Then $f_\bW$ can be rewritten as
\begin{align*}
    f_\bW(\bx) = \frac{1}{\sqrt{m}} \bc^\top\bh\Big(\sigma\big(\frac{1}{\sqrt{d}}\bW\bh(\bx)\big)\Big).
\end{align*}

To control the gradients and Hessians, we use the following lemma \cite{wainwright2019high}, which provides estimates for $\bc\in\R^{mp}$. 
\begin{lemma}\label{lem:bound-c}
    For any $\delta\in(0,1)$,
    define the event \begin{equation}\label{eq:good-event}
    \mathcal E_\delta := \Big\{\|\bc\|_2 \le 4\sqrt{pm} + 2\sqrt{\log(2/\delta)} \ \text{ and }\  \max_{j\in[m]}\|\bc_j\|_2 \le 4\sqrt{p}+2\sqrt{\log(2m/\delta)}\Big\}.
\end{equation}
It holds that
 $\mathbb P(\mathcal E_\delta)\ge 1-\delta $. 
\end{lemma}  
To simplify the analysis, we work on the above high-probability event associated with the initialization of \(\bc\).

\begin{lemma}[Gradient and Hessian]\label{lem:hessian}
    Let $\delta\in(0,1)$ and $C_{\sigma,b}>0$  be a constant that depends solely on $\sigma, b$. Suppose Assumption~\ref{ass:sigma} and the event  $\mathcal E_\delta$  occurs. 
    It holds for any $\bW$ and any $\bx \in \X$ that
    \[ \big\|\nabla f_\bW(\bx)\big\|_2 \le C_{\sigma,b} \, p  \, \Big( \sqrt{p} +\sqrt{\frac{\log({1}/{\delta})}{m}} \Big) \]
and
    \[\big\|\nabla^2 f_{\bW}(\bx)\big\|_2  \le \frac{ C_{\sigma, b} \, p^{\frac{ 3}{2}}\big(\sqrt{p} +  \sqrt{\log({m}/{\delta})} \big)}{\sqrt{m}}.\]
\end{lemma}
\begin{proof}
The proof is based on arguments from \cite{wang2026optimization}, where both \(\bc\) and \(\bW\) are trained. For completeness, we provide a detailed proof. 
We first estimate $\|\nabla  f_\bW(\bx)\|_2$. 
For all $v\in\R$, we denote $\bh'(v) = [b_1'(v),\ldots,b_p'(v)]^\top \in\R^p$ and  $\bh''(v) = [b_1''(v),\ldots,b_p''(v)]^\top \in\R^p$. 
Define 
\begin{align}\label{eq:bu}
    \bu(\bx) = \sigma\big(\frac{1}{\sqrt{d}}\bW'\bh(\bx)\big) \ \text{ and } \ \bD(\bx) = \text{diag}\big(\sigma'(\frac{1}{\sqrt{d}}\bw_i^\top\bh(\bx))\big)_{i=1}^m \in \R^{m\times m}.
\end{align}
From the form $f_\bW (\bx) = \frac{1}{\sqrt{m}} \bc^\top \bh(\bu(\bx))$, we have
\begin{align*}
    \frac{\partial \bh(\bu(\bx))}{\partial \bu(\bx)} 
    = \begin{bmatrix}
        \bh'(u_1(\bx)) & \mathbf{0} &\cdots &  \mathbf{0}\\
        \mathbf{0} & \bh'(u_2(\bx)) &\cdots &  \mathbf{0}\\
        \vdots & \vdots & \ddots & \vdots\\
        \mathbf{0} & \mathbf{0} & \mathbf{0} & \bh'(u_m(\bx))
    \end{bmatrix}\in\R^{mp\times m}
\end{align*}
and
\begin{align*}
    \frac{\partial \bu(\bx)}{\partial \bw_i} 
    = \begin{bmatrix}
        \mathbf{0}\\
        \vdots\\
        \frac{1}{\sqrt{d}}\sigma'(\frac{1}{\sqrt{d}}\bw_i^\top \bh( \bx))\bh(\bx)^\top\\
        \vdots\\
        \mathbf{0}
    \end{bmatrix}\in\R^{m\times pd}.
\end{align*}

According to the chain rule, for any $i\in[m]$, it holds that
\begin{align*}
    \partial_{\bw_i} f_{\bW}(\bx) &= \frac{\partial f_{\bW}(\bx)}{\partial \bh(\bu(\bx))} \ \frac{\partial \bh(\bu(\bx))}{\partial \bu(\bx)} \ \frac{\partial \bu(\bx)}{\partial \bw_i} =\frac{1}{\sqrt{m d}} \big\langle \bc_i, \bh'(u_i(\bx)) \big\rangle   \sigma'\big(\frac{1}{\sqrt{d}}\bw_i^\top\bh(\bx)\big)\bh(\bx)^\top.
\end{align*} 
Since $\bW = \text{Vec}\big(\{\bw_i\}_{i=1}^m\big)\in\R^{mpd}$ is the vectorization of $\bW$, it then holds that
$$\nabla f_{\bW}(\bx) = \frac{1}{\sqrt{m d}}\text{Vec}\Big(\big\{\sigma'\big(\frac{1}{\sqrt{d}}\bw_i^\top\bh(\bx)\big)\langle \bc_i, \bh'(u_i(\bx))\rangle \bh(\bx)\big\}_{i=1}^m\Big) \in \R^{mpd}.$$
Hence, we get
\begin{align}
    \|\nabla f_{\bW}(\bx)\|_2 
    &= \frac{1}{\sqrt{md}}\Big(\sum_{i=1}^m \sigma'\big(\frac{1}{\sqrt{d}}\bw_i^\top\bh(\bx)\big)^2 \, \big|\langle \bc_i, \bh'(u_i(\bx))\rangle\big|^2 \|\bh(\bx)\|_2^2\Big)^{\frac{1}{2}} \nonumber\\
    &\le B_\sigma'B_b\sqrt{\frac{p }{m}}\Big(\sum_{i=1}^m \big|\langle \bc_i, \bh'(u_i(\bx))\rangle\big|^2\Big)^{\frac{1}{2}} \nonumber\\
    &\le B_\sigma'B_bB_b'\sqrt{\frac{p^2 }{m}}\Big(\sum_{i=1}^m \|\bc_i\big\|_2^2\Big)^{\frac{1}{2}}
   \le B_\sigma'B_bB_b' \, p \, \Big(4\sqrt{p} + 2\sqrt{\frac{\log(\frac{2}{\delta})}{m}}\Big)\label{eq:partial_a_norm},
\end{align}
where the last inequality used \eqref{eq:good-event}.
The first part of the lemma is proved. 

Now, we turn to estimate the Hessian of $f_\bW(\bx)$, i.e.,  
\begin{align*}
    \nabla^2 f_\bW(\bx) 
    = \begin{bmatrix}
        \partial_{\bw_1}^2 f_\bW(\bx) & \cdots & \mathbf{0}\\
        \vdots & \ddots & \vdots\\
        \mathbf{0} & \cdots & \partial_{\bw_m}^2 f_\bW(\bx)
    \end{bmatrix}\in\R^{mpd\times mpd},
\end{align*}
where 
\begin{align*}
    \partial_{\bw_i}^2 f_\bW(\bx) = & \frac{1}{d\sqrt{m}}\Big(\sigma''\big(\frac{1}{\sqrt{d}}\bw_i^\top\bh(\bx)\big)\big\langle \bc_i, \bh'\big(\sigma\big(\frac{1}{\sqrt{d}}\bw_i^\top\bh(\bx)\big)\big)\big\rangle \nonumber\\&+\big(\sigma'\big(\frac{1}{\sqrt{d}}\bw_i^\top\bh(\bx)\big)^2\big\langle\bc_i, \bh''\big(\sigma\big(\frac{1}{\sqrt{d}}\bw_i^\top\bh(\bx)\big)\big)\big\rangle\big)\Big)\bh(\bx)\bh(\bx)^\top \in \R^{pd\times pd}.
\end{align*}
We rewrite $\nabla^2 f_\bW(\bx)$ as 
\[\partial_{\bw_i}^2 f_\bW(\bx) = \frac{1}{d\sqrt{m}}\langle\bc_i, \bv_i\rangle\bh(\bx)\bh(\bx)^\top \]
with $\bv_i = \sigma''\big(\frac{1}{\sqrt{d}}\bw_i^\top\bh(\bx)\big)\bh'\big(\sigma\big(\frac{1}{\sqrt{d}}\bw_i^\top\bh(\bx)\big)\big) + \sigma'\big(\frac{1}{\sqrt{d}}\bw_i^\top\bh(\bx)\big)^2 \,\bh''\big(\sigma\big(\frac{1}{\sqrt{d}}\bw_i^\top\bh(\bx)\big)\big)$.  

According to Assumption \ref{ass:sigma}, we can control $\|\bv_i\|_2$  as 
\begin{align*}
    \|\bv_i\|_2 &= \big\|\sigma''\big(\frac{1}{\sqrt{d}}\bw_i^\top\bh(\bx)\big) \, \bh'\big(\sigma\big(\frac{1}{\sqrt{d}}\bw_i^\top\bh(\bx)\big)\big) + \sigma'\big(\frac{1}{\sqrt{d}}\bw_i^\top\bh(\bx)\big)^2 \,\bh''\big(\sigma\big(\frac{1}{\sqrt{d}}\bw_i^\top\bh(\bx)\big)\big)\big\|_2 \\&\le \|\sigma''\|_\infty\sqrt{p}\|b'\|_\infty + \|\sigma'\|_\infty^2\sqrt{p}\|b''\|_\infty \le \sqrt{p}(B_{\sigma}''B_{b}' + B_{\sigma}'^2B_{b}'').
\end{align*}
Combining the estimate of $\|\bv_i\|_2$ and the fact that $\nabla^2 f_\bW(\bx)$ is a block diagonal matrix, we can use Lemma \ref{lem:block_matrix} to get
\begin{align}
    \big\|\nabla^2 f_\bW(\bx)\big\|_2 &= \max_{i\in[m]}\big\|\nabla_{\bw_i}^2 f_\bW(\bx)\big\|_2 = \max_{i\in[m]}\sup_{\|\ba\|_2=1} \big|\ba^\top \, \partial_{\bw_i}^2 f_\bW(\bx) \, \ba\big|\nonumber\\
    &= \frac{1}{d\sqrt{m}}\max_{i\in[m]} \sup_{\|\ba\|_2=1}\big|\langle\bc_i,\bv_i\rangle \langle \bh(\bx), \ba\rangle^2 \big| \le \frac{1}{d\sqrt{m}}\max_{i\in[m]} \big\|\bc_i\big\|_2 \, \big\|\bv_i\big\|_2 \, \big\| \bh(\bx)\big\|_2^2\nonumber\\
    &\le B_b^2\frac{p }{\sqrt{m}} \Big( 4\sqrt{p} + 2\sqrt{\log(\frac{2m}{\delta})}  \Big)\max_{i\in[m]} \big\|\bv_i\big\|_2 \nonumber\\
    &\le B_b^2(B_{\sigma}''B_{b}' + B_{\sigma}'^2B_{b}'')\frac{p^{\frac{ 3}{2}}}{\sqrt{m}}\Big( 4\sqrt{p} + 2\sqrt{\log(\frac{2m}{\delta})} \Big),
\end{align}
where the second equality used the fact that $\partial_{\bw_i}^2f_\bW(\bx)$ is symmetric, the first inequality used Cauchy-Schwarz inequality, the second inequality used Assumption \ref{ass:sigma} that $\sup_{t\in\R}|b(t)| \le B_b$ and \eqref{eq:good-event}.
The proof is complete.
\end{proof}

The following lemma shows that the largest and smallest eigenvalues of \(\nabla^2 \ell(y f_\bW(\bx))\) admit well-controlled upper and lower bounds, respectively. As a consequence, the loss function \(\ell(y f_\bW(\bx))\) is weakly convex and smooth with respect to \(\bW\).
\begin{lemma}[Smoothness and Curvature]\label{pro:smooth}
    Let $\delta\in(0,1)$.
    Suppose   Assumption~\ref{ass:sigma} and the event  $\mathcal E_\delta$  occurs. Assume $m\gtrsim  \log(m/\delta) $.
    It holds for any $\bW$ and any data point $(\bx,y)\in \Z$, that
    $$\lambda_{\min}\big(\nabla^2  \ell(y f_\bW(\bx))\big)  \ge -  \frac{C_{\sigma, b} \, p^{\frac{3}{2}} \big(\sqrt{\log(\frac{ m}{\delta})}+\sqrt{p} \big)}{\sqrt{m}}  \ell(yf_\bW(\bx)),$$ 
    $$\lambda_{\max}\big(\nabla^2 \ell(y f_\bW(\bx))\big)  \le   C_{\sigma,b } \,  p^3$$ 
   and
    $$\big\| \nabla \mathcal{L}_S(\bW)\big\|_2^2\le 4    C_{\sigma,b } \, p^3\L_S(\bW).$$
\end{lemma}
\begin{proof}
From the chain rule, the gradient of loss is given as
\[ \nabla  \ell(y f_\bW(\bx)) = \ell'(y f_\bW(\bx))y  \nabla  f_\bW(\bx) .\]
Note that
\begin{equation*}
    \nabla^2  \ell(y f_\bW(\bx)) = \ell''(y f_\bW(\bx)) y^2  \nabla  f_\bW(\bx)\nabla  f_\bW(\bx)^\top + \ell'(y f_\bW(\bx)) y  \nabla^2  f_\bW(\bx).
\end{equation*}
Since $\ell$ is convex, $\ell''(a) \ge 0$ for all $a\in\R$.
Then, $\ell''(y f_\bW(\bx))   \nabla  f_\bW(\bx)\nabla  f_\bW(\bx)^\top$ is a PSD matrix.
By further noting that $|\ell'(a)| \le 1$ and $|\ell''(a)|\le 1/4$ for all $a\in\R$, we have
\[-|\ell'(y f_\bW(\bx))|\big\|\nabla^2  f_\bW(\bx)\big\|_2 \le \lambda_{\min}\big(\nabla^2 \ell(y f_\bW(\bx))\big)\]
and
\[\lambda_{\max}\big(\nabla^2  \ell(y f_\bW(\bx))\big) \le
      \frac{1}{4} \big\| \nabla  f_\bW(\bx)\big\|_2^2 +  \big\| \nabla^2  f_\bW(\bx)\big\|_2 .
\]

Plugging the estimates of $\big\| \nabla  f_\bW(\bx)\big\|_2$ and $\big\| \nabla^2  f_\bW(\bx)\big\|_2$ in Lemma~\ref{lem:hessian} back and noting that $|\ell'(yf_\bW(\bx))|\le  \ell(yf_\bW(\bx))$, we know
\[\lambda_{\min}\big(\nabla^2  \ell(y f_\bW(\bx))\big)  \ge -  \frac{C_{\sigma, b} \, p^{\frac{3}{2}} \big(\sqrt{\log(\frac{ m}{\delta})}+\sqrt{p} \big)}{\sqrt{m}}  \ell(yf_\bW(\bx)),\]
and
   \[
\lambda_{\max}\big(\nabla^2  \ell(y f_\bW(\bx))\big)
\le
C_{\sigma,b}\Big[
p^2\Big(p+\frac{\log(\frac{1}{\delta})}{m}\Big)
+
\frac{p^{\frac{3}{2}}\big(\sqrt p+\sqrt{\log(\frac{m}{\delta})}\big)}{\sqrt m}
\Big].
\]
 
Furthermore, if $m\gtrsim  \log(m/\delta) $, we have
\[\lambda_{\max}\big(\nabla^2 \ell(y f_\bW(\bx))\big) \le
     C_{\sigma,b } \, p^3,\]
which completes the first two inequalities of the lemma.

Note $C_{\sigma,b } \, p^3$ is an upper bound for the term $\sup_{\bx\in\X} \frac{1}{4}\|\nabla f_{\bW}(\bx)\|_2^2 + \|\nabla^2 f_{\bW}(\bx)\|_2$ and $|y|=1$ for any $y\in\Y$.
    Then, it holds for any $\bW$ that
    \begin{align*}
        \big\| \nabla \mathcal{L}_S(\bW)\big\|_2^2 &= \Big\|\frac{1}{n}\sum_{i=1}^n \ell'(y_i f_{\bW}(\bx_i))y_i\nabla f_{\bW}(\bx_i)\Big\|_2^2 \le \Big|\frac{1}{n}\sum_{i=1}^n \ell'(y_if_{\bW}(\bx_i)) \Big|^2 \sup_{i\in[n]} \big\|\nabla f_{\bW}(\bx_i)\big\|_2^2\\
        &\le 4    C_{\sigma,b } \, p^3\Big|\frac{1}{n}\sum_{i=1}^n \ell(y_if_{\bW}(\bx_i)) \Big| =4    C_{\sigma,b } \, p^3\L_S(\bW),
    \end{align*}
    where the second inequality used the self-boundedness property and $|\ell'(\cdot)| \le 1$ of the logistic loss.
    This completes the proof of the lemma.
\end{proof}

\section{Proofs for optimization of DP-SGD with correlated noise}\label{appen:lambda1-opt}

Throughout the paper, we use \(\bW\) to denote the vectorized parameter 
\((\bw_1^\top,\ldots,\bw_m^\top)^\top\in\mathbb R^{mdp}\), unless otherwise specified. 
We occasionally use the same symbol for the matrix representation in the definition of 
\(f_{\bW}(\bx) = \frac{1}{\sqrt{m}} \bc^\top\bh\big(\sigma\big(\frac{1}{\sqrt{d}}\bW\bh(\bx)\big)\big)\), 
when the meaning is clear from context.

Moreover, under $\mathcal E_\delta$, \eqref{eq:partial_a_norm} gives
\[
\|\nabla f_{\bW}(\bx)\|_2\le G_\delta
:= B_\sigma' B_b B_b' \, p 
\Big(4\sqrt p + 2\sqrt{\frac{\log(2/\delta)}{m}}\Big),
\qquad
\forall \bx\in\X,\ \forall \bW .
\]
Since $|\ell'|\le 1$, this further yields a uniform upper bound on the per-example loss gradient:
\begin{equation}\label{eq:Gdelta}
\|\nabla \ell(y f_{\bW}(\bx))\|_2\le G_\delta,
\qquad
\forall (\bx,y)\in\mathcal Z,\ \forall \bW .
\end{equation}
Under $m\gtrsim\log(m/\delta)$, by enlarging $C_{\sigma,b}$ if necessary, we further have $G_\delta^2\le 4C_{\sigma,b}p^3$. 
For convenience, denote $\beta:=C_{\sigma,b}p^3$. 

For an arbitrary clipping threshold $C_{\mathrm{clip}}>0$, define
\[
\rho_{\mathrm{clip}} := \min\left\{1,\frac{C_{\mathrm{clip}}}{G_\delta}\right\}
\in(0,1], \quad 
G_{\mathrm{clip}} := \min\{C_{\mathrm{clip}},G_\delta\}
= \rho_{\mathrm{clip}} G_\delta, \quad 
\beta_{\rm clip}:=\rho_{\rm clip}\beta.
\]
For each $(\bx,y)\in\mathcal Z$ and $\bW$, let
\[
\alpha(\bW;\bx,y) := \min\left\{ 1,
\frac{C_{\mathrm{clip}}} {\|\nabla\ell(yf_{\bW}(\bx))\|_2} \right\},
\]
with the convention $\alpha(\bW;\bx,y)=1$ when $\nabla\ell(yf_{\bW}(\bx))=0$. Then, on $\mathcal E_\delta$,
\[
\rho_{\mathrm{clip}}\le \alpha(\bW;\bx,y)\le 1,
\qquad
\big\| \alpha(\bW;\bx,y)\nabla\ell(yf_{\bW}(\bx)) \big\|_2
\le G_{\mathrm{clip}}.
\]
In particular, when $C_{\mathrm{clip}}= G_\delta$, we have
$\rho_{\mathrm{clip}}=1$ and recover the inactive-clipping setting as a special case.

The key idea of the proof is to combine the cancellation structure of the correlated perturbations with the localization argument for KANs. Following the proof sketch in Section~\ref{sec:proofsketch}, we organize the detailed proof into the following six steps.

\textsc{Step 1: Reduce to an auxiliary unprojected dynamics.}

The correlated perturbations break the conditional-centering argument, and the projection in Algorithm~\ref{alg:dp-lambda-minibatch} further obstructs the cross-iteration cancellation structure of the noise. 
To analyze this case, we introduce an auxiliary unprojected sequence \(\{\widetilde{\mathbf W}_t\}_{t=0}^{T}\), defined recursively by
\[
    \widetilde{\mathbf W}_0=\mathbf W_0,\qquad 
    \widetilde{\mathbf W}_t 
    = \widetilde{\mathbf W}_{t-1} 
    - \eta\Big(
    \frac1B\sum_{i\in\mathcal B_t} g^{\mathrm{aux}}_{t,i} 
    + \frac{C_{\mathrm{clip}}}{B}\xi_t
    \Big),
\]
where, for each \(t\in[T]\) and \(i\in[n]\), $ \alpha_{t,i} := \min\big\{
    1,  \frac{C_{\mathrm{clip}}}  { \|\nabla \ell (y_i f_{\widetilde{\mathbf W}_{t-1}}(x_i) ) \|_2}  \big\},$  with the convention \(\alpha_{t,i}=1\) when the gradient vanishes, and $g^{\mathrm{aux}}_{t,i} := \alpha_{t,i} \nabla \ell\big(y_i f_{\widetilde{\mathbf W}_{t-1}}(x_i)\big)$. 

Define the full-batch clipped gradient at the auxiliary iterate by $\bar g_t :=  \frac1n\sum_{i=1}^n g^{\mathrm{aux}}_{t,i}$, 
and the corresponding mini-batch fluctuation by  \[\widetilde{\Delta}_t   := \frac1B\sum_{i\in\mathcal B_t} g^{\mathrm{aux}}_{t,i} -  \bar g_t,\]
which measures the discrepancy between the mini-batch gradient and the full empirical gradient. 
Thus, conditioned on the history up to iteration \(t-1\), $\mathbb E\big[\widetilde{\Delta}_t\mid\mathcal F_{t-1}\big]=0.$ 

Starting from \(\widetilde{\bW}_0=\bW_0\) and \(Z_0=\mathbf 0\), the auxiliary unprojected iterate can therefore be rewritten as
\[ \widetilde{\bW}_t = \widetilde{\bW}_{t-1}
-\eta\big( \bar g_t + \widetilde{\bDelta}_t + \frac{C_{\mathrm{clip}}\kappa}{B}(Z_t-\lambda Z_{t-1}) \big), \qquad t\in[T]. \]

\textsc{Step 2: Introduce the shifted iterate.}

To obtain the high-probability estimate above, one would naturally try to apply a comparator recursion to the auxiliary trajectory. However, the correlated perturbation $Z_t-\lambda Z_{t-1}$ is not conditionally centered with respect to the natural filtration. The term involving \(Z_{t-1}\) is correlated with the current iterate and creates a first-order drift term.

To handle this obstruction, we define the shifted iterate
\[
    \bU_t =\widetilde{\bW}_t+\eta c_\mathrm{priv} Z_t, \qquad t=0,1,\dots,T.
\]
Here, $ c_\mathrm{priv} =\frac{C_{\mathrm{clip}}\kappa}{B}$.  
The role of $\bU_t$ is to absorb the current Gaussian perturbation into the state variable, which allows us to rewrite the correlated-noise dynamics in a form that is more amenable to one-step analysis.

\textsc{Step 3: Derive a shifted potential recursion using the local KAN geometry.}

We now start the main proof route for $\lambda>0$.
The key step is to absorb the current Gaussian noise into a shifted iterate so that the correlated perturbation can be handled at the \emph{single-step} level.   
\begin{lemma}[Shifted one-step identity]
\label{lem:shifted-one-step}
For any comparator $\bW^*\in\R^{mdp}$ and any $t\in[T-1]$,
\begin{align}
    \|\bU_t-\bW^*\|_2^2 \;&  =\|\bU_{t-1}-\bW^*\|_2^2 - 2\eta\big\langle \bar g_t,\widetilde{\bW}_{t-1} -  \bW^*\big\rangle\notag - 2\eta\big\langle \widetilde{\bDelta}_t,\bU_{t-1}-\bW^*\big\rangle \\
    &- 2(1-\lambda)\eta c_\mathrm{priv} \big\langle Z_{t-1},\bU_{t-1}-\bW^*\big\rangle  + \eta^2\|\bar g_t + \widetilde{\bDelta}_t\|_2^2 + (1-\lambda)^2 \eta^2c_\mathrm{priv}^2\|Z_{t-1}\|_2^2 \notag\\
    & + 2(1-\lambda)\eta^2c_\mathrm{priv} \big\langle \widetilde{\bDelta}_t,Z_{t-1}\big\rangle - 2\lambda\eta^2c_\mathrm{priv} \big\langle \bar g_t,Z_{t-1}\big\rangle .
    \label{eq:shifted-one-step-identity}
\end{align}
\end{lemma}
\begin{proof}
By the definitions of $\widetilde{\bW}_t$ and $\bU_t$, we have
\begin{align*}
    \bU_t & = \widetilde{\bW}_t+\eta c_\mathrm{priv}Z_t = \widetilde{\bW}_{t-1} -\eta\big(\bar g_t+\widetilde{\bDelta}_t+c_\mathrm{priv}(Z_t-\lambda Z_{t-1})\big) + \eta c_\mathrm{priv}Z_t \\
    & = \widetilde{\bW}_{t-1} - \eta(\bar g_t+\widetilde{\bDelta}_t) + \eta\lambda c_\mathrm{priv} Z_{t-1}.
\end{align*}
Plugging $\widetilde{\bW}_{t-1} = \bU_{t-1}-\eta c_\mathrm{priv} Z_{t-1}$ into the above equality, we get
\begin{align*}
    \bU_t &= \bU_{t-1}-\eta c_\mathrm{priv}Z_{t-1} -\eta(\bar g_t+\widetilde{\bDelta}_t) + \eta\lambda c_\mathrm{priv}Z_{t-1} \\
    &= \bU_{t-1} - \eta\big(\bar g_t+\widetilde{\bDelta}_t+(1-\lambda)c_\mathrm{priv}Z_{t-1}\big).
\end{align*}
Hence,
\[
    \bU_t-\bW^* = \bU_{t-1}-\bW^* - \eta\big( \bar g_t+\widetilde{\bDelta}_t+(1-\lambda)c_\mathrm{priv}Z_{t-1} \big).
\]
It then follows
\begin{align*}
&\|\bU_t-\bW^*\|_2^2\\
&\!\!=\! 
\|\bU_{t-1}\!-\!\bW^*\|_2^2
\!-\!2\eta
\big\langle
\bar g_t\!+\!\widetilde{\bDelta}_t+(1\!-\!\lambda)c_\mathrm{priv}Z_{t-\!1},
\bU_{t\!-\!1}\!-\!\bW^*
\big\rangle
\! +\!
\eta^2
\big\|
\bar g_t\!+\!\widetilde{\bDelta}_t \!+\! (1\!-\!\lambda)c_\mathrm{priv}Z_{t\!-\!1}
\big\|_2^2 \\
&\!\!=\! 
\|\bU_{t\!-\!1}\!-\!\bW^*\|_2^2
\!-\!2\eta\big\langle \bar g_t,\widetilde{\bW}_{t-1}\!\!-\!\bW^*\big\rangle  \!-\!2\eta\big\langle \widetilde{\bDelta}_t,\bU_{t-1}\!\!-\!\bW^*\big\rangle
\!-\!2(1\!-\!\lambda)\eta c_\mathrm{priv}
\big\langle Z_{t\!-\!1},\bU_{t-1}\!\!-\!\bW^*\big\rangle \\
&  +\!\eta^2\|\bar g_t\!+\!\widetilde{\bDelta}_t\|_2^2
+(1-\lambda)^2\eta^2c_\mathrm{priv}^2\|Z_{t-1}\|_2^2  +2(1-\lambda)\eta^2c_\mathrm{priv}
\big\langle \widetilde{\bDelta}_t,Z_{t-1}\big\rangle
-2\lambda\eta^2c_\mathrm{priv}
\big\langle \bar g_t,Z_{t-1}\big\rangle,
\end{align*}
where we have used $-2\eta
\big\langle \bar g_t,\bU_{t-1}-\bW^*\big\rangle
= -2\eta \big\langle \bar g_t,\widetilde{\bW}_{t-1}-\bW^*\big\rangle
-2\eta^2c_\mathrm{priv} \big\langle \bar g_t,Z_{t-1}\big\rangle $ implied by $\bU_{t-1} = \widetilde{\bW}_{t-1}+\eta c_\mathrm{priv} Z_{t-1}$. The proof is complete. 
\end{proof}

For a fixed comparator $\bW^*\in\R^{mdp}$ and a shifted radius $\bar{R}>0$, define
\[ A_{\bar{R}}^{U}(t)
 = \Big\{ \max_{0\le s\le t}\|\bU_s-\bW^*\|_2\le \bar{R} \Big\}
\qquad \text{and} \qquad
A_{\bar{R}}^{U} =A_{\bar{R}}^{U}(T). \]
For $z>0$ and $V_Z>0$, define the Gaussian event
\[ G_Z(z,V_Z)  = \Big\{ \max_{0\le t\le T-1}\|Z_t\|_2\le z
\quad\text{and}\quad \sum_{t=1}^T\|Z_{t-1}\|_2^2\le V_Z \Big\}. \]
For $V_\Delta>0$, define the stopped mini-batch quadratic event
\[ G_{\Delta^2}(\bar{R},V_\Delta)  = \Big\{ \sum_{t=1}^T \|\widetilde{\bDelta}_t\|_2^2 \mathbf 1_{A_{\bar{R}}^{U}(t-1)}
\le V_\Delta \Big\}. \]
Define the stopped potential fluctuation process
\begin{align*}
    M_k &= 
\sum_{t=1}^k  \mathbf 1_{A_{\bar{R}}^{U}\!(t-\!1)}
\Big[ -2\eta\big\langle \widetilde{\bDelta}_t,\bU_{t-\!1}-\bW^*\big\rangle
-2(1-\lambda)\eta c_\mathrm{priv}\big\langle Z_{t-\!1},\bU_{t-\!1}-\bW^*\big\rangle \\
&\qquad +2(1-\lambda)\eta^2c_\mathrm{priv}\big\langle \widetilde{\bDelta}_t,Z_{t-\!1}\big\rangle
\Big],
\end{align*}
and define
\[ G_{\mathrm{pot}}(\bar{R},  M) := \Big\{ \max_{1\le k\le T} M_k\le   M \Big\}. \]
Finally, define the \textit{shifted good event}
\[ \mathcal{E}_{\rm good}  = G_Z(z,V_Z) \cap G_{\Delta^2}(\bar{R},V_\Delta)
\cap G_{\mathrm{pot}}(\bar{R},  M)  \]
and the frozen weighted empirical loss
\[ \mathcal L_{S,t}^{\alpha}(\bW)
:= \frac1n\sum_{i=1}^n \alpha_{t,i}\, \ell(y_i f_{\bW}(x_i)), \]
where the coefficients $\{\alpha_{t,i}\}_{i=1}^n$ are those defined in Step~1 at the auxiliary iterate $\widetilde{\bW}_{t-1}$ and are treated as fixed with respect to $\bW$.

\begin{lemma}[Localized comparator under shifted localization]
\label{lem:shifted-localized-comparator}
Suppose Assumption~\ref{ass:sigma} holds and the event $\mathcal E_\delta$ occurs.
Let $\bW^*\in\R^{mdp}$ be fixed.  
On the events $A^U_{\bar{R}}(t-1)$ and $G_Z(z,V_Z)$, assume
\begin{equation}
\label{eq:shifted-localized-comparator-width}
m \gtrsim C_{\sigma,b}^2p^3 \bigl(\log(m/\delta)+p\bigr) \bigl(\bar{R}+\eta c_\mathrm{priv}z\bigr)^4.
\end{equation}
Then
\begin{equation*}
\frac23\,
\mathcal L_{S,t}^{\alpha}(\widetilde{\bW}_{t-1})
\le
\big\langle
\bar g_t,
\widetilde{\bW}_{t-1}-\bW^*
\big\rangle
+
\frac43\,\mathcal L_S(\bW^*).
\end{equation*}
Furthermore,
\begin{equation*}
\frac{2\rho_{\rm clip}}{3}\,
\mathcal L_S(\widetilde{\bW}_{t-1})
\le
\big\langle
\bar g_t,
\widetilde{\bW}_{t-1}-\bW^*
\big\rangle
+
\frac43\,\mathcal L_S(\bW^*).
\end{equation*}
\end{lemma}
\begin{proof}

First, by Lemma~\ref{pro:smooth}, for every data point $(\bx_i,y_i)$ and every $\bV\in\R^{mdp}$,
\[ \lambda_{\min}
\left( \nabla^2 \ell(y_i f_{\bV}(\bx_i)) \right)
\ge -\nu \, \ell(y_i f_{\bV}(\bx_i)) \]
with $\nu  = \frac{ C_{\sigma,b}p^{\frac32}\bigl(\sqrt{\log(m/\delta})+\sqrt p\bigr)}{\sqrt m}.$
Since the coefficients $\{\alpha_{t,i}\}_{i=1}^n$ are fixed and nonnegative,
\[ \lambda_{\min}
\left( \nabla^2\mathcal L_{S,t}^{\alpha}(\bV) \right)
\ge -\nu\, \mathcal L_{S,t}^{\alpha}(\bV). \]
Moreover,
\[ \nabla \mathcal L_{S,t}^{\alpha}(\widetilde{\bW}_{t-1})
= \frac1n\sum_{i=1}^n \alpha_{t,i}
\nabla\ell(y_i f_{\widetilde{\bW}_{t-1}}(x_i))
= \bar g_t. \]
Since $\widetilde{\bW}_{t-1}
=
\bU_{t-1}-\eta c_\mathrm{priv}Z_{t-1}$, 
we have $\|\widetilde{\bW}_{t-1}-\bW^*\|_2
\le
\bar{R}+\eta c_\mathrm{priv}z$. 
Taking $D=\bar{R}+\eta c_\mathrm{priv}z$, 
the width condition guarantees 
$\nu D^2\le \frac12$. 

Recall that we defined the line segment $[\widetilde{\bW}_{t-1},\bW^*] = \{\alpha \widetilde{\bW}_{t-1}+ (1-\alpha)\bW^*: \alpha\in[0,1]\}$.
Applying Taylor's theorem, there exists a point $\bV\in[\widetilde{\bW}_{t-1},\bW^*]$ such that
\begin{align*}
\mathcal L_{S,t}^{\alpha}(\bW^*)
&\ge \mathcal L_{S,t}^{\alpha}(\widetilde{\bW}_{t-1})
+ \left\langle \nabla\mathcal L_{S,t}^{\alpha}(\widetilde{\bW}_{t-1}), \bW^*-\widetilde{\bW}_{t-1} \right\rangle  
- \frac{\nu}{2} \mathcal L_{S,t}^{\alpha}(\bV) \|\widetilde{\bW}_{t-1}-\bW^*\|_2^2.
\end{align*} 
Hence,
\begin{align}\label{eq:shifted-comp-taylor}
\mathcal L_{S,t}^{\alpha}(\widetilde{\bW}_{t-1})
- \mathcal L_{S,t}^{\alpha}(\bW^*)
\le\;& \left\langle
\bar g_t, \widetilde{\bW}_{t-1}-\bW^* \right\rangle  + \frac{\nu D^2}{2} \max_{\bV\in[\widetilde{\bW}_{t-1},\bW^*]} \mathcal L_{S,t}^{\alpha}(\bV).
\end{align}

On the other hand,   Lemma~\ref{lem:quasi-convexity}  with $\lambda_{\min}
\left( \nabla^2\mathcal L_{S,t}^{\alpha}(\bV) \right)
\ge -\nu \mathcal L_{S,t}^{\alpha}(\bV) $  gives
\[ \max_{\bV\in[\widetilde{\bW}_{t-1},\bW^*]}
\mathcal L_{S,t}^{\alpha}(\bV)
\le \rho \max\{\mathcal L_{S,t}^{\alpha}(\widetilde{\bW}_{t-1}),
\mathcal L_{S,t}^{\alpha}(\bW^*) \}, \]
where $\rho  = \big(1-\frac{\nu D^2}{2}\big)^{-1}$. 
Note $\nu D^2\le 1/2$, then $\frac{\nu D^2}{2}\rho = \frac{\nu D^2}{2-\nu D^2} \le \frac13.$
Therefore, combining the above inequality with \eqref{eq:shifted-comp-taylor} yields
\begin{align*}
\mathcal L_{S,t}^{\alpha}(\widetilde{\bW}_{t-1})-\mathcal L_{S,t}^{\alpha}(\bW^*)
&\le \big\langle \nabla\mathcal L_{S,t}^{\alpha}(\widetilde{\bW}_{t-1}),
\widetilde{\bW}_{t-1}-\bW^*
\big\rangle + \frac13 \big( \mathcal L_{S,t}^{\alpha}(\widetilde{\bW}_{t-1})+\mathcal L_{S,t}^{\alpha}(\bW^*) \big).
\end{align*}
Rearranging gives
\begin{align}\label{eq:emploss-traj}
\frac23\,\mathcal L_{S,t}^{\alpha}(\widetilde{\bW}_{t-1})
\le \big\langle \nabla\mathcal L_{S,t}^{\alpha}(\widetilde{\bW}_{t-1}),
\widetilde{\bW}_{t-1}-\bW^* \big\rangle
+ \frac43\,\mathcal L_{S,t}^{\alpha}(\bW^*).    
\end{align}
Since $\alpha_{t,i}\le 1$ and the loss is nonnegative, it holds $\mathcal L_{S,t}^{\alpha}(\bW^*) \le \mathcal L_S(\bW^*)$. 
Moreover, on $\mathcal E_\delta$, $\alpha_{t,i}\ge\rho_{\rm clip}$ for every
$i\in[n]$, and hence $\mathcal L_{S,t}^{\alpha}(\widetilde{\bW}_{t-1}) \ge \rho_{\rm clip}\, \mathcal L_S(\widetilde{\bW}_{t-1})$. 
Combining the above inequalities proves
\[ \frac{2\rho_{\rm clip}}{3} \mathcal L_S(\widetilde{\bW}_{t-1})
\le \big\langle \bar g_t, \widetilde{\bW}_{t-1}-\bW^* \big\rangle
+ \frac43\mathcal L_S(\bW^*). \]
This completes the proof.
\end{proof}

We next show that, on the shifted good event $\mathcal{E}_{\rm good}$, the shifted iterate remains localized and yields a pathwise empirical loss bound. In particular, the projection becomes inactive under a suitable choice of $R_*$.
\begin{lemma}[Conditional shifted-potential bootstrap and loss bound]
\label{lem:conditional-shifted-bootstrap-loss}
Suppose $0<\lambda<1$, Assumption~\ref{ass:sigma} holds and the event $\mathcal E_\delta$ occurs.
Let $\beta=C_{\sigma,b}p^3$ and $\beta_{\rm clip}=\rho_{\rm clip}\beta$.
Fix a comparator $\bW^*$ and a shifted radius $\bar{R}>0$, and assume
$\eta\le \frac{1}{12\beta_{\rm clip}}$, \eqref{eq:shifted-localized-comparator-width}, and \begin{equation}
\label{eq:conditional-shifted-self-consistency}
\|\bW_0-\bW^*\|_2^2
+ \frac{8}{3}\eta T\mathcal L_S(\bW^*)
+ M + 2\eta^2V_\Delta + \big((1-\lambda)^2\eta^2c_\mathrm{priv}^2
+12\beta_{\rm clip}\lambda^2\eta^3c_\mathrm{priv}^2\big)V_Z
\le \bar{R}^2.
\end{equation}
Then, on the good event $\mathcal{E}_{\rm good}$, the localization event
$A^U_{\bar{R}}$ also occurs and
\begin{align*}
\frac1T\!\sum_{t=1}^T\!
\mathcal L_{S,t}^{\alpha}(\widetilde{\bW}_{t-1})
\!\le\! 8\mathcal L_S(\bW^*)\! +\! \frac{\!3(\|\bW_0\!-\!\bW^*\|_2^2\!+\!M)\!}{\eta T}
 \!+\! \frac{\!6\eta V_\Delta\!}{T} \!+\! 3\Big(\!(1\!-\!\lambda)^2\eta \!+\! 12\beta_{\rm clip}\lambda^2\eta^2\!\Big) \frac{\!c_\mathrm{priv}^2\!V_Z\!}{T}.
\end{align*}
If we further assume $R_* \ge \bar{R} + \|\bW^*-\bW_0\|_2 + \eta c_\mathrm{priv}z$, 
then
\begin{align*}
\frac1T\sum_{t=1}^T
\mathcal L_S(\bW_{t-1})
\le\frac{1}{\rho_{\rm clip}}
\Big[ 8\mathcal L_S(\bW^*)  +&\, \frac{3( \|\bW_0-\bW^*\|_2^2+M)}{\eta T}
   + \frac{6\eta V_\Delta}{T}\\
&+ 3\Big((1-\lambda)^2\eta +12\beta_{\rm clip}\lambda^2\eta^2\Big) c_\mathrm{priv}^2\frac{V_Z}{T}
\Big].
\end{align*}
\end{lemma}

\begin{proof}
Fix an outcome in $\mathcal{E}_{\rm good}$.
We first show that the shifted localization event $A_{\bar{R}}^{U}$ holds.
The proof is done by contradiction. 
Suppose $A_{\bar{R}}^{U}$ does not hold. Since $\|\bU_0-\bW^*\|_2^2
= \|\bW_0-\bW^*\|_2^2 \le \bar{R}^2$  by \eqref{eq:conditional-shifted-self-consistency}, the first exit time
\[ \tau = \min\Big\{ t\in[T]: \|\bU_t-\bW^*\|_2>\bar{R} \Big\} \]
is well-defined. Then $\tau\in[T]$ and
$A_{\bar{R}}^{U}(\tau-1)$ holds. Hence, for every $1\le t\le\tau$, $\mathbf 1_{A_{\bar{R}}^{U}(t-1)}=1.$ 

For every $t\le\tau$, since
$A_{\bar{R}}^{U}(t-1)$ and $G_Z(z,V_Z)$ hold, we have
\[ \|\widetilde{\bW}_{t-1}-\bW^*\|_2
= \|\bU_{t-1}-\eta c_\mathrm{priv}Z_{t-1}-\bW^*\|_2
\le \bar{R}+\eta c_\mathrm{priv}z. \]

By Lemma~\ref{lem:shifted-one-step}, for every $t\le\tau$,
\begin{align*}
\|\bU_t-\bW^*\|_2^2
& =\|\bU_{t-1}-\bW^*\|_2^2 -2\eta
\big\langle \bar g_t, \widetilde{\bW}_{t-1}-\bW^* \big\rangle -2\eta \big\langle \widetilde{\bDelta}_t,
\bU_{t-1}-\bW^* \big\rangle\\
&-2(1-\lambda)\eta c_\mathrm{priv}
\big\langle Z_{t-1}, \bU_{t-1}-\bW^* \big\rangle +\eta^2
\|\bar g_t+\widetilde{\bDelta}_t\|_2^2
+ (1-\lambda)^2\eta^2c_\mathrm{priv}^2 \|Z_{t-1}\|_2^2\\
&+ 2(1-\lambda)\eta^2c_\mathrm{priv}
\big\langle \widetilde{\bDelta}_t,Z_{t-1}
\big\rangle - 2\lambda\eta^2c_\mathrm{priv}
\big\langle \bar g_t,Z_{t-1} \big\rangle.
\end{align*}

According to Lemma~\ref{lem:shifted-localized-comparator}, it holds
\[ -2\eta \big\langle \bar g_t, \widetilde{\bW}_{t-1}-\bW^* \big\rangle
\le -\frac{4\eta}{3} \mathcal L_{S,t}^{\alpha}(\widetilde{\bW}_{t-1})
+ \frac{8\eta}{3} \mathcal L_S(\bW^*). \]

Moreover, by Jensen's inequality,
\begin{align*}
\|\bar g_t\|_2^2
&= \Big\| \frac1n\sum_{i=1}^n \alpha_{t,i} \nabla\ell \big( y_i f_{\widetilde{\bW}_{t-1}}(x_i) \big) \Big\|_2^2\ \le \frac1n\sum_{i=1}^n \alpha_{t,i}^2 
\Big\| \nabla\ell \big( y_i f_{\widetilde{\bW}_{t-1}}(x_i) \big) \Big\|_2^2.
\end{align*}
For each $i\in[n]$, denote $\nabla\ell_{t,i}:=\nabla\ell \big( y_i f_{\widetilde{\bW}_{t-1}}(x_i) \big)$. 
By the definition of $\alpha_{t,i}$ and \eqref{eq:Gdelta},
\[
\alpha_{t,i}\|\nabla\ell_{t,i}\|_2
=\min\big\{\|\nabla\ell_{t,i}\|_2,C_{\rm clip}\big\}
\le G_{\rm clip}.
\]
Hence,
\[
\alpha_{t,i}^2\|\nabla\ell_{t,i}\|_2^2
\le G_{\rm clip}\alpha_{t,i}\|\nabla\ell_{t,i}\|_2.
\]
Moreover, using $|\ell'|\le\ell$ and \eqref{eq:partial_a_norm},
\[
\|\nabla\ell_{t,i}\|_2
= \big|\ell'\big(y_i f_{\widetilde{\bW}_{t-1}}(x_i)\big)\big|
\big\|\nabla f_{\widetilde{\bW}_{t-1}}(x_i)\big\|_2
\le G_\delta\,\ell\big(y_i f_{\widetilde{\bW}_{t-1}}(x_i)\big).
\]
Combining the above estimates, $G_{\rm clip}=\rho_{\rm clip}G_\delta$, and
$G_\delta^2\le4\beta$, we obtain
\begin{align*}
\|\bar g_t\|_2^2
&\le G_{\rm clip}G_\delta\frac1n\sum_{i=1}^n
\alpha_{t,i}\ell\big(y_i f_{\widetilde{\bW}_{t-1}}(x_i)\big) 
=\rho_{\rm clip}G_\delta^2\mathcal L_{S,t}^{\alpha}(\widetilde{\bW}_{t-1}) \le 4\beta_{\rm clip}\mathcal L_{S,t}^{\alpha}(\widetilde{\bW}_{t-1}).
\end{align*}
Hence,
\begin{align*}
\eta^2 \|\bar g_t+\widetilde{\bDelta}_t\|_2^2
 \le 2\eta^2\|\bar g_t\|_2^2
+ 2\eta^2\|\widetilde{\bDelta}_t\|_2^2
\le 8\beta_{\rm clip}\eta^2 \mathcal L_{S,t}^{\alpha}(\widetilde{\bW}_{t-1})
+ 2\eta^2\|\widetilde{\bDelta}_t\|_2^2.
\end{align*}

Furthermore, from Young's inequality $2ab\le \rho a^2+\rho^{-1}b^2$ 
with $a=\sqrt{\eta}\,\|\bar g_t\|_2,  b= \lambda\eta^{3/2}
c_\mathrm{priv} \|Z_{t-1}\|_2 $ and $ \rho=\frac{1}{12\beta_{\rm clip}}$, 
we obtain
\begin{align*}
-2\lambda\eta^2c_\mathrm{priv}
\langle \bar g_t,Z_{t-1}\rangle
&\le 2\lambda\eta^2c_\mathrm{priv}
\big| \langle \bar g_t,Z_{t-1}\rangle \big|
\le \frac{\eta}{12\beta_{\rm clip}} \|\bar g_t\|_2^2
+ 12\beta_{\rm clip}\lambda^2\eta^3 c_\mathrm{priv}^2 \|Z_{t-1}\|_2^2\\
&\le \frac{\eta}{3} \mathcal L_{S,t}^{\alpha}(\widetilde{\bW}_{t-1})
+ 12\beta_{\rm clip}\lambda^2\eta^3 c_\mathrm{priv}^2 \|Z_{t-1}\|_2^2.
\end{align*} 
Combining the above estimates with
$-\frac{4\eta}{3} + 8\beta_{\rm clip}\eta^2 + \frac{\eta}{3}
= -\eta+8\beta_{\rm clip}\eta^2 \le -\frac{\eta}{3}$, 
where the last inequality follows from $\eta\le\frac{1}{12\beta_{\rm clip}}$, gives
\begin{align*}
\|\bU_t-\bW^*\|_2^2
\le\;& \|\bU_{t-1}-\bW^*\|_2^2
- \frac{\eta}{3} \mathcal L_{S,t}^{\alpha}(\widetilde{\bW}_{t-1})
+ \frac{8\eta}{3} \mathcal L_S(\bW^*) -2\eta \big\langle \widetilde{\bDelta}_t, \bU_{t-1}-\bW^* \big\rangle\\
&-2(1-\lambda)\eta c_\mathrm{priv} \big\langle Z_{t-1}, \bU_{t-1}-\bW^* \big\rangle  + 2(1-\lambda)\eta^2c_\mathrm{priv} \big\langle
\widetilde{\bDelta}_t,Z_{t-1} \big\rangle\\
&+ 2\eta^2 \|\widetilde{\bDelta}_t\|_2^2  + \Big( (1-\lambda)^2\eta^2c_\mathrm{priv}^2
+ 12\beta_{\rm clip}\lambda^2\eta^3c_\mathrm{priv}^2 \Big) \|Z_{t-1}\|_2^2.
\end{align*}

Summing from $t=1$ to $\tau$ and using $\mathbf 1_{A_{\bar{R}}^{U}(t-1)}=1$ for every $1\le t\le\tau$, we obtain
\begin{align} \label{eq:conditional-summed-potential}
\|\bU_\tau-\bW^*\|_2^2
+ \frac{\eta}{3}
\sum_{t=1}^{\tau}
\mathcal L_{S,t}^{\alpha}(\widetilde{\bW}_{t-1})
 &\le \|\bW_0-\bW^*\|_2^2
+ \frac{8}{3}\eta T\mathcal L_S(\bW^*)
+ M_\tau  + 2\eta^2 \sum_{t=1}^{\tau} \|\widetilde{\bDelta}_t\|_2^2 \nonumber\\
&+ \Big( (1-\lambda)^2\eta^2c_\mathrm{priv}^2
+ 12\beta_{\rm clip}\lambda^2\eta^3c_\mathrm{priv}^2 \Big) \sum_{t=1}^{\tau} \|Z_{t-1}\|_2^2.
\end{align} 
On $G_{\mathrm{pot}}(\bar{R},M)$, $G_{\Delta^2}(\bar{R},V_\Delta)$, and $G_Z(z,V_Z)$, since the stopping indicators equal one for all $t\le\tau$, it holds that 
\[M_\tau\le M  \quad  \text{ and }  \quad  \sum_{t=1}^{\tau} \|\widetilde{\bDelta}_t\|_2^2
\le \sum_{t=1}^{T} \|\widetilde{\bDelta}_t\|_2^2
\mathbf 1_{A_{\bar{R}}^{U}(t-1)} \le V_\Delta 
\quad  \text{ and }  \quad 
\sum_{t=1}^{\tau} \|Z_{t-1}\|_2^2
\le V_Z. \]
Hence,
\begin{align*}
 &\|\bU_\tau-\bW^*\|_2^2
+ \frac{\eta}{3} \sum_{t=1}^{\tau} \mathcal L_{S,t}^{\alpha}(\widetilde{\bW}_{t-1}) \\
&\!\le \! \|\bW_0\!-\!\bW^*\|_2^2 + \frac{8}{3}\eta T\mathcal L_S(\bW^*)
+ \!M\! + 2\eta^2V_\Delta \!  + \!\Big( \!(1-\lambda)^2\eta^2c_\mathrm{priv}^2+ 12\beta_{\rm clip}\lambda^2\eta^3c_\mathrm{priv}^2 \!\Big)V_Z \le \bar{R}^2.
\end{align*}
Therefore, $\|\bU_\tau-\bW^*\|_2^2 \le \bar{R}^2$, 
which contradicts the definition of $\tau$.
Hence, $A_{\bar{R}}^{U}$ holds.

Under $A_{\bar{R}}^{U}$, using \eqref{eq:conditional-summed-potential} with $\tau=T$ and dropping the nonnegative term $\|\bU_T-\bW^*\|_2^2$, we obtain
\begin{align*}
\frac{\eta}{3} \sum_{t=1}^{T} \mathcal L_{S,t}^{\alpha}(\widetilde{\bW}_{t-1})
\le\;& \|\bW_0-\bW^*\|_2^2
+ \frac{8}{3}\eta T\mathcal L_S(\bW^*)
+ M + 2\eta^2V_\Delta\\
&+ \Big( (1-\lambda)^2\eta^2c_\mathrm{priv}^2
+ 12\beta_{\rm clip}\lambda^2\eta^3c_\mathrm{priv}^2 \Big)V_Z.
\end{align*}
Dividing both sides by $\eta T/3$ yields
\begin{align*}
\frac1T\!\sum_{t=1}^T\!
\mathcal L_{S,t}^{\alpha}(\widetilde{\bW}_{t-1})
\!\le\! 8\mathcal L_S(\bW^*)\! +\! \frac{\!3(\|\bW_0\!-\!\bW^*\|_2^2\!+\!M)\!}{\eta T}
 \!+\! \frac{\!6\eta V_\Delta\!}{T} \!+\! 3\Big(\!(1\!-\!\lambda)^2\eta \!+\! 12\beta_{\rm clip}\lambda^2\eta^2\!\Big) \frac{\!c_\mathrm{priv}^2\!V_Z\!}{T}.
\end{align*}
It remains to prove the empirical loss bound for the projected iterates.
On $A_{\bar{R}}^{U}\cap G_Z(z,V_Z)$, for every
$t=0,1,\dots,T$,
\begin{align*}
\|\widetilde{\bW}_t-\bW_0\|_2
&= \|\bU_t-\eta c_\mathrm{priv}Z_t-\bW_0\|_2
\le \|\bU_t-\bW^*\|_2 + \|\bW^*-\bW_0\|_2 + \eta c_\mathrm{priv}\|Z_t\|_2\\
&\le \bar{R} + \|\bW^*-\bW_0\|_2 + \eta c_\mathrm{priv}z.
\end{align*}
Thus, if $R_* \ge \bar{R} + \|\bW^*-\bW_0\|_2 + \eta c_\mathrm{priv}z$, 
then $ \widetilde{\bW}_t \in \mathcal K =
\mathcal B(\bW_0,R_*)$ for all $t=0,1,\dots,T$.  

By induction, the projected and auxiliary unprojected iterates coincide.
The claim is trivial at $t=0$. Suppose $\bW_{t-1} = \widetilde{\bW}_{t-1}$. 
Then the clipping coefficients, and hence the clipped mini-batch gradients, in the projected and auxiliary updates coincide. 
Therefore, the two updates have the same pre-projection point, namely $\widetilde{\bW}_t$.
Since $\widetilde{\bW}_t\in\mathcal K$, projection leaves it unchanged:
\[ \bW_t = \Pi_{\mathcal K}(\widetilde{\bW}_t) = \widetilde{\bW}_t \quad  \text{ and further } \quad  \bW_t = \widetilde{\bW}_t, \qquad \forall\, t=0,1,\dots,T. \]

Since $\alpha_{t,i}\ge\rho_{\rm clip}$ on $\mathcal E_\delta$ for every
$i\in[n]$, we have
\[ \mathcal L_{S,t}^{\alpha}(\widetilde{\bW}_{t-1})
\ge \rho_{\rm clip}\, \mathcal L_S(\widetilde{\bW}_{t-1})
= \rho_{\rm clip}\, \mathcal L_S(\bW_{t-1}). \]
Therefore, the preceding bound on the weighted empirical loss yields
\begin{align*}
\frac1T\sum_{t=1}^T
\mathcal L_S(\bW_{t-1})
\le\frac{1}{\rho_{\rm clip}}
\Big[ 8\mathcal L_S(\bW^*)  +&\, \frac{3( \|\bW_0-\bW^*\|_2^2+M)}{\eta T}
   + \frac{6\eta V_\Delta}{T}\\
&+ 3\Big((1-\lambda)^2\eta +12\beta_{\rm clip}\lambda^2\eta^2\Big) c_\mathrm{priv}^2\frac{V_Z}{T}
\Big].
\end{align*}
This completes the proof.
\end{proof}

\textsc{Step 4: High-probability control and a stopped optimization bound.}

Now, we show that the good event $\mathcal{E}_{\rm good}$ holds with high probability.   We need the following Laurent--Massart chi-square tail bound \cite{laurent2000adaptive}. 
\begin{lemma}
\label{lem:laurent-massart-gaussian}
Let \(G\sim \mathcal N(0,\mathbf{I}_d)\). Then \(\|G\|_2^2\sim \chi_d^2\), and   for every \(\delta\in(0,1)\),
\[ \mathbb P \left( \|G\|_2^2 \le d+2\sqrt{d\log(1/\delta)}+2\log(1/\delta) \right)\ge 1-\delta.\]
\end{lemma}
\begin{lemma}[High-probability control of the Gaussian noises]
\label{lem:good-event-Z}
 Suppose $Z_1,\dots,Z_T$ are i.i.d. standard Gaussian vectors in $\R^{mdp}$, and let $Z_0=\mathbf 0$.
For any $\delta_Z\in(0,1)$, define $z_{\delta_Z} = \sqrt{mdp}+\sqrt{2\log\big(\frac{2T}{\delta_Z}\big)}$  and $V_{Z,\delta_Z} = (T-1){mdp} + 2\sqrt{(T-1){mdp}\log\big(\frac{2}{\delta_Z}\big)} + 2\log\big(\frac{2}{\delta_Z}\big).$ 
Then,
\[ \mathbb P\left( G_Z(z_{\delta_Z},V_{Z,\delta_Z})
\right) \ge 1-\delta_Z. \]
\end{lemma}

\begin{proof}
We first control the maximum Gaussian norm.
For a standard Gaussian vector $Z_t\sim \mathcal N(0,\bfI_{{mdp}})$, the standard Gaussian norm concentration inequality gives, for each $t=1,\dots,T$,
\[ \mathbb P\Big(  \|Z_t\|_2> \sqrt{mdp}+\sqrt{2\log\big(\frac{2T}{\delta_Z}\big)} \Big)
\le \frac{\delta_Z}{2T}. \]
Since $Z_0=\mathbf 0$, taking a union bound over $t=1,\dots,T$ yields
\[ \mathbb P\big( \max_{0\le t\le T}\|Z_t\|_2
\le z_{\delta_Z} \big)
\ge 1-\frac{\delta_Z}{2}. \]

Next, we control the quadratic sum. Note $Z_0=\mathbf 0$, it holds
\[ \sum_{t=1}^T\|Z_{t-1}\|_2^2
= \sum_{s=0}^{T-1}\|Z_s\|_2^2
= \sum_{s=1}^{T-1}\|Z_s\|_2^2. \]
The random variable on the right-hand side follows a chi-square distribution with $K=(T-1){mdp}$ degrees of freedom. By the Laurent--Massart chi-square tail bound (see Lemma~\ref{lem:laurent-massart-gaussian}), it holds
\[ \mathbb P\Big( \sum_{t=1}^T\|Z_{t-1}\|_2^2
\le (T-1){mdp}
+ 2\sqrt{(T-1){mdp}\log\big(\frac{2}{\delta_Z}\big)}
+ 2\log\big(\frac{2}{\delta_Z}\big)\Big)
\ge 1-\frac{\delta_Z}{2}. \]
Hence, 
\[ \mathbb P\Big( \sum_{t=1}^T\|Z_{t-1}\|_2^2
\le V_{Z,\delta_Z} \Big)
\ge 1-\frac{\delta_Z}{2}. \]

Combining the two estimates by a union bound gives
\[ \mathbb P\left( G_Z(z_{\delta_Z},V_{Z,\delta_Z}) \right)\ge 1-\delta_Z. \]
This completes the proof.
\end{proof}

\begin{lemma}[High-probability control of the stopped mini-batch quadratic variation]
\label{lem:good-event-delta-square}
Suppose Assumption~\ref{ass:sigma} holds and the event $\mathcal E_\delta$ occurs.
Fix $\bar{R}>0$ and $\delta_\Delta\in(0,1)$. Define $V_{\Delta,\delta_\Delta} = \frac{2T G_{\rm clip}^2}{B} + 8G_{\rm clip}^2 \log\big(\frac{1}{\delta_\Delta}\big)$. 
Then, conditioned on the dataset and the initialization satisfying $\mathcal E_\delta$, it holds that
\[ \mathbb P_{\mathcal A}
\left( G_{\Delta^2}(\bar{R},V_{\Delta,\delta_\Delta}) \right)
\ge 1-\delta_\Delta. \]
\end{lemma}

\begin{proof}
Condition on the dataset $S$ and on the initialization satisfying $\mathcal E_\delta$.
Let  $\mathscr F_t = \sigma\bigl( \bW_0,\bc, \mathcal B_1,\dots,\mathcal B_t, Z_1,\dots,Z_t \bigr)$ be the natural filtration generated by the algorithmic randomness up to time $t$.
For each $t\in[T]$, define $ X_t = \|\widetilde{\bDelta}_t\|_2^2 \mathbf 1_{A^U_{\bar{R}}(t-1)}$.

We first bound the conditional expectation of $X_t$.
Conditioned on $\mathscr F_{t-1}$, the iterate $\widetilde{\bW}_{t-1}$ is fixed. Hence the clipped gradients $\{g^{\mathrm{aux}}_{t,i}\}_{i=1}^n$ are deterministic.
Moreover, on $\mathcal E_\delta$,
\[ \|g^{\mathrm{aux}}_{t,i}\|_2
= \big\| \alpha_{t,i} \nabla\ell
\bigl( y_i f_{\widetilde{\bW}_{t-1}}(\bx_i) \bigr) \big\|_2 
\le G_{\rm clip}, \qquad \forall i\in[n]. \]
Therefore, by Lemma~\ref{lem:finite-pop-var}, it holds that $\mathbb E \big[ \|\widetilde{\bDelta}_t\|_2^2 \,\big|\, \mathscr F_{t-1} \big] \le \frac{G_{\rm clip}^2}{B}$. 
Thus,
\[ \mathbb E[X_t\mid\mathscr F_{t-1}]
= \mathbf 1_{A^U_{\bar{R}}(t-1)} \mathbb E
\big[ \|\widetilde{\bDelta}_t\|_2^2
\,\big|\, \mathscr F_{t-1} \big]
\le \frac{G_{\rm clip}^2}{B}. \]
Summing over $t$, we get
\[ \sum_{t=1}^T \mathbb E[X_t\mid\mathscr F_{t-1}]
\le \frac{T G_{\rm clip}^2}{B} =:V. \]

Next, we bound $X_t$ almost surely.
From the definition, we know $\big\| \frac1B \sum_{i\in\mathcal B_t} g^{\mathrm{aux}}_{t,i} \big\|_2 \le G_{\rm clip}$ and  $\|\bar g_t\|_2 = \big\| \frac1n \sum_{i=1}^n g^{\mathrm{aux}}_{t,i} \big\|_2 \le G_{\rm clip}$, then 
\[ \|\widetilde{\bDelta}_t\|_2
\le 2G_{\rm clip}, \qquad X_t \le 4G_{\rm clip}^2 =:L. \]

According to Lemma~\ref{lem:adapted-chernoff} with $V=\frac{T G_{\rm clip}^2}{B}$ and  $L=4G_{\rm clip}^2$, with probability at least $1-\delta_\Delta$,
\[ \sum_{t=1}^T \|\widetilde{\bDelta}_t\|_2^2 \mathbf 1_{A^U_{\bar{R}}(t-1)}
\le \frac{2T G_{\rm clip}^2}{B}
+ 8G_{\rm clip}^2 \log\Big(\frac1{\delta_\Delta}\Big)
= V_{\Delta,\delta_\Delta}. \]
This completes the proof. 
\end{proof}

We next estimate the three terms in the stopped potential fluctuation $M_k$ separately. 
\begin{lemma} \label{lem:pot-fluctuation-1}
Suppose Assumption~\ref{ass:sigma} holds and the event $\mathcal E_\delta$ occurs.
Fix $\bar{R}>0$ and $\delta_{\mathrm{pot},1}\in(0,1)$. Define
\[ M_k^{(1)}
= -2\eta\sum_{t=1}^k \mathbf 1_{A_{\bar{R}}^{U}(t-1)}
\big\langle \widetilde{\bDelta}_t,\bU_{t-1}-\bW^*\big\rangle,
\qquad k\in[T]. \]
Then, conditioned on the dataset and the initialization satisfying
$\mathcal E_\delta$, it holds that
\[ \mathbb P_{\mathcal A}\Big( \max_{1\le k\le T}|M_k^{(1)}|
\le 4\sqrt{2}\,\eta G_{\rm clip}\bar{R} \sqrt{\frac{T\log(2/\delta_{\mathrm{pot},1})}{B}} \Big)
\ge 1-\delta_{\mathrm{pot},1}. \]
\end{lemma}

\begin{proof}
Condition on the dataset $S$ and on the initialization satisfying $\mathcal E_\delta$.
For each $t\in[T]$, define $X_t = -2\eta \mathbf 1_{A_{\bar{R}}^{U}(t-1)} \langle \widetilde{\bDelta}_t, \bU_{t-1}-\bW^*  \rangle.$ 
Then $M_k^{(1)} = \sum_{t=1}^k X_t$. 
Since  $\mathbb E[ \widetilde{\bDelta}_t \mid \mathscr F_{t-1} ] = 0,$ 
we have $\mathbb E[ X_t \mid \mathscr F_{t-1} ] = 0.$ 

We now show that $X_t$ is conditionally sub-Gaussian.
Conditioned on $\mathscr F_{t-1}$, define $h_t = \mathbf 1_{A_{\bar{R}}^{U}(t-1)} (\bU_{t-1}-\bW^*).$ 
Then $\|h_t\|_2 \le \bar R.$ 
Also define the deterministic scalars $a_i = \big\langle g^{\mathrm{aux}}_{t,i}, h_t \big\rangle$ for  $i\in[n]$. 
On the event $\mathcal E_\delta$, by the definition of
$g^{\mathrm{aux}}_{t,i}$,
\[ |a_i| \le \|g^{\mathrm{aux}}_{t,i}\|_2
\|h_t\|_2 \le G_{\rm clip}\bar R,
\qquad \forall i\in[n]. \]

Let $\bar a = \frac1n \sum_{i=1}^n a_i$. 
Conditioned on $\mathscr F_{t-1}$, the values $a_1,\ldots,a_n$ are fixed, and $ \big\langle \widetilde{\bDelta}_t,h_t \big\rangle
= \frac1B \sum_{i\in\mathcal B_t}a_i - \bar a$.  
By Lemma~\ref{lem:finite-pop-subgaussian}, conditionally on
$\mathscr F_{t-1}$, with $a_i=\langle g^{\mathrm{aux}}_{t,i},h_t\rangle $ and  $ |a_i|\le G_{\rm clip}\bar R$, 
we have, for every $\theta\in\R$,
\[ \mathbb E\Big[ \exp\Big( \theta\langle\widetilde{\bDelta}_t,h_t\rangle\Big)
\Big| \mathscr F_{t-1} \Big]
\le \exp\Big( \frac{\theta^2G_{\rm clip}^2\bar R^2}{2B}
\Big). \]
Since $X_t=-2\eta\langle\widetilde{\bDelta}_t,h_t\rangle$, 
it follows that
\[ \mathbb E\left[ e^{\theta X_t}
\middle| \mathscr F_{t-1} \right]
\le \exp\Big( \frac{2\theta^2\eta^2G_{\rm clip}^2\bar R^2}{B} \Big). \]

Therefore,
$\{X_t\}_{t=1}^T$ is a conditionally sub-Gaussian martingale difference sequence with variance proxy $\sigma_1^2 = \frac{4\eta^2 G_{\rm clip}^2 \bar R^2 }{B}$. 

By Theorem~1(a) of \cite{howard2020time}, applied to the sub-Gaussian process with \(S_k=\sum_{t=1}^k X_t\) and \(V_k=k\sigma_1^2\), and then to \(-S_k\), we obtain, for every \(u>0\),
\[ \mathbb P_{\mathcal A}\Big( \max_{1\le k\le T} \Big|\sum_{t=1}^k X_t\Big| \ge \sqrt{2T\sigma_1^2u} \Big) \le 2e^{-u}. \]

Taking $u = \log\big( \frac{2}{\delta_{\mathrm{pot},1}} \big),$ 
we obtain
\[ \mathbb P_{\mathcal A}\Big( \max_{1\le k\le T} |M_k^{(1)}|
\ge \eta 2 G_{\rm clip}  \bar R \sqrt{    \frac{2  T \log ( {2}/{\delta_{\mathrm{pot},1}} )  }{ B }
} \Big)
\le \delta_{\mathrm{pot},1}. \]
Enlarging the numerical constant to $4\sqrt{2}$ proves the stated bound.
\end{proof}

\begin{lemma} \label{lem:pot-fluctuation-2}
Suppose $0<\lambda<1$. Fix $\bar{R}>0$ and $\delta_{\mathrm{pot},2}\in(0,1)$. Define
\[  M_k^{(2)}  = -2(1-\lambda)\eta c_\mathrm{priv}
\sum_{t=1}^k \mathbf 1_{A_{\bar{R}}^{U}(t-1)}
\big\langle Z_{t-1},\bU_{t-1}-\bW^*\big\rangle,
\qquad k\in[T]. \]
Then, conditioned on the dataset and the initialization satisfying the event $\mathcal E_\delta$, it holds that
\[ \mathbb P_{\mathcal A}\Big( \max_{1\le k\le T}
|  M_k^{(2)}| \le 4(1-\lambda)\eta c_\mathrm{priv}\,\bar{R}
\sqrt{T\log ( {2}/{\delta_{\mathrm{pot},2}} )} \Big)
\ge 1-\delta_{\mathrm{pot},2}. \]
\end{lemma}

\begin{proof}
Condition on the dataset $S$, the comparator $\bW^*$, and the initialization $(\bW_0,\bc)$ satisfying the event $\mathcal E_\delta$. For $t\ge 2$, define the lagged filtration $\mathscr H_{t-1}
=
\sigma\bigl(
S,\bW_0,\bc,\bW^*,
\mathcal B_1,\dots,\mathcal B_{t-1},
Z_1,\dots,Z_{t-2}
\bigr)$. 
For $t=1$, the summand is zero since $Z_0=\mathbf 0$. For $t\ge 2$, the shifted iterate $\bU_{t-1}$ is $\mathscr H_{t-1}$-measurable, and so is $\mathbf 1_{A_{\bar{R}}^{U}(t-1)}$. Moreover, $Z_{t-1}\sim\mathcal N(\mathbf 0,\mathbf{I}_{mdp})$ is independent of $\mathscr H_{t-1}$.

Define
\[ Y_t  = -2(1-\lambda)\eta c_\mathrm{priv}\,
\mathbf 1_{A_{\bar{R}}^{U}(t-1)}
\big\langle Z_{t-1},\bU_{t-1}-\bW^*\big\rangle,
\qquad t=1,\dots,T. \]
Then, $M_k^{(2)}=\sum_{t=1}^k Y_t$. For $t=1$, $Y_1=0$. For $t\ge 2$, $\mathbb E[Y_t\mid \mathscr H_{t-1}]=0$.

Furthermore, conditioned on $\mathscr H_{t-1}$, $Y_t$ is Gaussian with variance
\[ 4(1-\lambda)^2\eta^2c_\mathrm{priv}^2 \mathbf 1_{A_{\bar{R}}^{U}(t-1)}
\|\bU_{t-1}-\bW^*\|_2^2 \le 4(1-\lambda)^2\eta^2c_\mathrm{priv}^2\bar{R}^2. \]
Therefore, for every $\theta\in\mathbb R$,
\[ \mathbb E\left[ \exp(\theta Y_t) \,\middle|\, \mathscr H_{t-1} \right]
\le \exp\left( 2\theta^2(1-\lambda)^2\eta^2c_\mathrm{priv}^2\bar{R}^2 \right). \]
That is, $\{Y_t\}_{t=1}^T$ is a conditionally sub-Gaussian martingale difference sequence with variance proxy $\sigma_2^2 =4(1-\lambda)^2\eta^2c_\mathrm{priv}^2\bar{R}^2$.
Then, by Theorem~1(a) of \cite{howard2020time}, applied to the conditionally sub-Gaussian martingale difference sequence $\{Y_t\}_{t=1}^T$ and to its negative,
\[ \mathbb P_{\mathcal A}\Big( \max_{1\le k\le T} \Big| \sum_{t=1}^kY_t \Big| \ge \sqrt{2T\sigma_2^2u} \Big)
\le 2e^{-u}. \]
Taking $u=\log(\frac{2}{\delta_{\mathrm{pot},2}})$, we get with probability at least $1-\delta_{\mathrm{pot},2}$,
\[ \max_{1\le k\le T} |  M_k^{(2)}|
\le 2\sqrt2(1-\lambda)\eta c_\mathrm{priv}\,\bar{R}
\sqrt{ T\log ( {2}/{\delta_{\mathrm{pot},2}} ) }. \]
Enlarging the constant to $4$ proves the claim.
\end{proof}

\begin{lemma} \label{lem:pot-fluctuation-3}
Suppose Assumption~\ref{ass:sigma} holds and the event $\mathcal E_\delta$ occurs.
Fix $\bar{R}>0$, $V_Z>0$, and $\delta_{\mathrm{pot},3}\in(0,1)$. Define
\[ M_k^{(3)} = 2(1-\lambda)\eta^2 c_\mathrm{priv}
\sum_{t=1}^k \mathbf 1_{A_{\bar{R}}^{U}(t-1)}
\big\langle \widetilde{\bDelta}_t,Z_{t-1}\big\rangle,
\qquad k\in[T]. \]
Then, conditioned on the dataset and the initialization satisfying $\mathcal E_\delta$, it holds that
\[ \mathbb P_{\mathcal A}\Big(
G_Z(z,V_Z) \cap \Big\{ \max_{1\le k\le T}|M_k^{(3)}|
> 4(1-\lambda)\eta^2 c_\mathrm{priv}G_{\rm clip}
\sqrt{ \frac{V_Z\log(2/\delta_{\mathrm{pot},3})}{B} } \Big\} \Big)
\le \delta_{\mathrm{pot},3}. \]
\end{lemma}

\begin{proof}
Condition on the dataset $S$, the comparator $\bW^*$, and the initialization $(\bW_0,\bc)$ satisfying $\mathcal E_\delta$.
Conditioned on $\mathscr F_{t-1}$, the vectors $Z_{t-1}$, $\widetilde{\bW}_{t-1}$, and $\{g^{\mathrm{aux}}_{t,i}\}_{i=1}^n$ are fixed, while the randomness in $\widetilde{\bDelta}_t$ comes only from the fresh mini-batch $\mathcal B_t$.

For each $t\in[T]$, define $Y_t = 2(1-\lambda)\eta^2c_\mathrm{priv} \mathbf 1_{A_{\bar{R}}^{U}(t-1)} \big\langle \widetilde{\bDelta}_t,Z_{t-1}\big\rangle$. 
Then $ M_k^{(3)}=\sum_{t=1}^kY_t$.  
Since  $\mathbb E[ \widetilde{\bDelta}_t \mid \mathscr F_{t-1} ] =0,$ 
we have $\mathbb E[ Y_t \mid \mathscr F_{t-1} ] =0.$

Conditioned on $\mathscr F_{t-1}$, set $h_t = \mathbf 1_{A_{\bar{R}}^{U}(t-1)}Z_{t-1}$. 
Since $\|g^{\mathrm{aux}}_{t,i}\|_2\le G_{\rm clip}$ for $i\in[n],$ 
Lemma~\ref{lem:finite-pop-subgaussian} gives, for every $\theta\in\mathbb R$,
\[ \mathbb E\Big[
\exp\big( \theta\langle\widetilde{\bDelta}_t,h_t\rangle \big)
\big| \mathscr F_{t-1} \Big]
\le \exp\Big(
\frac{ \theta^2G_{\rm clip}^2 \mathbf 1_{A_{\bar{R}}^{U}(t-1)} \|Z_{t-1}\|_2^2 }{ 2B } \Big). \]
Since $Y_t = 2(1-\lambda)\eta^2c_\mathrm{priv} \langle\widetilde{\bDelta}_t,h_t\rangle$, 
we obtain
\[ \mathbb E\left[ \exp(\theta Y_t)
\middle| \mathscr F_{t-1} \right]
\le \exp\Big( \frac{\theta^2\sigma_t^2}{2}\Big), \]
where $\sigma_t^2 = 4(1-\lambda)^2\eta^4c_\mathrm{priv}^2
G_{\rm clip}^2   \mathbf 1_{A_{\bar{R}}^{U}(t-1)} \|Z_{t-1}\|_2^2 /B .$ 
Thus, $\{Y_t\}_{t=1}^T$ is a conditionally sub-Gaussian martingale
difference sequence with predictable variance process $V_k^{(3)} = \sum_{t=1}^k\sigma_t^2$. 
On $G_Z(z,V_Z)$, we have
\[ V_T^{(3)}
\le 4(1-\lambda)^2\eta^4c_\mathrm{priv}^2
G_{\rm clip}^2 \frac{V_Z}{B} =:v. \]

By Theorem~1(a) of \cite{howard2020time}, applied to
$\sum_{t=1}^kY_t$ and its negative, for every $u>0$,
\[ \mathbb P_{\mathcal A}\Big( \max_{1\le k\le T}
\Big| \sum_{t=1}^kY_t
\Big| \ge \sqrt{2vu},
\quad
V_T^{(3)}\le v
\Big) \le 2e^{-u}. \]
Taking $u = \log ( \frac{2}{\delta_{\mathrm{pot},3}} )$  
and using the fact that $G_Z(z,V_Z)$ implies $V_T^{(3)}\le v$, we obtain
\[ \mathbb P_{\mathcal A}\Big(
G_Z(z,V_Z) \cap \Big\{ \max_{1\le k\le T}|M_k^{(3)}|
\ge 2\sqrt2(1-\lambda)\eta^2c_\mathrm{priv}G_{\rm clip}
\sqrt{ \frac{V_Z\log(2/\delta_{\mathrm{pot},3})}{B} } \Big\} \Big)
\le \delta_{\mathrm{pot},3}. \]
Enlarging the numerical constant from $2\sqrt2$ to $4$ proves the stated
bound.
\end{proof}

We now combine the above three fluctuation bounds to control the stopped potential-fluctuation event \(G_{\mathrm{pot}}\).

\begin{lemma}[High-probability control of the stopped potential fluctuation]
\label{lem:good-event-pot}
Suppose Assumption~\ref{ass:sigma} holds and the event $\mathcal E_\delta$ occurs. 
Fix $\bar{R}>0$, $z>0$, $V_Z>0$, and $\delta_{\mathrm{pot}}\in(0,1)$.
Define
\begin{align}
 M_{\mathrm{pot}}(\bar{R},z,V_Z;\delta_{\mathrm{pot}})
 =\;& 4\sqrt2\,\eta G_{\rm clip} \bar{R} \sqrt{\frac{T\log(6/\delta_{\mathrm{pot}})}{B}}
+ 4(1-\lambda)\eta c_\mathrm{priv}\,\bar{R} \sqrt{T\log\big( {6}/{\delta_{\mathrm{pot}}}\big)}
\notag\\
&+ 4(1-\lambda)\eta^2 c_\mathrm{priv}G_{\rm clip}
\sqrt{ \frac{V_Z\log(6/\delta_{\mathrm{pot}})}{B} }.
\label{eq:def-Mpot}
\end{align}
Then, conditioned on the dataset and the initialization satisfying $\mathcal E_\delta$, it holds that
\[
\mathbb P_{\mathcal A}\left(
G_Z(z,V_Z)\cap
G_{\mathrm{pot}}\bigl(\bar{R},M_{\mathrm{pot}}(\bar{R},z,V_Z;\delta_{\mathrm{pot}})\bigr)
\right)
\ge
\mathbb P_{\mathcal A}\bigl(G_Z(z,V_Z)\bigr)-\delta_{\mathrm{pot}}.
\]
\end{lemma}

\begin{proof}
Note that $M_k=M_k^{(1)}+M_k^{(2)}+M_k^{(3)}$ for all $k\in[T]$. 
Define the following three events
\begin{align*}
E_1 &= \Big\{ \max_{1\le k\le T} |M_k^{(1)}|
\le 4\sqrt2\,\eta G_{\rm clip}\bar{R} \sqrt{\frac{T\log(6/\delta_{\mathrm{pot}})}{B}} \Big\},
\\
E_2 &= \Big\{ \max_{1\le k\le T} |M_k^{(2)}|
\le 4(1-\lambda)\eta c_\mathrm{priv}\,\bar{R} \sqrt{T\log\big(\frac{6}{\delta_{\mathrm{pot}}}\big)} \Big\},
\\
E_3 &= \Big\{ \max_{1\le k\le T} |M_k^{(3)}|
\le 4(1-\lambda)\eta^2 c_\mathrm{priv}G_{\rm clip}
\sqrt{ \frac{V_Z\log(6/\delta_{\mathrm{pot}})}{B} } \Big\}.
\end{align*}

By Lemmas~\ref{lem:pot-fluctuation-1}, \ref{lem:pot-fluctuation-2}  and \ref{lem:pot-fluctuation-3}   with  $\delta_{\mathrm{pot},1} = \delta_{\mathrm{pot},2} = \delta_{\mathrm{pot},3} = \frac{\delta_{\mathrm{pot}}}{3}$, 
we have
\[ \mathbb P_{\mathcal A}(E_1) \ge 1-\frac{\delta_{\mathrm{pot}}}{3},
\qquad
\mathbb P_{\mathcal A}(E_2)
\ge 1-\frac{\delta_{\mathrm{pot}}}{3} 
\quad \text{ and } \quad 
\mathbb P_{\mathcal A} \bigl( G_Z(z,V_Z)\cap E_3^c \bigr) \le \frac{\delta_{\mathrm{pot}}}{3}. \]

On $E_1\cap E_2\cap E_3$, for every $k\in[T]$, it  holds that $|M_k| \le |M_k^{(1)}| + |M_k^{(2)}| + |M_k^{(3)}| \le M_{\mathrm{pot}}(\bar{R},z,V_Z;\delta_{\mathrm{pot}})$.  
Hence,
\[ E_1\cap E_2\cap E_3 \subseteq G_{\mathrm{pot}}
\bigl( \bar{R}, M_{\mathrm{pot}}(\bar{R},z,V_Z;\delta_{\mathrm{pot}}) \bigr). \]
Therefore,
\begin{align*}
&\mathbb P_{\mathcal A}\left( G_Z(z,V_Z) \cap
\left( G_{\mathrm{pot}} \bigl( \bar{R}, M_{\mathrm{pot}}(\bar{R},z,V_Z;\delta_{\mathrm{pot}})
\bigr) \right)^c \right) \\
&\le \mathbb P_{\mathcal A}(E_1^c)
+ \mathbb P_{\mathcal A}(E_2^c)
+ \mathbb P_{\mathcal A} \bigl( G_Z(z,V_Z)\cap E_3^c \bigr)
\le \delta_{\mathrm{pot}}.
\end{align*}
Rearranging proves the claim.
\end{proof}

Recall that
\[ z_{\delta_Z} = \sqrt{mdp}
+ \sqrt{2\log ( {2T}/{\delta_Z} )} , \qquad 
V_{\Delta,\delta_\Delta} = \frac{2TG_{\rm clip}^2}{B}
+ 8G_{\rm clip}^2 \log ( {1}/{\delta_\Delta} ), \]
\[ V_{Z,\delta_Z} = (T-1)mdp + 2\sqrt{ (T-1)mdp \log ( {2}/{\delta_Z} ) } + 2\log ( {2}/{\delta_Z} ), \]
and
\begin{align}\label{eq:def-Mpot-delta}
M_{\delta_{\mathrm{pot}}}
=\;& 4\sqrt2\,\eta G_{\rm clip}\bar{R}
\sqrt{ \frac{ T\log(6/\delta_{\mathrm{pot}}) }{B} } +
4(1-\lambda)\eta c_\mathrm{priv}\bar{R}
\sqrt{ T\log ( {6}/{\delta_{\mathrm{pot}}} ) } \notag\\
&+ 4(1-\lambda)\eta^2c_\mathrm{priv}G_{\rm clip}
\sqrt{ \frac{ V_{Z,\delta_Z} \log(6/\delta_{\mathrm{pot}}) }{B} }.
\end{align}

We now combine the high-probability bounds for \(G_Z\), \(G_{\Delta^2}\), and \(G_{\mathrm{pot}}\) with the shifted bootstrap lemma to obtain a conditional optimization bound for the projected
iterates.

\begin{lemma}[High-probability shifted bootstrap] \label{lem:high-probability-shifted-bootstrap}
Suppose $0<\lambda<1$, Assumption~\ref{ass:sigma} holds and the event
$\mathcal E_\delta$ occurs.
Fix a comparator $\bW^*$ and a shifted radius $\bar{R}>0$.
Let $\delta_Z,\delta_\Delta,\delta_{\mathrm{pot}}\in(0,1)$.
Assume $\eta\le \frac{1}{12\beta_{\rm clip}}$, \eqref{eq:shifted-localized-comparator-width} holds, and \eqref{eq:conditional-shifted-self-consistency} holds with $M=M_{\delta_{\mathrm{pot}}}, V_\Delta=V_{\Delta,\delta_\Delta}$ and $V_Z=V_{Z,\delta_Z}$. 
If $R_* \ge \bar{R} + \|\bW^*-\bW_0\|_2 + \eta c_\mathrm{priv}z_{\delta_Z}$, 
then, conditional on the dataset and the initialization satisfying
$\mathcal E_\delta$, with probability at least $1-(\delta_Z+\delta_\Delta+\delta_{\mathrm{pot}})$ over the randomness of the algorithm,
\begin{align*}
\frac1T\sum_{t=1}^T \mathcal L_S(\bW_{t-1})
\le \frac{1}{\rho_{\rm clip}}
\Big[ & 8\mathcal L_S(\bW^*) + \frac{ 3\big( \|\bW_0-\bW^*\|_2^2 + M_{\delta_{\mathrm{pot}}} \big)}{ \eta T }
+ \frac{ 6\eta V_{\Delta,\delta_\Delta} }{T}\\
&+ 3\Big( (1-\lambda)^2\eta + 12\beta_{\rm clip}\lambda^2\eta^2\Big)  \frac{c_\mathrm{priv}^2 V_{Z,\delta_Z} }{T} \Big].
\end{align*}
\end{lemma}

\begin{proof}
By Lemma~\ref{lem:good-event-Z}, Lemma~\ref{lem:good-event-delta-square} and Lemma~\ref{lem:good-event-pot}, we have
\[ \mathbb P_{\mathcal A}\Big( G_Z(z_{\delta_Z},V_{Z,\delta_Z})
\cap G_{\Delta^2} \big( \bar{R}, V_{\Delta,\delta_\Delta} \big)
\cap G_{\mathrm{pot}} \big( \bar{R}, M_{\delta_{\mathrm{pot}}} \big) \Big)
\ge 1- \big( \delta_Z + \delta_\Delta + \delta_{\mathrm{pot}} \big). \]
Combining this with Lemma~\ref{lem:conditional-shifted-bootstrap-loss} completes the proof.
\end{proof}

\textsc{Step 5 and 6: Close the bootstrap, show projection inactivity and Choose the comparator using initialization separability.}

 Compared with the proof sketch in Section~\ref{sec:proofsketch}, the last two steps are treated together in the detailed proof. In the sketch, Step 5 closes the localization bootstrap for a fixed comparator, while Step 6 chooses the NTK-separability comparator. In the rigorous argument, these two steps are coupled, since the shifted localization event, the bootstrap condition, and the projection-inactivity argument all depend on the chosen comparator \(\bW^*\). We therefore first establish the existence of a suitable comparator under the initialization separability assumption, and then use this comparator to close the bootstrap and transfer the bound to the projected iterates.
\begin{lemma}[Comparator under NTK separability]\label{lem:ntk-comparator-lambda0}
Suppose Assumption~\ref{ass:ntk} holds. Assume 
\[m \gtrsim \log(m/\delta)\big(\log^2(T/\rho_{\rm clip}) + \log(n/\delta)\big)/\gamma^4. \]
Let $\tau \asymp \frac{1}{\gamma} \big(\log (T/\rho_{\rm clip}) + \sqrt{\log(n/\delta)}\big)$. Assume $R_* \ge \tau$. Then, with probability at least $1-\delta$ over the randomness of the initialization, there exists a comparator $\bW^* = \bW_0 + \tau \bu \in \mathcal K$ such that
\begin{equation*}
L_S(\bW^*)\le \frac{\rho_{\rm clip}}{T} \quad 
\text{ and } \quad 
\|\bW^*-\bW_0\|_2^2 = \tau^2 \asymp \frac{\log^2 (T/\rho_{\rm clip}) + \log(n/\delta)}{\gamma^2}.
\end{equation*}
\end{lemma}

\begin{proof}
Let $\bu$ be the unit vector in Assumption~\ref{ass:ntk}, and define $\bW^*=\bW_0+\tau\bu$. 
Since $\|\bu\|_2=1$ and $\tau\le R_*$, we have
\[ \bW^*\in\mathcal K, \qquad \|\bW^*-\bW_0\|_2=\tau. \]

For any fixed $i\in[n]$, by Taylor's theorem, there exists $\alpha_i\in[0,1]$ with $\bW_i'=\alpha_i\bW_0+(1-\alpha_i)\bW^*$ 
such that
\[ f_{\bW^*}(\bx_i) = f_{\bW_0}(\bx_i) + \big\langle \nabla f_{\bW_0}(\bx_i), \bW^*-\bW_0 \big\rangle + \frac12 (\bW^*-\bW_0)^\top \nabla^2 f_{\bW_i'}(\bx_i) (\bW^*-\bW_0). \]
Multiplying both sides by $y_i$ and using $\bW^\ast-\bW_0=\tau\bu$, we obtain
\[y_i f_{\bW^*}(\bx_i)=y_i f_{\bW_0}(\bx_i)+\tau\, y_i\big\langle \nabla  f_{\bW_0}(\bx_i), \bu \big\rangle+\frac{\tau^2}{2}y_i\, \bu^\top \nabla^2 f_{\bW_i'}(\bx_i)\bu.\]
Since $|y_i|\le 1$, it holds that
\[y_i f_{\bW^*}(\bx_i)\ge-|f_{\bW_0}(\bx_i)|+\tau\, y_i\big\langle \nabla  f_{\bW_0}(\bx_i), \bu \big\rangle-\frac{\tau^2}{2}\big\|\nabla^2 f_{\bW_i'}(\bx_i)\big\|_2.\]
From Assumption~\ref{ass:ntk}, we know that 
\[ y_i\big\langle \nabla f_{\bW_0}(\bx_i), \bu \big\rangle \ge \gamma,\]
which implies
\begin{equation}\label{eq:margin-lower-before-width}
y_i f_{\bW^*}(\bx_i)\ge\tau\gamma-|f_{\bW_0}(\bx_i)|-\frac{\tau^2}{2}\big\|\nabla^2 f_{\bW_i'}(\bx_i)\big\|_2.
\end{equation}

Now, we estimate the lower bound of the right hand side of \eqref{eq:margin-lower-before-width}. 
From Eq.~(27) in \cite{wang2026optimization}, with probability at least
$1-\delta/2$ over the initialization,
\begin{equation} \label{eq:init-output-bound-reliability}
|f_{\bW_0}(\bx_i)|
\le B_b \sqrt{2p\log ( {2n}/{\delta} )},
\qquad \forall i\in[n].
\end{equation}
Moreover, $\mathcal E_{\delta/2}$ occurs with probability at least $1-\delta/2$. On this event, Lemma~\ref{lem:hessian} gives
\[ \big\|\nabla^2 f_{\bW_i'}(\bx_i)\big\|_2
\le \frac{ C_{\sigma,b}p^{3/2} \big(\sqrt p+\sqrt{\log(2m/\delta)}\big) }{ \sqrt m }. \]
Plugging the above two bounds into \eqref{eq:margin-lower-before-width}, we get
\begin{align*}
 y_i f_{\bW^*}(\bx_i) & \ge \tau\gamma-B_b \sqrt{2p\log ( \frac{2n}{\delta} )}-\frac{\tau^2 C_{\sigma, b} \, p^{\frac{ 3}{2}}\big(\sqrt{p} +  \sqrt{\log({2m}/{\delta})} \big)}{2\sqrt{m}}\\ 
       &\ge \frac{\tau\gamma}{2} - \frac{\tau^2 C_{\sigma, b} \, p^{\frac{ 3}{2}}\big(\sqrt{p} +  \sqrt{\log({2m}/{\delta})}  \big)}{2\sqrt{m}}  \ge \frac{\tau\gamma}{4} \gtrsim \log(T/\rho_{\rm clip}),
\end{align*}
where the second inequality uses $\tau \asymp \frac{\log(T/\rho_{\rm clip})+\sqrt{\log(n/\delta)}}{\gamma}$, 
and the third inequality uses $m \gtrsim \frac{ \log(m/\delta) (\log^2(T/\rho_{\rm clip})+\log(n/\delta) ) }{ \gamma^4 }
\gtrsim  \frac{\log(m/\delta)\tau^2}{\gamma^2}$.

It then follows that 
\begin{align*}
\ell\big(y_if_{\bW^*}(\bx_i)\big)
&= \log\big(1 + \exp\big(-y_if_{\bW^*}(\bx_i)\big)\big)
\le \exp\big(-y_if_{\bW^*}(\bx_i)\big)\\
&\le \exp\big(-\log(T/\rho_{\rm clip})\big)
= \frac{\rho_{\rm clip}}{T}.
\end{align*}
Averaging over $i\in[n]$ yields
\[ \L_S(\bW^\ast)\le \frac{\rho_{\rm clip}}{T},\]
which proves the first result of the lemma. 

The second result follows directly from the definition of $\tau$. This completes the proof.
\end{proof}

Recall that $c_\mathrm{priv} = \frac{C_{\mathrm{clip}}\kappa}{B}$, $\beta = C_{\sigma,b}p^3$, $G_{\rm clip} = \min\{C_{\rm clip},G_\delta\}$ and  
$\rho_{\rm clip} = \min\big\{1,\frac{C_{\rm clip}}{G_\delta}\big\}$. 
Define $\eta_{\rm eff}:=\eta\rho_{\rm clip}$ and 
$c_{\rm eff}:=\frac{c_{\rm priv}}{\rho_{\rm clip}}$. 
Now, we give our main result.

\begin{theorem}[Restatement of Theorem~\ref{thm:lambda-positive-ntk-opt}]
Let $\delta\in(0,1)$, $0<\lambda<1$, $\kappa>0$, and suppose
Assumptions~\ref{ass:sigma} and~\ref{ass:ntk} hold.
Set $R_* \asymp \frac{\log(T/\rho_{\rm clip})+\sqrt{\log(n/\delta)}}{\gamma}$. 
Let $\{\bW_t\}_{t=0}^{T-1}$ be generated by
Algorithm~\ref{alg:dp-lambda-minibatch} with $\eta_{\rm eff} \le \frac{1}{12C_{\sigma,b}p^3}$ and $C_{\rm clip}>0$. 
Assume
\[
\eta_{\rm eff}\gamma\sqrt{T}
\Big( \frac{1}{\sqrt B} + (1-\lambda)c_{\rm eff} \Big)
+\eta_{\rm eff}^2\gamma^2
\Big( \frac{T}{B}+1 \Big)
\lesssim 1
\]
and
\[
\widetilde{\Omega}(\gamma^{-4})
\le m \le
\widetilde{\O} \Big(
(\eta_{\rm eff}^2c_{\rm eff}^2\gamma^2d)^{-1}
\min \big\{ 1,
\big( T((1-\lambda)^2+\lambda^2\eta_{\rm eff}) \big)^{-1}\big\}\Big),
\]
where $\widetilde{\O}$ and $\widetilde{\Omega}$ suppress
polylogarithmic factors in $m,n,T,\delta^{-1},\rho_{\rm clip}^{-1}$.
Then, with probability at least $1-\delta$ over the initialization
and the algorithmic randomness,
\begin{align*}
\frac1T\sum_{t=0}^{T-1}\mathcal L_S(\bW_t)
\lesssim
\mathrm{polylog}\big(\frac{nT}{\delta\rho_{\rm clip}}\big)
\cdot
\Big[\;& \frac{1}{\gamma^2\eta_{\rm eff} T}
+ \frac{1}{\gamma\sqrt{BT}}
+ \eta_{\rm eff}\big( \frac1B+\frac1T \big)+ \frac{(1-\lambda)c_{\rm eff}}{\gamma \sqrt T} \\
&  + \left( (1-\lambda)^2+\lambda^2\eta_{\rm eff} \right)
\eta_{\rm eff} c_{\rm eff}^2md \Big]
=:A_{\rm corr}.
\end{align*} 
\end{theorem}

\begin{proof}
Fix auxiliary failure probabilities $\delta_{\mathrm{init}} = \delta_{\mathrm{ntk}} = \delta_Z = \delta_\Delta = \delta_{\mathrm{pot}} = \frac{\delta}{5}$. 
Since replacing $\delta$ by a constant fraction only affects logarithmic factors by absolute constants, we suppress this distinction below. Set
\[
\tau_\gamma^2 \asymp
\frac{\log^2(T/\rho_{\rm clip})+\log(n/\delta)}{\gamma^2}.
\]
We choose the shifted localization radius as $\bar R = C_{\bar R}\tau_\gamma$, where $C_{\bar R}>0$ is a sufficiently large universal constant.
Our proof consists of the following steps.

\textsc{(i). Comparator construction under initialization separability.}

By Lemma~\ref{lem:ntk-comparator-lambda0}, under the width condition
\[
m \gtrsim
\frac{\log(m/\delta_{\mathrm{ntk}})
\big(\log^2(T/\rho_{\rm clip})+\log(n/\delta_{\mathrm{ntk}})\big)}
{\gamma^4},
\]
there exists a comparator $\bW^*=\bW_0+\tau_{\mathrm{ntk}}\bu$
such that, with probability at least $1-\delta_{\mathrm{ntk}}$ over the initialization,
\[
\mathcal L_S(\bW^*) \le \frac{\rho_{\rm clip}}{T}
\quad \text{ and } \quad
\|\bW^*-\bW_0\|_2 \le \tau_{\mathrm{ntk}},
\]
where $\tau_{\mathrm{ntk}}^2
\lesssim
\frac{\log^2(T/\rho_{\rm clip})+\log(n/\delta)}{\gamma^2}$. 
By taking $C_{\bar R}$ sufficiently large, we may assume $\tau_{\mathrm{ntk}}\le\tau_\gamma$.

\textsc{(ii). Verification of the shifted-bootstrap assumptions.}

We now verify the assumptions of Lemma~\ref{lem:high-probability-shifted-bootstrap}.
Throughout this step, we use the same convention as in the theorem statement and suppress polylogarithmic factors in $m,n,T,\delta^{-1},\rho_{\rm clip}^{-1}$ and polynomial factors in $p$ and $C_{\sigma,b}$.

By the definitions of $z_{\delta_Z}$, $V_{Z,\delta_Z}$, $V_{\Delta,\delta_\Delta}$, and $M_{\delta_{\mathrm{pot}}}$, and since all auxiliary failure probabilities are constant fractions of $\delta$, we have
\[
z_{\delta_Z} \lesssim \sqrt{ md+\log ( {nT}/{\delta} ) } , 
\quad
\frac{V_{Z,\delta_Z}}{T} \lesssim md+ \frac{\log(nT/\delta)}{T},
\quad 
V_{\Delta,\delta_\Delta} \lesssim G_{\rm clip}^2\Big( \frac{T}{B}
+ \log\big(\frac{nT}{\delta}\big) \Big).
\]
Moreover,
\[
\begin{aligned}
M_{\delta_{\mathrm{pot}}} \lesssim{}&
\eta G_{\rm clip}\bar R \sqrt{ \frac{T\log(nT/\delta)}{B} }
+ (1-\lambda)\eta c_{\mathrm{priv}}\bar R
\sqrt{ T\log\Big(\frac{nT}{\delta}\Big) }\\
&+ (1-\lambda)\eta^2 c_{\mathrm{priv}}G_{\rm clip}
\sqrt{ \frac{ Tmd\log(nT/\delta) }{ B } }.
\end{aligned}
\]
Using $G_{\rm clip}=\rho_{\rm clip}G_\delta$,
$c_{\rm priv}=\rho_{\rm clip}c_{\rm eff}$,
$\eta_{\rm eff}=\eta\rho_{\rm clip}$,
and $\bar R\asymp\tau_\gamma$, we obtain
\[
\begin{aligned}
M_{\delta_{\mathrm{pot}}}
\lesssim{}&
\eta_{\rm eff}\tau_\gamma
\sqrt{\frac{T\log(nT/\delta)}{B}}
+(1-\lambda)\eta_{\rm eff}c_{\rm eff}\tau_\gamma
\sqrt{T\log\big(\frac{nT}{\delta}\big)}\\
&+(1-\lambda)\eta_{\rm eff}^2c_{\rm eff}
\sqrt{\frac{Tmd\log(nT/\delta)}{B}},
\end{aligned}
\]
up to the suppressed factors.

We next verify the localization condition required by Lemma~\ref{lem:shifted-localized-comparator}. The lemma requires
\[
m \gtrsim C_{\sigma,b}^2p^3 \bigl(\log(m/\delta)+p\bigr)
\bigl( \bar R+\eta c_{\mathrm{priv}}z_{\delta_Z} \bigr)^4.
\]
By the upper bound on $m$ in the theorem, together with the bound on $z_{\delta_Z}$ and
$\eta c_{\mathrm{priv}}=\eta_{\rm eff}c_{\rm eff}$, we have $\eta c_{\mathrm{priv}}z_{\delta_Z} =\eta_{\rm eff}c_{\rm eff}z_{\delta_Z} \lesssim \tau_\gamma$  up to the suppressed factors.
Since $\bar R\asymp\tau_\gamma$, it follows that $\bar R+\eta c_{\mathrm{priv}}z_{\delta_Z} \lesssim \tau_\gamma$. 
Since $\tau_\gamma^4
\asymp
\frac{ (\log^2(T/\rho_{\rm clip})+\log(n/\delta) )^2}{\gamma^4}$, 
the stated lower bound on $m$ implies that the localization condition in Lemma~\ref{lem:shifted-localized-comparator} holds.

It remains to verify the self-consistency condition \eqref{eq:conditional-shifted-self-consistency} in Lemma~\ref{lem:conditional-shifted-bootstrap-loss}.
Using $\|\bW^*-\bW_0\|_2^2 \le \tau_\gamma^2$,
$\mathcal L_S(\bW^*) \le \rho_{\rm clip}/T$,
$\beta_{\rm clip}=\rho_{\rm clip}\beta$, and the definitions of
$\eta_{\rm eff}$ and $c_{\rm eff}$, the left-hand side of
\eqref{eq:conditional-shifted-self-consistency} is bounded, up to the suppressed factors, by
\[
\tau_\gamma^2
+\eta_{\rm eff}
+M_{\delta_{\mathrm{pot}}}
+\eta_{\rm eff}^2\Big(\frac{T}{B}+1\Big)
+\eta_{\rm eff}^2c_{\rm eff}^2
\Big((1-\lambda)^2+\lambda^2\eta_{\rm eff}\Big)Tmd.
\]
Substituting the preceding estimates and using the step-size and width conditions in the theorem gives
\[
\|\bW_0-\bW^*\|_2^2
+\frac{8}{3}\eta T\mathcal L_S(\bW^*)
+M_{\delta_{\mathrm{pot}}}
+2\eta^2V_{\Delta,\delta_\Delta}
+\big((1-\lambda)^2\eta^2c_{\mathrm{priv}}^2
+12\beta_{\rm clip}\lambda^2\eta^3c_{\mathrm{priv}}^2\big)V_{Z,\delta_Z}
\lesssim \tau_\gamma^2.
\]
Since $\bar R=C_{\bar R}\tau_\gamma$, choosing $C_{\bar R}$ sufficiently large verifies \eqref{eq:conditional-shifted-self-consistency}.

Finally, the same bound
$\eta c_{\mathrm{priv}}z_{\delta_Z}
=\eta_{\rm eff}c_{\rm eff}z_{\delta_Z}\lesssim\tau_\gamma$,
together with $\bar R\asymp\tau_\gamma$ and
$\|\bW^*-\bW_0\|_2\le\tau_\gamma$, gives
\[
\bar R+\|\bW^*-\bW_0\|_2
+\eta c_{\mathrm{priv}}z_{\delta_Z}
\lesssim \tau_\gamma.
\]
Choosing the constant in $R_*=C\tau_\gamma$ sufficiently large ensures
\[
R_*
\ge
\bar R
+
\|\bW^*-\bW_0\|_2
+
\eta c_{\mathrm{priv}}z_{\delta_Z}.
\]
Therefore, all assumptions of
Lemma~\ref{lem:high-probability-shifted-bootstrap} are verified.

\textsc{(iii). Application of the shifted-bootstrap lemma.}

Applying Lemma~\ref{lem:high-probability-shifted-bootstrap}
and using
$\mathcal L_S(\bW^*) \le \rho_{\rm clip}/T$ and
$\|\bW_0-\bW^*\|_2^2 \le \tau_\gamma^2$,
we obtain that, conditioned on the initialization and comparator
events, with probability at least
$1-(\delta_Z+\delta_\Delta+\delta_{\mathrm{pot}})$ over the
algorithmic randomness,
\begin{align}
\frac1T \sum_{t=0}^{T-1} \mathcal L_S(\bW_t)
\lesssim\;&
\frac1T
+\frac{\tau_\gamma^2}{\eta_{\rm eff}T}
+\frac{M_{\delta_{\mathrm{pot}}}}{\eta_{\rm eff}T}
+\eta_{\rm eff}\Big(\frac1B+\frac1T\Big) \notag\\
&+\Big((1-\lambda)^2\eta_{\rm eff}
+\lambda^2\eta_{\rm eff}^2\Big)c_{\rm eff}^2
\Big(md+\frac{\log(nT/\delta)}{T}\Big).
\label{eq:proof-corr-clean-pre}
\end{align}

Substituting the estimate of $M_{\delta_{\mathrm{pot}}}$ from Step~(ii) into \eqref{eq:proof-corr-clean-pre} and using
$\tau_\gamma^2\asymp
(\log^2(T/\rho_{\rm clip})+\log(n/\delta))/\gamma^2$,
we obtain, up to the suppressed logarithmic and $p$-dependent factors,
\begin{align*}
\frac1T\sum_{t=0}^{T-1}\mathcal L_S(\bW_t)
\lesssim
\mathrm{polylog}\big(\frac{nT}{\delta\rho_{\rm clip}}\big)
\cdot
\Big[\;& \frac{1}{\gamma^2\eta_{\rm eff} T}
+ \frac{1}{\gamma\sqrt{BT}}
+ \eta_{\rm eff}\big( \frac1B+\frac1T \big)+ \frac{(1-\lambda)c_{\rm eff}}{\sqrt T}
\big( \frac{1}{\gamma}\\
&
+ \frac{\eta_{\rm eff}\sqrt{md}}{\sqrt B} \big)
+ \left( (1-\lambda)^2+\lambda^2\eta_{\rm eff} \right)
\eta_{\rm eff}c_{\rm eff}^2md \Big]
=:A_{\rm corr}.
\end{align*}
Here we also used
$m\gtrsim\log(nT/\delta)/(Td)$ and
$1/T\lesssim1/(\gamma^2\eta_{\rm eff}T)$.
This is exactly the claimed bound.

\textsc{(iv). Probability estimate.}

The above argument holds on the intersection of the initialization event $\mathcal E_{\delta_{\mathrm{init}}}$, the comparator event from
Lemma~\ref{lem:ntk-comparator-lambda0}, and the algorithmic good event from Lemma~\ref{lem:high-probability-shifted-bootstrap}.
Their total failure probability is at most $\delta_{\mathrm{init}} + \delta_{\mathrm{ntk}} + \delta_Z + \delta_\Delta + \delta_{\mathrm{pot}} = \delta.$ 
Therefore, the stated optimization bound holds with probability at least $1-\delta$.
\end{proof}

\section{Proofs for population of DP-SGD with correlated noise}\label{appen:lambda1-gen}

We introduce the concept of algorithmic stability to control generalization gap, and then combining generalization gap with optimization risk bound to get population risk bound. 

For a randomized algorithm $\A$, let $\A(S)\in\R^{mdp}$ be the output of $\A$ based on dataset $S$. The on-average argument stability measures the on-average sensitivity of the output up to the perturbation of the dataset.
\begin{definition}[On-average argument stability \cite{lei2020fine}]\label{def:stability}
Let $S=\{z_1,\ldots,z_n\}$ and $\widetilde{S}=\{z'_1,\ldots,z'_n\}$ be drawn independently from $\mathcal{P}$. For any $i\in[n]$, define $S^{(i)}=\{ z_1,\ldots,z_{i-1},z_i', z_{i+1},\ldots,z_n\}$.
Let $\mathcal{A}(S)$ and $\mathcal{A}(S^{ (i)})$ be produced by an randomized algorithm $\mathcal{A}$ based on $S$ and $S^{(i)}$ respectively. We say $\A$ is on-average argument $\epsilon$-stable if \[\E_{S,\widetilde{S},\A}\Big[\frac{1}{n}\sum_{i=1}^n\|\A(S)-\A(S^{(i)})\|_2 \Big] \le \epsilon .\]
\end{definition}

We consider using the connection between the on-average argument stability and generalization error bounds \cite{lei2020fine}. 
\begin{lemma}[\cite{lei2020fine}]\label{lem:connection}
If $\A$ is on-average argument $\epsilon$-stable and the loss $\ell$ is $L$-Lipschitz with respect to $\A(S)$, then
\[\E_{S,\mathcal{A}} \big[ \mathcal{L}(\mathcal{A}(S)) - \mathcal{L}_S(\mathcal{A}(S)) \big]\le  2L \epsilon. 
\]
\end{lemma}

To compare the outputs on two neighboring datasets $S$ and $S^{(i)}$, we couple two runs of Algorithm~\ref{alg:dp-lambda-minibatch} by using the same initialization, the same sequence of mini-batches, and the same Gaussian noise sequence. Under this coupling, the additive perturbations appear identically in the two updates and cancel in the difference of the iterates. The resulting stability recursion therefore depends only on the gradient part of the projected update. It is important, however, that the empirical losses appearing below are still evaluated along the two coupled private trajectories.

Let
\[ \Delta_t^{(i)} =\|\bW_t-\bW_t^{(i)}\|_2,
\qquad
\bW_{\alpha,t}^{(i)} =\alpha \bW_t+(1-\alpha)\bW_t^{(i)},
\quad \alpha\in[0,1].\]
For each $t\ge0$, define the shared mini-batch loss
\[ \mathcal L_{t,i}^{\mathrm{sh}}(\bW)
 = \frac{1}{B}\sum_{j\in \mathcal B_t\setminus\{i\}} \ell \left(y_j f_{\bW}(\bx_j)\right),\]
where the sum is interpreted as $0$ if $\mathcal B_t\setminus\{i\}=\varnothing$. We also define
\[ \mathcal L_{\mathcal B_t}(\bW)
 = \frac{1}{B}\sum_{j\in\mathcal B_t}\ell \left(y_j f_{\bW}(\bx_j)\right),\]
\[ \mathcal L_{\mathcal B_t}^{(i)}(\bW)
 = \frac{1}{B}\sum_{j\in\mathcal B_t}
\ell\left(y_j^{(i)} f_{\bW}(\bx_j^{(i)})\right)
 = \mathcal L_{t,i}^{\mathrm{sh}}(\bW)
+ \frac{1}{B}\mathbf 1_{\{i\in\mathcal B_t\}} \ell \left(y_i' f_{\bW}(\bx_i')\right).\]
Further, let
\[  a_{\delta,m} =\frac{2C_{\sigma, b} \, p^{\frac{3}{2}} \big(\sqrt{\log(\frac{ m}{\delta})}+\sqrt{p} \big)}{\sqrt{m}}, 
\qquad
b_{\delta,m} = C_{\sigma,b}\, p\Bigl(\sqrt p+\sqrt{\tfrac{\log(1/\delta)}{m}}\Bigr). \]

\begin{lemma}
\label{lem:sgd-one-step-stability}
Suppose Assumption~\ref{ass:sigma} holds and the event $\mathcal E_\delta$ occurs. Let $\{\bW_t\}_{t\ge0}$ and $\{\bW_t^{(i)}\}_{t\ge0}$ be two coupled runs of Algorithm~\ref{alg:dp-lambda-minibatch} with $\lambda>0$ on $S$ and $S^{(i)}$, respectively, driven by the same initialization, the same mini-batch sequence, and the same Gaussian noise sequence. Assume $\eta_{\rm eff}\le \frac{1}{12C_{\sigma,b}p^3}$ and \(m\gtrsim C_{\sigma,b}^2p^3(\log(m/\delta)+p)R_*^4\). Then, for all $t\ge0$, it holds that
\[
\Delta_{t+1}^{(i)}
\le
\bigl(1+a_{\delta,m}\eta(\mathcal L_{\mathcal B_t}(\bW_t)+\mathcal L_{\mathcal B_t}^{(i)}(\bW_t^{(i)}))\bigr)\Delta_t^{(i)}
+\frac{\eta b_{\delta,m}}{B}\mathbf 1_{\{i\in\mathcal B_t\}}
\bigl(\ell(y_i f_{\bW_t}(\bx_i))+\ell(y_i' f_{\bW_t^{(i)}}(\bx_i'))\bigr).
\]
\end{lemma}

\begin{proof}
For $z=(\bx,y)$, define $\phi_z(\bW)=\operatorname{clip}_{C_{\rm clip}}\bigl(\nabla\ell(yf_{\bW}(\bx))\bigr)$. 
Let $\zeta_t=\frac{C_{\rm clip}}{B}\xi_t$ denote the common additive Gaussian perturbation in the two coupled runs. Since the two algorithms use the same initialization, the same mini-batches, and the same Gaussian noise sequence, the projected updates can be written as
\[
\bW_{t+1}=\Pi_{\mathcal K}\bigl(\bW_t-\eta(v_t+\zeta_t)\bigr),
\qquad
\bW_{t+1}^{(i)}=\Pi_{\mathcal K}\bigl(\bW_t^{(i)}-\eta(v_t^{(i)}+\zeta_t)\bigr),
\]
where $v_t=\frac1B\sum_{j\in\mathcal B_t}\phi_{z_j}(\bW_t)$ and 
$v_t^{(i)}=\frac1B\sum_{j\in\mathcal B_t}\phi_{z_j^{(i)}}(\bW_t^{(i)})$. 

By the non-expansiveness of the Euclidean projection,
\begin{equation}
    \Delta_{t+1}^{(i)}
\le \bigl\|\bW_t-\eta v_t-\bigl(\bW_t^{(i)}-\eta v_t^{(i)}\bigr)\bigr\|_2,
\label{eq:stability-proj-step}
\end{equation} 
where the common noise term cancels exactly.

We split off the possibly replaced sample. Define $v_{t,i}^{\mathrm{sh}}(\bW)=\frac1B\sum_{j\in\mathcal B_t\setminus\{i\}}\phi_{z_j}(\bW)$. 
Then
\[
v_t=v_{t,i}^{\mathrm{sh}}(\bW_t)+\frac1B\mathbf 1_{\{i\in\mathcal B_t\}}\phi_{z_i}(\bW_t),
\qquad
v_t^{(i)}=v_{t,i}^{\mathrm{sh}}(\bW_t^{(i)})+\frac1B\mathbf 1_{\{i\in\mathcal B_t\}}\phi_{z_i'}(\bW_t^{(i)}).
\]
Substituting these decompositions into \eqref{eq:stability-proj-step} gives
\begin{align}
\Delta_{t+1}^{(i)}
&\le
\bigl\|\bW_t-\eta v_{t,i}^{\mathrm{sh}}(\bW_t)-\bigl(\bW_t^{(i)}-\eta v_{t,i}^{\mathrm{sh}}(\bW_t^{(i)})\bigr)\bigr\|_2 \nonumber\\
&\quad+
\frac{\eta}{B}\mathbf 1_{\{i\in\mathcal B_t\}}
\bigl(\|\phi_{z_i}(\bW_t)\|_2+\|\phi_{z_i'}(\bW_t^{(i)})\|_2\bigr).
\label{eq:stability-split}
\end{align}

We first bound the first term in \eqref{eq:stability-split}.  For a fixed $z=(\bx,y)$, consider the clipped-gradient map $\phi_z$ defined above. This map is locally Lipschitz and hence differentiable almost everywhere. We use $\nabla\phi_z(\bW)$ to denote its Jacobian with respect to $\bW$. When clipping is inactive,
\[
\nabla\phi_z(\bW)
=
\ell''(yf_{\bW}(\bx))\nabla f_{\bW}(\bx)\nabla f_{\bW}(\bx)^\top
+y\ell'(yf_{\bW}(\bx))\nabla^2f_{\bW}(\bx).
\]
When clipping is active, writing $\alpha_z(\bW)=\frac{C_{\rm clip}}{\|\nabla\ell(yf_{\bW}(\bx))\|_2}$ and 
$u_z(\bW)=\frac{\nabla\ell(yf_{\bW}(\bx))}{\|\nabla\ell(yf_{\bW}(\bx))\|_2}$,  
we have
\[
\nabla\phi_z(\bW)
=
\alpha_z(\bW)\bigl(\mathbf I-u_z(\bW)u_z(\bW)^\top\bigr)\nabla^2\ell(yf_{\bW}(\bx)).
\]
Since $u_z(\bW)$ is parallel to $\nabla f_{\bW}(\bx)$, it holds  $\bigl(\mathbf I-u_z(\bW)u_z(\bW)^\top\bigr)\nabla f_{\bW}(\bx)\nabla f_{\bW}(\bx)^\top=\mathbf 0$. 
Thus, in both cases we can write
\[
\nabla\phi_z(\bW)=P_z(\bW)+E_z(\bW),
\]
where $P_z(\bW)$ is the positive-semidefinite rank-one term $\ell''(yf_{\bW}(\bx))\nabla f_{\bW}(\bx)\nabla f_{\bW}(\bx)^\top$ when clipping is inactive and is zero when clipping is active. The remaining term is denoted by $E_z(\bW)$.

On $\mathcal E_\delta$, we have
\[
P_z(\bW)\succeq0,
\qquad
\|P_z(\bW)\|_{\mathrm{op}}\le \rho_{\rm clip}G_\delta^2,
\qquad
\|E_z(\bW)\|_{\mathrm{op}}\le \frac{a_{\delta,m}}{2}\ell(yf_{\bW}(\bx)).
\]
Indeed, in the inactive regime, using $\ell''(u)\le|\ell'(u)|$ gives
\[
\|P_z(\bW)\|_{\mathrm{op}}
\le
\min\{C_{\rm clip},G_\delta\}G_\delta
=
\rho_{\rm clip}G_\delta^2,
\]
while in the active regime $P_z(\bW)=0$. Moreover, in both regimes, $\alpha_z(\bW)\le1$, $|\ell'(u)|\le\ell(u)$, and Lemma~\ref{lem:hessian} imply the stated bound on $E_z(\bW)$.

Averaging over the shared samples gives
\[
\nabla v_{t,i}^{\mathrm{sh}}(\bW)=P_{t,i}(\bW)+E_{t,i}(\bW),
\]
where $P_{t,i}(\bW)\succeq0, $ $
\|P_{t,i}(\bW)\|_{\mathrm{op}}\le\rho_{\rm clip}G_\delta^2 $ and $
\|E_{t,i}(\bW)\|_{\mathrm{op}}\le\frac{a_{\delta,m}}{2}\mathcal L_{t,i}^{\mathrm{sh}}(\bW)$.

By absolute continuity along the line segment,
\[
v_{t,i}^{\mathrm{sh}}(\bW_t)-v_{t,i}^{\mathrm{sh}}(\bW_t^{(i)})
=
(\bar P_{t,i}+\bar E_{t,i})(\bW_t-\bW_t^{(i)}),
\]
where $\bar P_{t,i}=\int_0^1P_{t,i}(\bW_{\alpha,t}^{(i)})\,d\alpha$   and $ \bar E_{t,i}=\int_0^1E_{t,i}(\bW_{\alpha,t}^{(i)})\,d\alpha.$

Since $\bar P_{t,i}\succeq0$ and $\|\bar P_{t,i}\|_{\mathrm{op}}\le\rho_{\rm clip}G_\delta^2$, using $G_\delta^2\le4C_{\sigma,b}p^3$ gives
\[
\eta\|\bar P_{t,i}\|_{\mathrm{op}}
\le
\eta_{\rm eff}G_\delta^2
\le
\frac13,
\qquad
\|\mathbf I-\eta\bar P_{t,i}\|_{\mathrm{op}}\le1.
\]

We next control $\bar E_{t,i}$. Both iterates lie in $\mathcal K=\mathcal B(\bW_0,R_*)$, hence the whole segment $\{\bW_{\alpha,t}^{(i)}:\alpha\in[0,1]\}$ lies in $\mathcal K$, and  $\|\bW_t-\bW_t^{(i)}\|_2\le 2R_*$. 
Under the width condition \(m\gtrsim C_{\sigma,b}^2p^3(\log(m/\delta)+p)R_*^4\), applying Lemma~\ref{lem:quasi-convexity} with $G=\mathcal L_{t,i}^{\mathrm{sh}}$ and $D=2R_*$  to the shared mini-batch loss $\mathcal L_{t,i}^{\mathrm{sh}}$ yields
\[
\max_{\alpha\in[0,1]}\mathcal L_{t,i}^{\mathrm{sh}}(\bW_{\alpha,t}^{(i)})
\le
2\max\bigl\{\mathcal L_{t,i}^{\mathrm{sh}}(\bW_t),\mathcal L_{t,i}^{\mathrm{sh}}(\bW_t^{(i)})\bigr\}.
\]
Since the loss is nonnegative, then  $\mathcal L_{t,i}^{\mathrm{sh}}(\bW_t)\le\mathcal L_{\mathcal B_t}(\bW_t)$ and 
$\mathcal L_{t,i}^{\mathrm{sh}}(\bW_t^{(i)})\le\mathcal L_{\mathcal B_t}^{(i)}(\bW_t^{(i)})$. Therefore, 
\[
\|\bar E_{t,i}\|_{\mathrm{op}}
\le
a_{\delta,m}\bigl(\mathcal L_{\mathcal B_t}(\bW_t)+\mathcal L_{\mathcal B_t}^{(i)}(\bW_t^{(i)})\bigr).
\]
Therefore,
\[
\bigl\|\bW_t-\eta v_{t,i}^{\mathrm{sh}}(\bW_t)-\bigl(\bW_t^{(i)}-\eta v_{t,i}^{\mathrm{sh}}(\bW_t^{(i)})\bigr)\bigr\|_2
\le
\bigl(1+a_{\delta,m}\eta(\mathcal L_{\mathcal B_t}(\bW_t)+\mathcal L_{\mathcal B_t}^{(i)}(\bW_t^{(i)}))\bigr)\Delta_t^{(i)}.
\]

It remains to bound the two single-sample terms in \eqref{eq:stability-split}. Since clipping cannot increase the norm,
\[
\|\phi_z(\bW)\|_2
\le
\|\nabla\ell(yf_{\bW}(\bx))\|_2
\le
b_{\delta,m}\ell(yf_{\bW}(\bx)),
\]
where we used $|\ell'(u)|\le\ell(u)$ and Lemma~\ref{lem:hessian}. Hence
\[
\|\phi_{z_i}(\bW_t)\|_2\le b_{\delta,m}\ell(y_i f_{\bW_t}(\bx_i)),
\qquad
\|\phi_{z_i'}(\bW_t^{(i)})\|_2\le b_{\delta,m}\ell(y_i' f_{\bW_t^{(i)}}(\bx_i')).
\]
Substituting these estimates into \eqref{eq:stability-split} yields
\[
\Delta_{t+1}^{(i)}
\le
\bigl(1+a_{\delta,m}\eta(\mathcal L_{\mathcal B_t}(\bW_t)+\mathcal L_{\mathcal B_t}^{(i)}(\bW_t^{(i)}))\bigr)\Delta_t^{(i)}
+\frac{\eta b_{\delta,m}}{B}\mathbf 1_{\{i\in\mathcal B_t\}}
\bigl(\ell(y_i f_{\bW_t}(\bx_i))+\ell(y_i' f_{\bW_t^{(i)}}(\bx_i'))\bigr).
\]
This completes the proof.
\end{proof}

We first control the cumulative mini-batch losses along the two coupled trajectories.
For $\delta_f\in(0,1)$, define the initialization-output event
\[
\mathcal E_{\delta_f}^{\mathrm{out}}
\!=\!
\Big\{
\max_{1\le j\le n}|f_{\bW_0}(\bx_j)|
\!\le\! B_b\sqrt{2p\log\big({4n}/{\delta_f}\big)}
\ \text{ and }\
\max_{1\le j\le n}|f_{\bW_0}(\bx_j')|
\!\le\! B_b\sqrt{2p\log\big({4n}/{\delta_f}\big)}
\Big\},
\]
where $\widetilde S=\{(\bx_j',y_j')\}_{j=1}^n$ is an independent copy of $S$.

\begin{lemma}
\label{lem:loss-bound}
Suppose Assumption~\ref{ass:sigma} holds. Then $\mathbb P\big(\mathcal E_{\delta_f}^{\mathrm{out}}\big)\ge 1-\delta_f$.
Moreover, on the event $\mathcal E_\delta\cap\mathcal E_{\delta_f}^{\mathrm{out}}$, for every $t\ge0$ and every $i,j\in[n]$,
\[
\max\bigl\{|f_{\bW_t}(\bx_j)|,|f_{\bW_t^{(i)}}(\bx_j^{(i)})|\bigr\}\le F_{\delta,\delta_f}
 \text{ and } 
\max\Bigl\{\ell\bigl(y_jf_{\bW_t}(\bx_j)\bigr),\ell\bigl(y_j^{(i)}f_{\bW_t^{(i)}}(\bx_j^{(i)})\bigr)\Bigr\}\le U_{\delta,\delta_f},
\]
where $F_{\delta,\delta_f}=B_b\sqrt{2p\log\Bigl(\frac{4n}{\delta_f}\Bigr)}+b_{\delta,m}R_*$ and 
$U_{\delta,\delta_f}=\log\bigl(1+e^{F_{\delta,\delta_f}}\bigr)\le\log2+F_{\delta,\delta_f}$. 
\end{lemma}

\begin{proof}
We first control the initial outputs. Fix any input $\bx\in\R^d$. Recall that
\[
f_{\bW_0}(\bx)=\frac{1}{\sqrt m}\bc^\top\bh\Bigl(\sigma\Bigl(\frac{1}{\sqrt d}\bW_0\bh(\bx)\Bigr)\Bigr).
\]
Conditioned on $(\bW_0,\bx)$, the only randomness comes from $\bc\sim\mathcal N(0,\mathbf I_{mp})$. Hence,
\[
f_{\bW_0}(\bx)\sim\mathcal N\Bigl(0,\frac{1}{m}\Bigl\|\bh\Bigl(\sigma\Bigl(\frac{1}{\sqrt d}\bW_0\bh(\bx)\Bigr)\Bigr)\Bigr\|_2^2\Bigr).
\]
Since each spline basis function is bounded by $B_b$ on the range of $\sigma$, then we have
\[
\Bigl\|\bh\Bigl(\sigma\Bigl(\frac{1}{\sqrt d}\bW_0\bh(\bx)\Bigr)\Bigr)\Bigr\|_2^2\le mpB_b^2.
\]
Applying the Gaussian tail bound to the $2n$ points $\bx_1,\ldots,\bx_n,\bx_1',\ldots,\bx_n'$ and taking a union bound yields $\mathbb P\bigl(\mathcal E_{\delta_f}^{\mathrm{out}}\bigr)\ge1-\delta_f$. 

We next extend the bound from the initialization to the whole projected trajectories. On the event $\mathcal E_\delta$ (see \eqref{eq:good-event}), Lemma~\ref{lem:hessian} implies that
\[
\|\nabla f_{\bW}(\bx)\|_2\le b_{\delta,m},\qquad \forall\,\bW\in\R^{mdp},\ \bx\in\mathcal X.
\]
Since both trajectories are projected onto $\mathcal K=\mathcal B(\bW_0,R_*)$, we have
\[
\|\bW_t-\bW_0\|_2\le R_*,
\qquad
\|\bW_t^{(i)}-\bW_0\|_2\le R_*.
\]
Therefore, by the mean-value theorem,
\[
\max\bigl\{|f_{\bW_t}(\bx_j)-f_{\bW_0}(\bx_j)|,|f_{\bW_t^{(i)}}(\bx_j^{(i)})-f_{\bW_0}(\bx_j^{(i)})|\bigr\}\le b_{\delta,m}R_*.
\]
Since $\bx_j^{(i)}$ is either $\bx_j$ or $\bx_i'$, combining this estimate with $\mathcal E_{\delta_f}^{\mathrm{out}}$ gives
\[
\max\bigl\{|f_{\bW_t}(\bx_j)|,|f_{\bW_t^{(i)}}(\bx_j^{(i)})|\bigr\}
\le B_b\sqrt{2p\log\Bigl(\frac{4n}{\delta_f}\Bigr)}+b_{\delta,m}R_*
=F_{\delta,\delta_f}.
\]
Finally, the logistic loss satisfies $\ell(u)=\log(1+e^{-u})\le\log(1+e^{|u|})\le\log2+|u|$. 
Hence,
\[
\max\Bigl\{\ell\bigl(y_jf_{\bW_t}(\bx_j)\bigr),\ell\bigl(y_j^{(i)}f_{\bW_t^{(i)}}(\bx_j^{(i)})\bigr)\Bigr\}\le U_{\delta,\delta_f},
\]
which completes the proof.
\end{proof}

\begin{lemma}[Freedman's inequality] \label{lem:freedman}
Let $\{X_t,\mathscr F_t\}_{t=1}^T$ be a martingale difference sequence, that is,
$X_t$ is $\mathscr F_t$-measurable and $\mathbb E[X_t\mid \mathscr F_{t-1}]=0$  a.s. for all $t\in[T]$.
If  $|X_t|\le L$ a.s. for all  $ t\in[T]$,  and the predictable quadratic variation $V_T =\sum_{t=1}^T \mathbb E[X_t^2\mid \mathscr F_{t-1}] \le v$ a.s.,
 then for any $\delta\in(0,1)$,
\[ \mathbb P\Big( \sum_{t=1}^T X_t
\ge \sqrt{2v\log( {1}/{\delta})}
+ \frac{2L}{3}\log( {1}/{\delta}) \Big)
\le \delta. \]
\end{lemma}
 
\begin{lemma}  
\label{lem:cumulative-minibatch-loss}
Suppose the assumptions of Lemma~\ref{lem:sgd-one-step-stability} hold, and let
$\delta_{\mathrm{mb}}\in(0,1)$. Then, conditioned on $S$, $S^{(i)}$, and the initialization satisfying $\mathcal E_\delta\cap\mathcal E_{\delta_f}^{\mathrm{out}}$, with probability at least $1-2\delta_{\mathrm{mb}}$ over the remaining algorithmic randomness,
\begin{align*}
\sum_{t=0}^{T-1}\mathcal L_{\mathcal B_t}(\bW_t)
&\le
\sum_{t=0}^{T-1}\mathcal L_S(\bW_t)
+
C\,U_{\delta,\delta_f}
\Big(
\sqrt{\frac{T\log(1/\delta_{\mathrm{mb}})}{B}}
+
\log\Big(\frac{1}{\delta_{\mathrm{mb}}}\Big)
\Big),\\
\sum_{t=0}^{T-1}\mathcal L_{\mathcal B_t}^{(i)}(\bW_t^{(i)})
&\le
\sum_{t=0}^{T-1}\mathcal L_{S^{(i)}}(\bW_t^{(i)})
+
C\,U_{\delta,\delta_f}
\Big(
\sqrt{\frac{T\log(1/\delta_{\mathrm{mb}})}{B}}
+
\log\Big(\frac{1}{\delta_{\mathrm{mb}}}\Big)
\Big),
\end{align*}
where $C>0$ is an absolute constant.
\end{lemma}

\begin{proof}
We prove the first inequality, and the second one follows in the same way.
Define the pre-sampling filtration
\[
\mathscr F_t=\sigma\bigl(S,S^{(i)},\bW_0,\mathcal B_0,\ldots,\mathcal B_{t-1},Z_0,\ldots,Z_{t-1}\bigr),
\qquad t\in[T]\cup\{0\}.
\]
Thus, $\mathscr F_t$ contains all randomness revealed before the fresh mini-batch $\mathcal B_t$ is drawn.

For each $t=0,\ldots,T-1$, define $X_t=\mathcal L_{\mathcal B_t}(\bW_t)-\mathcal L_S(\bW_t)$. 
Conditioned on $\mathscr F_t$, the iterate $\bW_t$ is fixed and the only randomness in $X_t$ comes from the fresh mini-batch $\mathcal B_t$. Since $\mathcal B_t$ is sampled uniformly without replacement, it holds  $\mathbb E_\A[X_t\mid\mathscr F_t]=0$. 
Moreover, on the event $\mathcal E_\delta\cap\mathcal E_{\delta_f}^{\mathrm{out}}$, Lemma~\ref{lem:loss-bound} gives
\[
|X_t|\le \sup_{z\in S}|\ell(yf_{\bW_t}(\bx))|\le U_{\delta,\delta_f}.
\]

We next bound the conditional variance. Conditioned on $\mathscr F_t$, let
$u_j=\ell(y_jf_{\bW_t}(\bx_j))$, so that $0\le u_j\le U_{\delta,\delta_f}$ for all $j\in[n]$. By Lemma~\ref{lem:finite-pop-var},
\[
\mathbb E_A[X_t^2\mid\mathscr F_t]\le \frac{U_{\delta,\delta_f}^2}{B}.
\]
Thus, regarding $X_t$ as the $(t+1)$-st martingale difference with respect to the filtration $\{\mathscr F_t\}_{t=0}^T$, Lemma~\ref{lem:freedman} with $v=TU_{\delta,\delta_f}^2/B$ implies that, with probability at least $1-\delta_{\mathrm{mb}}$,
\[
\sum_{t=0}^{T-1}X_t
\le
\sqrt{\frac{2TU_{\delta,\delta_f}^2}{B}\log\Big(\frac{1}{\delta_{\mathrm{mb}}}\Big)}
+
\frac{2U_{\delta,\delta_f}}{3}\log\Big(\frac{1}{\delta_{\mathrm{mb}}}\Big).
\]
Using
\[
\sum_{t=0}^{T-1}\mathcal L_{\mathcal B_t}(\bW_t)
=
\sum_{t=0}^{T-1}\mathcal L_S(\bW_t)+\sum_{t=0}^{T-1}X_t
\]
proves the first inequality.

The second inequality follows in the same way. A union bound over the two estimates yields probability at least $1-2\delta_{\mathrm{mb}}$.
\end{proof}

For each $i\in[n]$ and $t\ge0$, define $M_t^{(i)}=\eta a_{\delta,m}\bigl(\mathcal L_{\mathcal B_t}(\bW_t)+\mathcal L_{\mathcal B_t}^{(i)}(\bW_t^{(i)})\bigr)$. 
Let $\mathcal E_{\mathrm{opt}}^{(i)}$ denote the event on which $\frac1T\sum_{t=0}^{T-1}\mathcal L_S(\bW_t)\le A_{\mathrm{corr}}$ and 
$\frac1T\sum_{t=0}^{T-1}\mathcal L_{S^{(i)}}(\bW_t^{(i)})\le A_{\mathrm{corr}}$ hold.

For each $i\in[n]$, let $\mathcal E_{\mathrm{mb}}^{(i)}$ be the event on which both inequalities in Lemma~\ref{lem:cumulative-minibatch-loss} hold, namely,
\begin{align*}
\sum_{t=0}^{T-1}\mathcal L_{\mathcal B_t}(\bW_t)
&\le
\sum_{t=0}^{T-1}\mathcal L_S(\bW_t)
+C\,U_{\delta,\delta_f}\Bigl(\sqrt{\frac{T\log(1/\delta_{\mathrm{mb}})}{B}}+\log\Bigl(\frac{1}{\delta_{\mathrm{mb}}}\Bigr)\Bigr),\\
\sum_{t=0}^{T-1}\mathcal L_{\mathcal B_t}^{(i)}(\bW_t^{(i)})
&\le
\sum_{t=0}^{T-1}\mathcal L_{S^{(i)}}(\bW_t^{(i)})
+C\,U_{\delta,\delta_f}\Bigl(\sqrt{\frac{T\log(1/\delta_{\mathrm{mb}})}{B}}+\log\Bigl(\frac{1}{\delta_{\mathrm{mb}}}\Bigr)\Bigr).
\end{align*}

\begin{lemma} \label{lem:control-M-corr}
Suppose the assumptions of Lemma~\ref{lem:cumulative-minibatch-loss} hold. Assume
\[
m\gtrsim C_{\sigma,b}^2p^3\bigl(\log(m/\delta)+p\bigr)\eta^2
\Bigl(TA_{\mathrm{corr}}+U_{\delta,\delta_f}\sqrt{\frac{T\log(1/\delta_{\mathrm{mb}})}{B}}
+U_{\delta,\delta_f}\log\Bigl(\frac{1}{\delta_{\mathrm{mb}}}\Bigr)\Bigr)^2.
\]
Then, on the event $\mathcal E_{\mathrm{opt}}^{(i)}\cap\mathcal E_{\mathrm{mb}}^{(i)}\cap\mathcal E_{\delta_f}^{\mathrm{out}}$,
it holds
$\sum_{t=0}^{T-1}M_t^{(i)}\le1$. 
\end{lemma}

\begin{proof}
On the event $\mathcal E_{\mathrm{mb}}^{(i)}$, Lemma~\ref{lem:cumulative-minibatch-loss} implies
\begin{align*}
&\sum_{t=0}^{T-1}M_t^{(i)}
 =\eta a_{\delta,m}\sum_{t=0}^{T-1}
\bigl(\mathcal L_{\mathcal B_t}(\bW_t)+\mathcal L_{\mathcal B_t}^{(i)}(\bW_t^{(i)})\bigr)\\
&\le
\eta a_{\delta,m}\Bigl[
\sum_{t=0}^{T-1}\mathcal L_S(\bW_t)
+\sum_{t=0}^{T-1}\mathcal L_{S^{(i)}}(\bW_t^{(i)})
+2C\,U_{\delta,\delta_f}\Bigl(\sqrt{\frac{T\log(1/\delta_{\mathrm{mb}})}{B}}
+\log\Bigl(\frac{1}{\delta_{\mathrm{mb}}}\Bigr)\Bigr)
\Bigr].
\end{align*}
On $\mathcal E_{\mathrm{opt}}^{(i)}$, each of the first two sums is at most $TA_{\mathrm{corr}}$. Hence,
\[
\sum_{t=0}^{T-1}M_t^{(i)}
\le
2\eta a_{\delta,m}\Bigl(TA_{\mathrm{corr}}
+C\,U_{\delta,\delta_f}\Bigl(\sqrt{\frac{T\log(1/\delta_{\mathrm{mb}})}{B}}
+\log\Bigl(\frac{1}{\delta_{\mathrm{mb}}}\Bigr)\Bigr)\Bigr).
\]
The claim follows by substituting the definition of $a_{\delta,m}$ and using the assumed lower bound on $m$.
\end{proof}

Define $R_t^{(i)}=\frac{\eta b_{\delta,m}}{B}\mathbf 1_{\{i\in\mathcal B_t\}}
\bigl(\ell(y_i f_{\bW_t}(\bx_i))+\ell(y_i' f_{\bW_t^{(i)}}(\bx_i'))\bigr)$. 

\begin{lemma}\label{lem:stability-good-event}
Under the assumptions of Lemma~\ref{lem:control-M-corr}, define $\mathcal E_{\mathrm{stab}}^{(i)} = \mathcal E_{\mathrm{opt}}^{(i)}\cap\mathcal E_{\mathrm{mb}}^{(i)} \cap\mathcal E_{\delta_f}^{\mathrm{out}}$. 
Then, on the event $\mathcal E_{\mathrm{stab}}^{(i)}$, for every $t=0,\ldots,T-1$, it holds $\Delta_{t+1}^{(i)}\le e\sum_{k=0}^t R_k^{(i)}$. 
\end{lemma}

\begin{proof}
By Lemma~\ref{lem:sgd-one-step-stability}, for every $t=0,\ldots,T-1$,
\[
\Delta_{t+1}^{(i)}
\le
\bigl(1+M_t^{(i)}\bigr)\Delta_t^{(i)}+R_t^{(i)}.
\]
Iterating this recursion and using $\Delta_0^{(i)}=0$, we obtain
\[
\Delta_{t+1}^{(i)}
\le
\sum_{k=0}^t R_k^{(i)}
\prod_{s=k+1}^t\bigl(1+M_s^{(i)}\bigr).
\]

On the event $\mathcal E_{\mathrm{stab}}^{(i)}$, Lemma~\ref{lem:control-M-corr} yields $\sum_{s=0}^{T-1}M_s^{(i)}\le 1$. 
Since $M_s^{(i)}\ge0$, for every $0\le k\le t\le T-1$,
\[ \prod_{s=k+1}^t\bigl(1+M_s^{(i)}\bigr)
\le
\exp\Bigl(\sum_{s=k+1}^tM_s^{(i)}\Bigr)
\le e, \]
where we used $1+x\le e^x$ for $x\ge0$. Substituting this estimate into the previous display gives
\[ \Delta_{t+1}^{(i)}\le e\sum_{k=0}^t R_k^{(i)}, \]
which completes the proof.
\end{proof}

For each $i\in[n]$, define the overall good event $\widehat{\mathcal E}^{(i)} = \mathcal E_\delta\cap\mathcal E_{\mathrm{opt}}^{(i)} \cap\mathcal E_{\mathrm{mb}}^{(i)} \cap\mathcal E_{\delta_f}^{\mathrm{out}}$. 

\begin{theorem}[On-average argument stability] \label{thm:on-average-stability}
Under the assumptions of Lemmas~\ref{lem:control-M-corr} and~\ref{lem:stability-good-event}, assume that $\mathbb P\bigl((\widehat{\mathcal E}^{(i)})^c\bigr)\le\delta_{\mathrm{stab}}$ for all $i\in[n]$. 
Then
\[
\mathbb E_{S,\widetilde S,\A}\Bigl[\frac1n\sum_{i=1}^n\|\bW_T-\bW_T^{(i)}\|_2\Bigr]
\le
\frac{2e\eta b_{\delta,m}}{n}
\sum_{t=0}^{T-1}\mathbb E_{S,\A}\bigl[\mathcal L_S(\bW_t)\bigr]
+2R_*\delta_{\mathrm{stab}}.
\]
\end{theorem}

\begin{proof}
Fix $i\in[n]$. On the event $\widehat{\mathcal E}^{(i)}$, Lemma~\ref{lem:stability-good-event} yields $\Delta_T^{(i)}\le e\sum_{t=0}^{T-1}R_t^{(i)}$. 
Since both trajectories are projected onto $\mathcal K=\mathcal B(\bW_0,R_*)$, we also have $\Delta_T^{(i)}\le2R_*$. Therefore,
\begin{align*}
\mathbb E_{S,\widetilde S,\A}[\Delta_T^{(i)}]
&\le
e\,\mathbb E_{S,\widetilde S,\A}\Bigl[\sum_{t=0}^{T-1}R_t^{(i)}\mathbf 1_{\widehat{\mathcal E}^{(i)}}\Bigr]
+2R_*\mathbb P\bigl((\widehat{\mathcal E}^{(i)})^c\bigr) \le
e\sum_{t=0}^{T-1}\mathbb E_{S,\widetilde S,\A}[R_t^{(i)}]
+2R_*\delta_{\mathrm{stab}}.
\end{align*}
Averaging over $i\in[n]$ gives
\[
\mathbb E_{S,\widetilde S,\A}\Bigl[\frac1n\sum_{i=1}^n\Delta_T^{(i)}\Bigr]
\le
\frac{e}{n}\sum_{t=0}^{T-1}\sum_{i=1}^n\mathbb E_{S,\widetilde S,\A}[R_t^{(i)}]
+2R_*\delta_{\mathrm{stab}}.
\]

We now estimate the first term. Since $\bW_t$ and $\bW_t^{(i)}$ are determined before the fresh mini-batch $\mathcal B_t$ is drawn, they are independent of $\mathcal B_t$. Hence,
\[
\mathbb E\Bigl[\frac1B\mathbf 1_{\{i\in\mathcal B_t\}}\ \Big|\ S,\widetilde S,\bW_t,\bW_t^{(i)}\Bigr]=\frac1n.
\]
It follows from the definition of $R_t^{(i)}$ that
\[
\frac1n\sum_{i=1}^n\mathbb E_{S,\widetilde S,\A}[R_t^{(i)}]
=
\frac{\eta b_{\delta,m}}{n}
\mathbb E_{S,\widetilde S,\A}\Bigl[
\frac1n\sum_{i=1}^n
\bigl(\ell(y_i f_{\bW_t}(\bx_i))
+\ell(y_i' f_{\bW_t^{(i)}}(\bx_i'))\bigr)
\Bigr].
\]
Using $\frac1n\sum_{i=1}^n\ell(y_i f_{\bW_t}(\bx_i))=\mathcal L_S(\bW_t)$ 
and the fact that $S^{(i)}$ has the same distribution as $S$, we obtain
\[
\frac1n\sum_{i=1}^n\mathbb E_{S,\widetilde S,\A}[R_t^{(i)}]
=
\frac{2\eta b_{\delta,m}}{n}\,
\mathbb E_{S,\A}\bigl[\mathcal L_S(\bW_t)\bigr].
\]
Summing over $t=0,\ldots,T-1$ proves
\[
\mathbb E_{S,\widetilde S,\A}\Bigl[\frac1n\sum_{i=1}^n\|\bW_T-\bW_T^{(i)}\|_2\Bigr]
\le
\frac{2e\eta b_{\delta,m}}{n}
\sum_{t=0}^{T-1}\mathbb E_{S,\A}\bigl[\mathcal L_S(\bW_t)\bigr]
+2R_*\delta_{\mathrm{stab}},
\]
which completes the proof.
\end{proof}
\begin{lemma}\label{lem:conditional-failure}
Let $\mathcal I$ denote the initialization and let $\mathcal R$ denote all remaining randomness. 
If an event $\mathcal E$ satisfies $\mathbb P_{\mathcal I,\mathcal R}(\mathcal E^c)\le \alpha\beta$, 
then, with probability at least $1-\alpha$ over $\mathcal I$,
\[
\mathbb P_{\mathcal R}(\mathcal E^c\mid \mathcal I)\le \beta.
\]
\end{lemma}
\begin{proof}
The claim follows from Markov's inequality applied to
$\mathbb P_{\mathcal R}(\mathcal E^c\mid\mathcal I)$.
\end{proof}

Recall that $U_{\delta,\delta_f}\lesssim \sqrt{p\log(n/\delta)}+p^{3/2}R_*$. 

\begin{theorem}[Restatement of Theorem~\ref{thm:lambda1-gen}]
Suppose Assumption~\ref{ass:sigma} holds and the parameter conditions of Theorem~\ref{thm:lambda-positive-ntk-opt} are satisfied. Assume further that, conditional on the initialization, an i.i.d.\ sample $S\sim\mathcal P^n$ satisfies Assumption~\ref{ass:ntk} with margin $\gamma$ with probability at least $1-\delta_{\rm sep}$, where $\delta_{\rm sep}\le\delta/8$. Let $\{\bW_t\}_{t=0}^T$ be generated by Algorithm~\ref{alg:dp-lambda-minibatch} with $C_{\mathrm{clip}}>0$ and $\eta_{\rm eff}\le \frac{1}{12C_{\sigma,b}p^3}.$ 
Assume
\[
m\gtrsim R_*^2\log(m/\delta)\eta^2\Bigl(TA_{\mathrm{corr}}+(\log(n/\delta)+R_*)\Bigl(\sqrt{\frac{T\log(1/\delta)}{B}}+\log\Bigl(\frac{1}{\delta}\Bigr)\Bigr)\Bigr)^2.
\]
Then, with probability at least $1-\delta$ over the initialization,
\[
\frac1T\sum_{t=0}^{T-1}\mathbb E_{S,\A}[\mathcal L(\bW_t)]
\lesssim
\Bigl(1+\frac{\eta T\log(1/\delta)}{n}\Bigr)
\Bigl(A_{\mathrm{corr}}+\log\Bigl(\frac n\delta\Bigr)(\sqrt m+R_*)\delta\Bigr).
\]
\end{theorem}

\begin{proof}
Let $\mathcal E_{\rm sep}(S;\gamma)$ denote the event that Assumption~\ref{ass:ntk} holds for $S$ with margin $\gamma$. Choose auxiliary failure probabilities $\delta_f=\delta_{\mathrm{mb}}=\frac{\delta}{8}$ and $\bar\delta_{\mathrm{opt}}=\frac{\delta^2}{64}$. Applying Theorem~\ref{thm:lambda-positive-ntk-opt} with failure probability $\bar\delta_{\mathrm{opt}}$ and Lemma~\ref{lem:conditional-failure} with $\alpha=\beta=\delta/8$, there exists an initialization event $\mathcal E_{\mathrm{cond}}$ with $\mathbb P(\mathcal E_{\mathrm{cond}}^c)\le\delta/8$ such that, on $\mathcal E_{\mathrm{cond}}$, the conditional probability that the optimization bound in Theorem~\ref{thm:lambda-positive-ntk-opt} fails for an i.i.d.\ sample satisfying Assumption~\ref{ass:ntk} is at most $\delta/8$. Here, replacing the failure probability by $\bar\delta_{\mathrm{opt}}$ only changes logarithmic factors, which are absorbed in $A_{\mathrm{corr}}$. We condition on the initialization event $\mathcal E_{\delta/8}\cap\mathcal E_{\mathrm{cond}}$. Since $S^{(i)}\overset{d}=S$, the sample-separability assumption gives
\[
\mathbb P\bigl(\mathcal E_{\rm sep}(S;\gamma)^c\bigr)
\le\delta_{\rm sep},
\qquad
\mathbb P\bigl(\mathcal E_{\rm sep}(S^{(i)};\gamma)^c\bigr)
\le\delta_{\rm sep}.
\]
Hence, by a union bound, both $S$ and $S^{(i)}$ satisfy
Assumption~\ref{ass:ntk} with the same margin lower bound $\gamma$
with probability at least $1-2\delta_{\rm sep}$, possibly with different separating directions.
Since $S^{(i)}\overset{d}=S$, the conditional optimization guarantee above applies to both coupled runs, yielding
\[
\mathbb P\bigl((\mathcal E_{\mathrm{opt}}^{(i)})^c\bigr)
\le
2\delta_{\rm sep}+\frac{\delta}{4}.
\] 
Moreover, Lemma~\ref{lem:cumulative-minibatch-loss} gives $\mathbb P\bigl(\mathcal E_{\delta/8}\cap\mathcal E_{\delta_f}^{\mathrm{out}}\cap(\mathcal E_{\mathrm{mb}}^{(i)})^c\bigr)\le 2\delta_{\mathrm{mb}}$. 
Together with $\mathbb P((\mathcal E_{\delta_f}^{\mathrm{out}})^c)\le\delta_f$ from Lemma~\ref{lem:loss-bound}, for $\widehat{\mathcal E}^{(i)}=\mathcal E_{\delta/8}\cap\mathcal E_{\delta_f}^{\mathrm{out}}\cap\mathcal E_{\mathrm{opt}}^{(i)}\cap\mathcal E_{\mathrm{mb}}^{(i)}$, 
we obtain
\[
\mathbb P\bigl((\widehat{\mathcal E}^{(i)})^c\bigr)\le \delta_f+2\delta_{\rm sep}+\frac{\delta}{4}+2\delta_{\mathrm{mb}}\le\delta,
\]
where we used $\delta_{\rm sep}\le\delta/8$.
Thus, Theorem~\ref{thm:on-average-stability} applies with $\delta_{\mathrm{stab}}=\delta$ and yields
\[
\mathbb E_{S,\widetilde S,\A}\Bigl[\frac1n\sum_{i=1}^n\|\bW_T-\bW_T^{(i)}\|_2\Bigr]
\le
\frac{2e\eta b_{\delta/8,m}}{n}\sum_{t=0}^{T-1}\mathbb E_{S,\A}\bigl[\mathcal L_S(\bW_t)\bigr]+2R_*\delta.
\]

We now pass from stability to generalization. On $\mathcal E_{\delta/8}$, Lemma~\ref{lem:hessian} gives $\|\nabla f_{\bW}(\bx)\|_2\le b_{\delta/8,m}$ for all $\bW\in\mathcal K, \bx\in\mathcal X$. 
Since $|\ell'(u)|\le1$ for the logistic loss, the loss is $b_{\delta/8,m}$-Lipschitz with respect to $\bW$ on this event. Applying Lemma~\ref{lem:connection}, we obtain
\begin{align*}
\mathbb E_{S,\A}\bigl[\mathcal L(\bW_T)-\mathcal L_S(\bW_T)\bigr]
&\le 2b_{\delta/8,m}\Bigl(\frac{2e\eta b_{\delta/8,m}}{n}\sum_{t=0}^{T-1}\mathbb E_{S,\A}\bigl[\mathcal L_S(\bW_t)\bigr]+2R_*\delta\Bigr)\\
&\lesssim \frac{\eta\log(1/\delta)}{n}\sum_{t=0}^{T-1}\mathbb E_{S,\A}\bigl[\mathcal L_S(\bW_t)\bigr]+\sqrt{\log(1/\delta)}\,R_*\delta,
\end{align*}
where we used $b_{\delta/8,m}\lesssim\sqrt{\log(1/\delta)}$.

We now derive the averaged population risk bound. The same stability argument applies to the algorithm stopped at any time $t$. Therefore, for every $t=0,\ldots,T-1$,
\[
\mathbb E_{S,\A}\bigl[\mathcal L(\bW_t)-\mathcal L_S(\bW_t)\bigr]
\lesssim
\frac{\eta\log(1/\delta)}{n}\sum_{s=0}^{T-1}\mathbb E_{S,\A}\bigl[\mathcal L_S(\bW_s)\bigr]+\sqrt{\log(1/\delta)}\,R_*\delta.
\]
Averaging this bound over $t=0,\ldots,T-1$ gives
\begin{align*}
\frac1T\!\sum_{t=0}^{T-1}\!\mathbb E_{S,\A}\bigl[\mathcal L(\bW_t)\bigr]
&\!\le\! \frac1T\!\sum_{t=0}^{T-1}\!\mathbb E_{S,\A}\bigl[\mathcal L_S(\bW_t)\bigr]
\!+\!\frac{\eta\log(1/\delta)}{n}\!\sum_{t=0}^{T-1}\!\mathbb E_{S,\A}\bigl[\mathcal L_S(\bW_t)\bigr]
\!+\!\sqrt{\log(1/\delta)}\,R_*\delta\\
&=\Bigl(1+\frac{\eta T\log(1/\delta)}{n}\Bigr)\frac1T\sum_{t=0}^{T-1}\mathbb E_{S,\A}\bigl[\mathcal L_S(\bW_t)\bigr]
+\sqrt{\log(1/\delta)}\,R_*\delta.
\end{align*}

Let $\mathcal E_{\mathrm{opt}}$ be the event on which $S$ satisfies Assumption~\ref{ass:ntk} with margin $\gamma$ and $\frac1T\sum_{t=0}^{T-1}\mathcal L_S(\bW_t)\le A_{\mathrm{corr}}$. 
Then, on $\mathcal E_{\mathrm{cond}}$, $\mathbb P(\mathcal E_{\mathrm{opt}}^c)\le\delta_{\rm sep}+\delta/8$. On the event $\mathcal E_{\delta/8}$, for any $\bx\in\mathcal X$, $y\in\{-1,+1\}$, and $\bW\in\mathcal B(\bW_0,R_*)$, it holds that
\[
|f_{\bW}(\bx)|\le |f_{\bW_0}(\bx)|+b_{\delta/8,m}R_*\le B_b\sqrt p\,\|\bc\|_2+b_{\delta/8,m}R_*\lesssim \log\Bigl(\frac n\delta\Bigr)(\sqrt m+R_*),
\]
where we used $\|\bc\|_2\lesssim\sqrt m$ on $\mathcal E_{\delta/8}$. Consequently,
\[
\ell(yf_{\bW}(\bx))=\log\bigl(1+e^{-yf_{\bW}(\bx)}\bigr)\le\log2+|f_{\bW}(\bx)|\lesssim\log\Bigl(\frac n\delta\Bigr)(\sqrt m+R_*).
\]
Since $\bW_t\in\mathcal B(\bW_0,R_*)$ for all $t$ by construction of Algorithm~\ref{alg:dp-lambda-minibatch}, the same bound holds for $\mathcal L_S(\bW_t)$ uniformly over $t$. Therefore, conditioning on the initialization event $\mathcal E_{\delta/8}\cap\mathcal E_{\mathrm{cond}}$ and decomposing over $\mathcal E_{\mathrm{opt}}$, we obtain
\begin{align*}
\frac1T\sum_{t=0}^{T-1}\mathbb E_{S,\A}\bigl[\mathcal L_S(\bW_t)\bigr]
&\le \mathbb E_S\Bigl[\mathbb E_{\A}\Bigl[\frac1T\sum_{t=0}^{T-1}\mathcal L_S(\bW_t)\mathbf 1_{\mathcal E_{\mathrm{opt}}}\,\Big|\,S,\bW_0,\bc\Bigr]\Bigr]\\
&\quad+\mathbb E_S\Bigl[\mathbb E_{\A}\Bigl[\frac1T\sum_{t=0}^{T-1}\mathcal L_S(\bW_t)\mathbf 1_{\mathcal E_{\mathrm{opt}}^c}\,\Big|\,S,\bW_0,\bc\Bigr]\Bigr]\\
&\lesssim A_{\mathrm{corr}}+\log\Bigl(\frac n\delta\Bigr)(\sqrt m+R_*)\delta.
\end{align*}

Substituting this estimate into the averaged population risk bound gives
\[
\frac1T\sum_{t=0}^{T-1}\mathbb E_{S,\A}[\mathcal L(\bW_t)]
\lesssim
\Bigl(1+\frac{\eta T\log(1/\delta)}{n}\Bigr)
\Bigl(A_{\mathrm{corr}}+\log\Bigl(\frac n\delta\Bigr)(\sqrt m+R_*)\delta\Bigr).
\]
Finally, $\mathbb P(\mathcal E_{\delta/8}\cap\mathcal E_{\mathrm{cond}})\ge1-\delta/4\ge1-\delta$, so the above bound holds with probability at least $1-\delta$ over the initialization. This completes the proof.
\end{proof}

\section{Privacy Calibration under Temporal Correlation}
\label{appen:privacy-calibration}

\subsection{Example: Finite-horizon improvement under matched privacy budgets}
\label{appen:finite-horizon}

We justify the finite-horizon example stated in Section~\ref{sec:rates}. Consider the full-batch setting $B=n$ with $C_{\rm clip}=G_\delta$, any fixed finite horizon $T\ge2$, and compare correlated and independent noise using the same step size and target privacy budget $(\epsilon,\delta)$. Since $C_{\rm clip}=G_\delta$, we have $\rho_{\rm clip}=1$ and $\eta_{\rm eff}=\eta$.

For full-batch training, the differing record participates at every iteration. In the dominating-pair argument, retaining the exact finite-horizon weights gives the normalized squared sensitivity bound
\[
\frac{\Delta_\lambda^2}{C_{\rm clip}^2}
\le
\sum_{t=1}^T
\left(\frac{1-\lambda^t}{1-\lambda}\right)^2.
\]
Therefore, up to universal constants, a sufficient Gaussian calibration is
\[
\kappa_\lambda^2
\asymp
\frac{\log(1/\delta)}{\epsilon^2}
\sum_{t=1}^T
\left(\frac{1-\lambda^t}{1-\lambda}\right)^2.
\]
When $\lambda=0$, this reduces to the independent noise calibration
\[
\kappa_0^2
\asymp
\frac{T\log(1/\delta)}{\epsilon^2},
\]
and hence, equivalently,
\[
\kappa_\lambda^2
=
\kappa_0^2
\frac1T
\sum_{t=1}^T
\left(\frac{1-\lambda^t}{1-\lambda}\right)^2
\]
up to constant factors.

We now compare the privacy-dependent terms in the two population risk bounds. Relative to independent noise, the linear and quadratic privacy terms in $A_{\rm corr}$ are multiplied, respectively, by
\[
r_1(\lambda)
=
(1-\lambda)
\left(
\frac1T\sum_{t=1}^T
\left(\frac{1-\lambda^t}{1-\lambda}\right)^2
\right)^{1/2}
=
\left(
\frac1T\sum_{t=1}^T(1-\lambda^t)^2
\right)^{1/2},
\]
and
\[
r_2(\lambda)
=
\bigl((1-\lambda)^2+\eta\lambda^2\bigr)
\frac1T\sum_{t=1}^T
\left(\frac{1-\lambda^t}{1-\lambda}\right)^2.
\]
For every $\lambda\in(0,1)$, $r_1(\lambda)<1$. Moreover, $r_1(\lambda)$ is strictly decreasing in $\lambda$ and converges to $0$ as $\lambda\to1$. For the quadratic term, as $\lambda\to0^+$,
\[
r_2(\lambda)
=
1-\frac{2\lambda}{T}
+\left(\eta-\frac1T\right)\lambda^2
+\O(\lambda^3).
\]
Hence, for every fixed finite $T\ge2$, $r_2(\lambda)<1$ for all sufficiently small positive $\lambda$. Therefore, there exists $\lambda_T\in(0,1)$ such that
\[
r_1(\lambda_T)<1,
\qquad
r_2(\lambda_T)<1.
\]

Consequently, under the common admissibility conditions of Theorem~\ref{thm:lambda1-gen}, the population risk bound for correlated noise satisfies
\[
\frac1T\sum_{t=0}^{T-1}\E_{S,\A}\bigl[\mathcal L(\bW_t)\bigr]
\lesssim
\Bigl(1+\frac{\eta T}{n}\Bigr)
\Biggl(
A_{\rm com}
+\frac{r_1(\lambda_T)}{n\epsilon}
\left(
\frac1\gamma+\eta\sqrt{\frac{md}{n}}
\right)
+\frac{r_2(\lambda_T)\eta Tmd}{n^2\epsilon^2}
\Biggr),
\]
where $A_{\rm com}$ collects the terms common to the correlated and independent noise mechanisms. The corresponding independent noise bound has the same form with $r_1=r_2=1$. Thus, the linear privacy-dependent term is strictly reduced for every positive correlation level, while for every fixed finite horizon one can choose a positive correlation level that also reduces the quadratic term. Consequently, temporal correlation yields a strictly smaller population risk upper bound under the same privacy budget in this setting.

\subsection{Example: Representative rate under conservative worst-case privacy calibration}
\label{appen:subsampling}

We consider a second representative setting based on a conservative worst-case privacy calibration for mini-batch training. To obtain an explicit $(\epsilon,\delta)$-DP population risk rate, we calibrate the noise multiplier $\kappa$ using the following privacy guarantee for Algorithm~\ref{alg:dp-lambda-minibatch}. This calibration is sufficient to establish a vanishing population risk rate, although sharper privacy accounting may require a smaller noise multiplier.

\begin{theorem}[Privacy guarantee]\label{thm:DPlambda1}
Let $\epsilon,\delta\in(0,1]$, let $\lambda\in(0,1)$, and define $r=B/n$. Assume $rT\ge 3\log(2/\delta)$. 
If
\[
\kappa^2
\ge
8
\Big(\frac{1-\lambda^T}{1-\lambda}\Big)^2
\Big(
rT+\sqrt{3rT\log\big(\frac{2}{\delta}\big)}
\Big)
\frac{\log\big(\frac{2.5}{\delta}\big)}{\epsilon^2},
\]
then Algorithm~\ref{alg:dp-lambda-minibatch} satisfies $(\epsilon,\delta)$-DP. In particular, up to universal constants, it suffices to take
\[
\kappa^2
\asymp
\Big(\frac{1-\lambda^T}{1-\lambda}\Big)^2
\Big(
\frac{B}{n}T+
\sqrt{\frac{B}{n}T\log\big(\frac{1}{\delta}\big)}
\Big)
\log\big(\frac{1}{\delta}\big)\epsilon^{-2}.
\]
\end{theorem}

Combining Theorems~\ref{thm:lambda1-gen} and~\ref{thm:DPlambda1}, consider a representative regime with $\eta\asymp 1$, $C_{\rm clip}\asymp G_\delta$, and $\lambda\in(0,1)$ fixed and bounded away from $1$. Let $0<\rho<1$ be fixed, take $m\asymp \gamma^{-6}\polylog(n/\delta)$, 
$\delta\le
\min\big\{
\frac{1}{n\sqrt m},\,
\frac{\gamma}{n}
\big\}$, 
and set $B=\rho n $ and $
T\asymp
\min\big\{
\frac{\sqrt n}{\gamma},\,
\frac{n\epsilon\gamma^2}{\sqrt d}
\big\}.$
Then both the averaged empirical and population risks are of order $\O\big(
\frac{1}{\gamma\sqrt n}
+
\frac{\sqrt d}{\gamma^4n\epsilon}
\big)$ 
up to logarithmic factors. Under this conservative worst-case privacy calibration, the reduction in the effective perturbation due to temporal correlation is offset by the corresponding increase in the required noise multiplier, so the resulting rate has the same leading dependence on $n$, $d$, and $\epsilon$ as independent noise DP-SGD. The particular choices of $B$, $T$, and $C_{\rm clip}$ are made to obtain a transparent rate and are not claimed to be optimal. This should be viewed as a worst-case guarantee rather than as evidence that temporal correlation cannot improve the privacy--utility tradeoff. Sharper privacy accounting or more favorable temporal structure may preserve part of the optimization gain. A formal statement is given below.

\begin{corollary}[Representative worst-case rate]\label{cor:DP-SGD-corr}
Suppose Assumptions~\ref{ass:sigma} and~\ref{ass:ntk} hold. Let $\epsilon,\delta\in(0,1]$ and let $\{\bW_t\}_{t=0}^T$ be generated by Algorithm~\ref{alg:dp-lambda-minibatch} with $\eta\asymp 1$ and $C_{\rm clip}\asymp G_\delta$. Assume $\lambda\in(0,1)$ is a fixed constant bounded away from $1$, and let $0<\rho<1$ be a fixed constant. Suppose $m\asymp \gamma^{-6}\polylog(n/\delta)$ and 
$\delta\le
\min\big\{
\frac{1}{n\sqrt m},\,
\frac{\gamma}{n}
\big\}$. 
Set $B=\rho n$, $T\asymp
\min\big\{
\frac{\sqrt n}{\gamma},\,
\frac{n\epsilon\gamma^2}{\sqrt d}
\big\}$, 
assume $\rho T\ge 3\log(2/\delta)$, and choose $\kappa$ according to Theorem~\ref{thm:DPlambda1}. Then, with probability at least $1-\delta$ over the initialization and the algorithmic randomness,
\[
\frac1T\sum_{t=0}^{T-1}\mathcal L_S(\bW_t)
\lesssim
\polylog\left(\frac{n}{\delta}\right)
\Big(
\frac{1}{\gamma\sqrt n}
+
\frac{\sqrt d}{\gamma^4 n\epsilon}
\Big).
\]
For the population risk statement, assume further the sample-separability condition of Theorem~\ref{thm:lambda1-gen}. With probability at least $1-\delta$ over the initialization,
\[
\frac1T\sum_{t=0}^{T-1}
\mathbb E_{S,\A}\bigl[\mathcal L(\bW_t)\bigr]
\lesssim
\polylog\left(\frac{n}{\delta}\right)
\Big(
\frac{1}{\gamma\sqrt n}
+
\frac{\sqrt d}{\gamma^4 n\epsilon}
\Big).
\]
\end{corollary}

\paragraph{Proof of Theorem~\ref{thm:DPlambda1}.}
We proceed in two steps. First, we conduct a general privacy analysis of Algorithm~\ref{alg:dp-lambda-minibatch} in terms of \emph{dominating pairs}. We then use this analysis to derive an analytical upper bound on the noise multiplier $\kappa$.

\subsubsection{Step 1: Dominating pairs for arbitrary correlation matrices}

It suffices to analyze the sequence of noisy clipped-gradient releases, since the iterates generated by Algorithm~\ref{alg:dp-lambda-minibatch} are obtained from these releases by post-processing. We normalize the per-example clipped gradients by $C_{\rm clip}$ and multiply each mini-batch average by $B/C_{\rm clip}$. Under this normalization, each per-example contribution has Euclidean norm at most one and the Gaussian noise has standard deviation $\kappa$.

The normalized mechanism can be written in the matrix form
\[
\mathbf X+\mathbf C^{-1}\mathbf Z,
\qquad
\mathbf X=\sum_{i=1}^n\mathbf G^{(i)},
\]
where $\mathbf G^{(i)}\in\mathbb R^{T\times mdp}$ denotes the normalized contribution of the $i$-th record across the training iterations, and
\[
Z_{t,j}\sim\mathcal N(0,\kappa^2).
\]
The $t$-th row of $\mathbf G^{(i)}$ satisfies
\[
\|\mathbf G^{(i)}_{t,:}\|_2\le1
\]
if record $i$ is included in the $t$-th mini-batch, and $\mathbf G^{(i)}_{t,:}=0$ otherwise. Note that $\mathbf G^{(i)}_{t,:}$ may be chosen adaptively based on earlier outcomes.

The matrix $\mathbf C^{-1}\in\mathbb R^{T\times T}$ is given by
\[
C^{-1}_{t,u}
=
\begin{cases}
1, & u=t,\\
-\lambda, & u=t-1,\\
0, & \text{otherwise},
\end{cases}
\]
and its inverse is
\begin{equation}\label{eq:definition_encoder_matrix}
C_{t,u}
=
\begin{cases}
1, & u=t,\\
\lambda^{t-u}, & u<t,\\
0, & \text{otherwise}.
\end{cases}
\end{equation}
Analogously, for a neighboring dataset $S'$, the mechanism can be represented as
\[
\mathbf X'+\mathbf C^{-1}\mathbf Z,
\]
where $\mathbf X'$ differs from $\mathbf X$ only through the contribution of a single record.

To proceed with the privacy analysis, we formalize the phrase ``differ in the contribution of one data record'' used in Definition~\ref{def:DP}. DP-SGD with Poisson subsampling, where each record is independently included with a fixed rate $r\in[0,1]$, is typically analyzed under the insertion/removal relation. In contrast, Algorithm~\ref{alg:dp-lambda-minibatch} uses mini-batches of fixed size $B$, sampled uniformly without replacement. For this fixed-size subsampling scheme, we use the standard zero-out relation; see~\cite{annamalai2025hitchhiker} for an overview of papers using this relation.

\begin{definition}[Zero-out relation]
Datasets $S$ and $S'$ of equal size $n$ are neighboring under the zero-out relation if there exists a pair of records $(x_i,y_i)$ and $(x'_i,y'_i)$ such that
\[
S'=S\setminus\{(x_i,y_i)\}\cup\{(x'_i,y'_i)\},
\]
and the normalized gradient contribution of one of the two records is zero, i.e.,
\[
\mathbf G^{(i)}=0
\qquad\text{or}\qquad
\mathbf G'^{(i)}=0.
\]
\end{definition}

To simplify the analysis, we further subdivide this relation into two asymmetric parts.

\begin{definition}[Zero-out-remove relation]
Datasets $S$ and $S'$ of equal size $n$ are neighboring under the zero-out-remove relation if there exists a pair of records $(x_i,y_i)$ and $(x'_i,y'_i)$ such that
\[
S'=S\setminus\{(x_i,y_i)\}\cup\{(x'_i,y'_i)\},
\]
and the normalized gradient contribution of the substituted record in $S'$ is zero, i.e.,
\[
\mathbf G'^{(i)}=0.
\]
\end{definition}

\begin{definition}[Zero-out-add relation]
Datasets $S$ and $S'$ of equal size $n$ are neighboring under the zero-out-add relation if there exists a pair of records $(x_i,y_i)$ and $(x'_i,y'_i)$ such that
\[
S'=S\setminus\{(x_i,y_i)\}\cup\{(x'_i,y'_i)\},
\]
and the normalized gradient contribution of the substituted record in $S$ is zero, i.e.,
\[
\mathbf G^{(i)}=0.
\]
\end{definition}

In the following, we use $S\simeq_0 S'$, $S\simeq_{0,r}S'$, and $S\simeq_{0,a}S'$ for datasets neighboring under the zero-out, zero-out-remove, and zero-out-add relations, respectively.

Finally, we recall the hockey-stick divergence and dominating pairs, which provide an alternative characterization of the privacy profile $\delta(\epsilon)$~\cite{zhu2022optimal}.

\begin{definition}[Hockey-stick divergence]
The hockey-stick divergence of order $\alpha\ge0$ between two distributions $P,Q$ on a measure space $(\Omega,\mathcal F)$ is
\[
H_\alpha(P\|Q)
=
\sup_{E\in\mathcal F}
\bigl(P(E)-\alpha Q(E)\bigr).
\]
\end{definition}

\begin{definition}[Dominating pair]
Consider a randomized algorithm $\mathcal A$ and two distributions $P,Q$. If
\[
H_\alpha\bigl(
\mathcal A(\cdot\mid S)
\|
\mathcal A(\cdot\mid S')
\bigr)
\le
H_\alpha(P\|Q)
\]
for all neighboring datasets $S\simeq S'$ and all $\alpha\ge0$, then $(P,Q)$ is a dominating pair for $\mathcal A$.
\end{definition}

If $(P,Q)$ is a dominating pair for $\mathcal A$, then the corresponding hockey-stick divergence controls the privacy profile~\cite{barthe2013beyond}. In particular, if the relevant divergences in both neighboring directions are bounded by $\delta$ at $\alpha=e^\epsilon$, then $\mathcal A$ satisfies $(\epsilon,\delta)$-DP.

We first recall the following symmetry relation. The proof is identical to that of Lemma~28 in~\cite{zhu2022optimal}.

\begin{lemma}[\cite{zhu2022optimal}]
The following two statements are equivalent:
\begin{enumerate}
    \item $(P,Q)$ dominates algorithm $\mathcal A$ for zero-out-add neighbors.
    \item $(Q,P)$ dominates algorithm $\mathcal A$ for zero-out-remove neighbors.
\end{enumerate}
\end{lemma}

\begin{proof}
Let $T[P,Q]:[0,1]\to[0,1]$ be the trade-off function of distributions $P$ and $Q$. We use that $T[P,Q]$ and $T[Q,P]$ are inverse functions and that
\[
H_\alpha(P\|Q)=1+T[P,Q]^*(-\alpha),
\]
where $f^*$ denotes the convex conjugate of $f$. Consider two datasets $S\simeq_{0,a}S'$. Then
\begin{align*}
(P,Q)\text{ dominates }\mathcal A(S),\mathcal A(S')
&\iff
H_\alpha(\mathcal A(S)\|\mathcal A(S'))
\le H_\alpha(P\|Q)\\
&\iff
T[\mathcal A(S'),\mathcal A(S)]
\ge T[Q,P]\\
&\iff
T[\mathcal A(S),\mathcal A(S')]
\ge T[P,Q]\\
&\iff
H_\alpha(\mathcal A(S')\|\mathcal A(S))
\le H_\alpha(Q\|P),
\end{align*}
which proves the claim.
\end{proof}

We further use Lemma~4.5 of~\cite{choquette2023privacy}, restated as Lemma~3.3 in~\cite{choquette2024near}.

\begin{lemma}[\cite{choquette2024near}]
\label{lemma:mog_dim_reduction}
Let $\mathbf c_1,\ldots,\mathbf c_k\in\mathbb R^{n\times p}$ and let $\mathbf c'_1,\ldots,\mathbf c'_k\in\mathbb R^n$ satisfy
\[
\|\mathbf c_i[j,:]\|_2\le \mathbf c_i'(j)
\qquad
\text{for all }i,j.
\]
Let
\[
P=\mathcal N(0,\sigma^2 I),
\qquad
Q=\sum_i p_i\mathcal N(\mathbf c_i,\sigma^2 I),
\]
and
\[
P'=\mathcal N(0,\sigma^2 I),
\qquad
Q'=\sum_i p_i\mathcal N(\mathbf c_i',\sigma^2 I).
\]
Then, for every $\alpha\ge0$,
\[
H_\alpha(P\|Q)\le H_\alpha(P'\|Q').
\]
Furthermore, the conclusion continues to hold when the $j$-th row of each $\mathbf c_i$ is chosen as a function of the first $j-1$ rows, subject to the same row-norm bound, while the vectors $\mathbf c_i'$ remain fixed.
\end{lemma}

Using these results, we now derive a dominating pair for the normalized mechanism. The argument follows the proof strategy of Lemma~3.3 in~\cite{choquette2024near}.

\begin{theorem}\label{theorem:general_dominating_pair}
Given the number of iterations $T$, batch size $B$, and dataset size $n$, define the subsampling rate
\[
r=\frac{B}{n}.
\]
Let
\[
s(\mathbf y)
=
r^{\|\mathbf y\|_0}
(1-r)^{T-\|\mathbf y\|_0},
\qquad
\mathbf y\in\{0,1\}^T.
\]
Define
\[
P
=
\sum_{\mathbf y\in\{0,1\}^T}
s(\mathbf y)
\mathcal N(\mathbf C\mathbf y,\kappa^2 I_T),
\qquad
Q=\mathcal N(0,\kappa^2 I_T).
\]
Then the normalized mechanism associated with Algorithm~\ref{alg:dp-lambda-minibatch} is dominated by $(P,Q)$ under the zero-out-remove relation and by $(Q,P)$ under the zero-out-add relation.
\end{theorem}

\begin{proof}
Consider datasets $S\simeq_{0,r}S'$ and assume without loss of generality that the $n$-th record is zeroed out in $S'$. Then
\[
\mathcal A(S)
=
\mathbf X+\mathbf C^{-1}\mathbf Z,
\qquad
\mathbf X=\sum_{i=1}^n\mathbf G^{(i)},
\]
whereas
\[
\mathcal A(S')
=
\mathbf X'+\mathbf C^{-1}\mathbf Z,
\qquad
\mathbf X'=\sum_{i=1}^{n-1}\mathbf G^{(i)}.
\]
Applying the invertible linear transformation $\mathbf C$ gives
\[
\widetilde{\mathcal A}(S)
=
\mathbf C\mathbf X+\mathbf Z,
\qquad
\widetilde{\mathcal A}(S')
=
\mathbf C\mathbf X'+\mathbf Z.
\]
Since the transformation is invertible, the original and transformed mechanisms have the same privacy profile.

Following the standard reduction used in~\cite{choquette2024near}, the contributions shared by the two neighboring datasets can be removed, so it suffices to distinguish
\[
\mathbf C\mathbf G^{(n)}+\mathbf Z
\qquad\text{from}\qquad
\mathbf Z.
\]
Given a participation vector $\mathbf y\in\{0,1\}^T$, with $y_u=1$ if the differing record participates at iteration $u$, the $t$-th row satisfies
\[
\|(\mathbf C\mathbf G^{(n)})_{t,:}\|_2
\le
\sum_{u=1}^T
C_{t,u}y_u
=
(\mathbf C\mathbf y)_t,
\]
where we used $\|\mathbf G^{(n)}_{u,:}\|_2\le y_u$ and $C_{t,u}\ge0$. Lemma~\ref{lemma:mog_dim_reduction} therefore reduces the adaptive matrix-valued Gaussian mixture to the vector-valued Gaussian mixture stated in the theorem. The result for the reverse neighboring relation follows from the preceding symmetry lemma.
\end{proof}

\subsubsection{Step 2: Analytical noise multiplier bound}

We now use the dominating pair from Theorem~\ref{theorem:general_dominating_pair} to derive an explicit sufficient noise multiplier $\kappa$ for a desired privacy budget $(\epsilon,\delta)$.

The main idea is to derive a high-probability bound on the number of iterations in which the differing record participates. This gives a corresponding high-probability bound on the Gaussian mean $\|\mathbf C\mathbf y\|_2$. We calibrate the Gaussian noise to this sensitivity while allocating part of the privacy failure probability to the complement of the high-probability event.

\begin{lemma}\label{lemma:joint_convexity_bound}
Consider an event $E\subseteq\{0,1\}^T$ with probability $S(E)$ under the participation distribution $s(\mathbf y)$. Let $P,Q$ be defined as in Theorem~\ref{theorem:general_dominating_pair}. Then, for all $\epsilon\ge0$,
\[
\max\left\{
H_{e^\epsilon}(P\|Q),
H_{e^\epsilon}(Q\|P)
\right\}
\le
S(E)
\max_{\mathbf y\in E}
H_{e^\epsilon}
\left(
\mathcal N(\mathbf C\mathbf y,\kappa^2I_T)
\|
\mathcal N(0,\kappa^2I_T)
\right)
+
S(\overline E).
\]
\end{lemma}

\begin{proof}
For the $(P,Q)$ direction, by the law of total probability,
\[
P
=
S(E)P(\cdot\mid E)
+
S(\overline E)P(\cdot\mid\overline E).
\]
Joint convexity of the hockey-stick divergence gives
\begin{align*}
H_{e^\epsilon}(P\|Q)
&\le
S(E)H_{e^\epsilon}(P(\cdot\mid E)\|Q)
+
S(\overline E)
H_{e^\epsilon}(P(\cdot\mid\overline E)\|Q)\\
&\le
S(E)
\max_{\mathbf y\in E}
H_{e^\epsilon}
\left(
\mathcal N(\mathbf C\mathbf y,\kappa^2I_T)
\|
\mathcal N(0,\kappa^2I_T)
\right)
+
S(\overline E),
\end{align*}
where we used that the hockey-stick divergence is at most one and that joint convexity implies quasi-convexity over the mixture on $E$.

For the reverse direction, write
\[
Q=\sum_{\mathbf y}s(\mathbf y)Q
\qquad\text{and}\qquad
P=\sum_{\mathbf y}s(\mathbf y)P_{\mathbf y},
\]
where
\[
P_{\mathbf y}=\mathcal N(\mathbf C\mathbf y,\kappa^2I_T).
\]
Joint convexity gives the same bound with
$H_{e^\epsilon}(Q\|P_{\mathbf y})$. For Gaussians with a common covariance, the two directions have the same hockey-stick divergence by translation and reflection symmetry. Hence the stated bound follows.
\end{proof}

Next, we control the number of participations of the differing record.

\begin{lemma}\label{lemma:participation_tail_bound}
Given $T$ training steps, dataset size $n$, and batch size $B$, define
\[
r=\frac{B}{n}.
\]
Choose any constant $c_1\in(0,1)$ satisfying
\[
rT\ge
3\log\left(\frac{1}{c_1\delta}\right).
\]
Let $\mathbf y$ denote the participation vector of a fixed record over the $T$ independently sampled mini-batches. Define
\[
\tau
=
rT+
\sqrt{
3rT\log\left(\frac{1}{c_1\delta}\right)
}
\]
and
\[
E
=
\left\{
\mathbf y\in\{0,1\}^T:
\|\mathbf y\|_0\le\tau
\right\}.
\]
Then
\[
S(E)\ge1-c_1\delta,
\qquad
S(\overline E)\le c_1\delta.
\]
\end{lemma}

\begin{proof}
For a fixed record, each iteration includes that record with probability $r=B/n$. Since the mini-batches are sampled independently across iterations, the indicators $y_t$ are independent Bernoulli$(r)$ random variables. Hence
\[
\|\mathbf y\|_0
\sim
\operatorname{Binomial}(T,r).
\]
The claim follows directly from the multiplicative Chernoff bound.
\end{proof}

We next bound the encoded sensitivity.

\begin{lemma}\label{lemma:max_sensitivity}
Let $\mathbf C$ be the $T\times T$ lower-triangular matrix with
\[
C_{i,j}=\lambda^{i-j},
\qquad
i\ge j.
\]
For any $\mathbf y\in\{0,1\}^T$ satisfying $\|\mathbf y\|_0\le\tau$,
\[
\|\mathbf C\mathbf y\|_2
\le
\left(
\frac{1-\lambda^T}{1-\lambda}
\right)
\sqrt{\tau}.
\]
\end{lemma}

\begin{proof}
By consistency of the induced matrix $2$-norm,
\[
\|\mathbf C\mathbf y\|_2
\le
\|\mathbf C\|_2\|\mathbf y\|_2.
\]
Since $\mathbf y$ is binary,
\[
\|\mathbf y\|_2
=
\sqrt{\|\mathbf y\|_0}
\le
\sqrt{\tau}.
\]
Moreover,
\[
\|\mathbf C\|_2
\le
\sqrt{
\|\mathbf C\|_1
\|\mathbf C\|_\infty
}
\le
\frac{1-\lambda^T}{1-\lambda}.
\]
Combining the two inequalities yields the result.
\end{proof}

We can now complete the proof of Theorem~\ref{thm:DPlambda1}. Set $c_1=1/2$ and define
\[
\tau
=
rT+
\sqrt{
3rT\log\left(\frac{2}{\delta}\right)
}.
\]
By Lemma~\ref{lemma:participation_tail_bound},
\[
S(\overline E)\le\delta/2.
\]
On the event $E$, Lemma~\ref{lemma:max_sensitivity} gives
\[
\Delta_2
:=
\max_{\mathbf y\in E}
\|\mathbf C\mathbf y\|_2
\le
\left(
\frac{1-\lambda^T}{1-\lambda}
\right)
\sqrt{\tau}.
\]
The Gaussian mechanism on $E$ is $(\epsilon,\delta/2)$-DP whenever
\[
\kappa^2
\ge
\Delta_2^2
\frac{
2\log(2.5/\delta)
}{
\epsilon^2
}.
\]
Therefore, the stronger sufficient condition
\[
\kappa^2
\ge
8
\left(
\frac{1-\lambda^T}{1-\lambda}
\right)^2
\left(
rT+
\sqrt{
3rT\log\left(\frac{2}{\delta}\right)
}
\right)
\frac{
\log\left(\frac{2.5}{\delta}\right)
}{
\epsilon^2
}
\]
implies that
\[
\max_{\mathbf y\in E}
H_{e^\epsilon}
\left(
\mathcal N(\mathbf C\mathbf y,\kappa^2I_T)
\|
\mathcal N(0,\kappa^2I_T)
\right)
\le
\delta/2.
\]
Lemma~\ref{lemma:joint_convexity_bound} then yields
\[
\max\left\{
H_{e^\epsilon}(P\|Q),
H_{e^\epsilon}(Q\|P)
\right\}
\le
S(E)\frac{\delta}{2}
+
S(\overline E)
\le
\delta.
\]
Hence Algorithm~\ref{alg:dp-lambda-minibatch} satisfies $(\epsilon,\delta)$-DP, completing the proof of Theorem~\ref{thm:DPlambda1}.

\begin{proof}[Proof of Corollary~\ref{cor:DP-SGD-corr}]
Since $C_{\rm clip}\asymp G_\delta$, we have $\rho_{\rm clip}=\Theta(1)$, $\eta_{\rm eff}\asymp\eta$, and $c_{\rm eff}\asymp c_{\rm priv}$. Hence the additional factors arising from active clipping are absorbed into absolute constants throughout the proof.

Since $\lambda\in(0,1)$ is fixed and bounded away from $1$, we have $\frac{1-\lambda^T}{1-\lambda}=\O(1)$ and $(1-\lambda)^2+\lambda^2\eta=\O(1)$. Moreover, $B=\rho n$, so $r=\rho$. Since $rT\ge3\log(2/\delta)$, the square-root term in Theorem~\ref{thm:DPlambda1} is no larger than a constant multiple of $rT$. Therefore, choosing $\kappa$ according to Theorem~\ref{thm:DPlambda1} gives
$$
\kappa^2\lesssim\frac{T\log(1/\delta)}{\epsilon^2}.
$$
In the remainder of the proof, logarithmic factors are absorbed into $\polylog(n/\delta)$.

We first verify the conditions of Theorem~\ref{thm:lambda-positive-ntk-opt}. The lower bound on $m$ is satisfied because $m\asymp\gamma^{-6}\polylog(n/\delta)$ is larger than $\gamma^{-4}$ up to logarithmic factors. For the upper bound on $m$, using $B=\rho n$, the bound on $\kappa$, and $(1-\lambda)^2+\lambda^2\eta=\O(1)$, the right-hand side of the width upper bound is, up to logarithmic factors,
$$
\frac{B^2}{\eta^2\kappa^2\gamma^2d}\cdot\frac1T
\gtrsim
\frac{n^2\epsilon^2}{\eta^2\gamma^2dT^2}.
$$
Since $T\lesssim n\epsilon\gamma^2/\sqrt d$ and $m\asymp\gamma^{-6}\polylog(n/\delta)$, the width upper bound is satisfied after adjusting constants and logarithmic factors.

We next verify the step-size condition in Theorem~\ref{thm:lambda-positive-ntk-opt}. Since $\rho_{\rm clip}=\Theta(1)$ and $\eta\asymp1$ is chosen as a sufficiently small absolute constant, we have $\eta_{\rm eff}\asymp\eta$. Moreover, since $B=\rho n$ and $T\lesssim\sqrt n/\gamma$, we have
$$
\eta\gamma\sqrt T\frac1{\sqrt B}\lesssim1,
\qquad
\eta^2\gamma^2\left(\frac TB+1\right)\lesssim1.
$$
Using $\kappa\lesssim\sqrt{T}\,\polylog^{1/2}(n/\delta)/\epsilon$, we further obtain, up to logarithmic factors,
$$
\eta\gamma\sqrt T\,\frac{\kappa(1-\lambda)}{B}
\lesssim
\eta\gamma\frac{T}{n\epsilon}
\lesssim1,
$$
where the last inequality follows from $T\lesssim n\epsilon\gamma^2/\sqrt d$, together with $\gamma\le1$ and $d\ge1$. Hence the step-size condition holds.

We now bound $A_{\rm corr}$. Since $\rho_{\rm clip}=\Theta(1)$, $\eta_{\rm eff}\asymp\eta$, and $c_{\rm eff}\asymp c_{\rm priv}$, the active-clipping version of Theorem~\ref{thm:lambda-positive-ntk-opt} yields, up to logarithmic factors,
$$
A_{\rm corr}
\lesssim
\frac{1}{\gamma^2\eta T}
+\frac{1}{\gamma\sqrt{nT}}
+\eta\left(\frac1n+\frac1T\right)
+\frac{1}{n\epsilon}\left(\frac1\gamma+\frac{\eta\sqrt{md}}{\sqrt n}\right)
+\frac{\eta Tmd}{n^2\epsilon^2}.
$$
By the choice
$$
T\asymp\min\left\{\frac{\sqrt n}{\gamma},\,\frac{n\epsilon\gamma^2}{\sqrt d}\right\},
$$
the first term satisfies
$$
\frac{1}{\gamma^2\eta T}
\lesssim
\frac{1}{\gamma\sqrt n}
+\frac{\sqrt d}{\gamma^4n\epsilon}.
$$
The second and third terms are of no larger order. For the linear-noise term,
$$
\frac{1}{n\epsilon}\left(\frac1\gamma+\frac{\eta\sqrt{md}}{\sqrt n}\right)
\lesssim
\polylog\left(\frac n\delta\right)\frac{\sqrt d}{\gamma^4n\epsilon},
$$
where we used $d\ge1$, $\gamma\le1$, $\eta\asymp1$, and $m\asymp\gamma^{-6}\polylog(n/\delta)$. For the quadratic-noise term, using $T\lesssim n\epsilon\gamma^2/\sqrt d$, we obtain
$$
\frac{\eta Tmd}{n^2\epsilon^2}
\lesssim
\frac{\eta m\gamma^2\sqrt d}{n\epsilon}
\lesssim
\polylog\left(\frac n\delta\right)\frac{\sqrt d}{\gamma^4n\epsilon}.
$$
Hence
$$
A_{\rm corr}
\lesssim
\polylog\left(\frac n\delta\right)
\left(
\frac{1}{\gamma\sqrt n}
+\frac{\sqrt d}{\gamma^4n\epsilon}
\right).
$$
This proves the optimization bound.

It remains to verify the additional condition in Theorem~\ref{thm:lambda1-gen}. Since $R_*\asymp\gamma^{-1}\polylog(n/\delta)$, the bound above and the choice of $T$ give $TA_{\rm corr}\lesssim\gamma^{-2}\polylog(n/\delta)$. The additional mini-batch term in the width condition of Theorem~\ref{thm:lambda1-gen} is no larger than $TA_{\rm corr}$ under the present choices of $B$ and $T$. Therefore, the right-hand side of the required width condition is bounded by
$$
R_*^2\eta^2(TA_{\rm corr})^2
\lesssim
\gamma^{-6}\polylog(n/\delta).
$$
Hence the condition is satisfied by $m\asymp\gamma^{-6}\polylog(n/\delta)$.

Applying Theorem~\ref{thm:lambda1-gen}, we obtain
$$
\frac1T\sum_{t=0}^{T-1}\mathbb E_{S,\A}[\mathcal L(\bW_t)]
\lesssim
\left(1+\frac{\eta T\log(1/\delta)}{n}\right)
\left(
A_{\rm corr}
+\log\left(\frac n\delta\right)(\sqrt m+R_*)\delta
\right).
$$
Since $T\lesssim\sqrt n/\gamma$, the factor $1+\eta T\log(1/\delta)/n$ is absorbed into the logarithmic factor. Moreover, $\delta\le\min\{1/(n\sqrt m),\,\gamma/n\}$ implies
$$
\log\left(\frac n\delta\right)(\sqrt m+R_*)\delta
\lesssim
\polylog\left(\frac n\delta\right)\frac1n,
$$
which is dominated by the displayed rate. Therefore,
$$
\frac1T\sum_{t=0}^{T-1}\mathbb E_{S,\A}\bigl[\mathcal L(\bW_t)\bigr]
\lesssim
\polylog\left(\frac{n}{\delta}\right)
\left(
\frac{1}{\gamma\sqrt n}
+\frac{\sqrt d}{\gamma^4n\epsilon}
\right).
$$
This completes the proof.
\end{proof}

\section{Proofs for DP-SGD with Independent Noise}\label{appen:lambda0}
\subsection{Privacy guarantee of  DP-SGD with independent noise}\label{appen:privacy0}
In this subsection, we prove the privacy guarantee of Algorithm~\ref{alg:dp-lambda-minibatch} in the case $\lambda=0$, i.e., standard DP-SGD with independent Gaussian noise. The  proof  is standard.
Nevertheless, for consistency, we include the details.

Throughout the proof, we use Rényi differential privacy (RDP) that allows for a more refined analysis of privacy loss.
In this subsection, for the independent-noise baseline, we use the standard replacement neighboring relation: two datasets \(S\) and \(S'\) of the same size are neighboring if they differ in one data point.
\begin{definition}[Rényi differential privacy {\cite{mironov2017renyi}}]
For $\alpha > 1$ and $\rho > 0$, a randomized algorithm $\mathcal A$ is said to satisfy $(\alpha,\rho)$-RDP if, for every pair of neighboring datasets $S,S'$,
\[
D_\alpha(\mathcal A(S)\,\|\,\mathcal A(S'))
:=
\frac{1}{\alpha-1}\log
\mathbb E_{\theta\sim \mathcal A(S')}
\!\left[
\left(
\frac{\mathcal A(S)(\theta)}{\mathcal A(S')(\theta)}
\right)^\alpha
\right]
\le \rho .
\]
\end{definition}
 
A connection $(\epsilon,\delta)$-DP and RDP is established in the following lemma. 
\begin{lemma}[RDP to $(\epsilon,\delta)$-DP {\cite{mironov2017renyi}}]
\label{lem:rdp-to-dp-lambda0}
If $\mathcal A$ satisfies $(\alpha,\rho)$-RDP for some $\alpha>1$, then for every $\delta\in(0,1)$, $\mathcal A$ also satisfies
\[
\Big(\rho+\frac{\log(1/\delta)}{\alpha-1},\,\delta\Big)\text{-DP}.
\]
\end{lemma}
To achieve DP, we need the concept of $\ell_2$-sensitivity defined as follows. 
\begin{definition}[$\ell_2$-sensitivity]\label{def:sensitivity}
The $\ell_2$-sensitivity of a function (mechanism) $\mathcal{M}:\mathcal{Z}^n \rightarrow \mathcal{W}$ is defined as 
$
\Delta  = \sup_{S, S'} \|\mathcal{M}(S) - \mathcal{M}(S')\|_2,
$ where $S$ and $S'$ are neighboring datasets.
\end{definition}
A basic mechanism to obtain RDP is Gaussian mechanism. 
\begin{lemma}[Gaussian mechanism \cite{mironov2017renyi}]\label{lem:gaussian-rdp} Consider a function $\mathcal{M}:  \mathcal{Z}^n\rightarrow \mathcal{R}^d$ with the $\ell_2$-sensitivity parameter $\Delta$,  and a dataset $S\subset\mathcal{Z}^n$.  
{The Gaussian mechanism $\mathcal{G}(S,\sigma)=\mathcal{M}(S)+\mathbf{b}$, where $\mathbf{b}\sim \mathcal{N}(0,\sigma^2\mathbf{I}_d)$, } satisfies $(\lambda,\frac{\lambda \Delta^2}{2\sigma^2})$-RDP. 
\end{lemma}
The following post-processing property enables flexible use of private data outputs while preserving rigorous privacy guarantees.
\begin{lemma}[Post-processing {\cite{mironov2017renyi}}]
\label{lem:post-processing-lambda0}
If $\mathcal A$ satisfies $(\alpha,\rho)$-RDP and $f$ is any deterministic function, then $f\circ \mathcal A$ also satisfies $(\alpha,\rho)$-RDP.
\end{lemma}
The following RDP composition theorem characterizes the privacy of a composition of parallel or adaptive mechanisms in terms of the privacy guarantees of the individual mechanisms.
 \begin{lemma}[Composition of RDP {\cite{mironov2017renyi}}]\label{lem:composition_RDP}
Fix an order $\alpha>1$. For each $i\in[k]$, let $\A_i:\mathcal Z^n\to\mathcal W_i$
be a randomized mechanism satisfying $(\alpha,\rho_i)$-RDP.
Then the following statements hold.
\begin{enumerate}[label=({\alph*}),leftmargin=*]
\item \textit{Joint (simultaneous) release.} 
Let $\A(S)=(\A_1(S),\ldots,\A_k(S))$.
Suppose $\{\A_i\}_{i=1}^k$ are independent.
Then $\A$ satisfies
$(\alpha,\sum_{i=1}^k\rho_i)$-RDP.

\item \textit{Adaptive composition.}
Suppose $\A_1,\ldots,\A_k$ are applied sequentially, and for each $i\in[k]$,
$\A_i$ may depend on the previous outputs
$\A_1(S),\ldots,\A_{i-1}(S)$.
If for every fixed realization $w_{<i}:= (\A_1(S),\ldots,\A_{i-1}(S))$ of previous outputs,
the conditional mechanism $\A_i(\cdot\,;w_{<i})$ satisfies $(\alpha,\rho_i)$-RDP,
then the overall mechanism
\[
\A(S)=\big(\A_1(S),\A_2(S;\A_1(S)),\ldots,\A_k(S;\A_1(S),\ldots,\A_{k-1}(S))\big)
\]
satisfies $(\alpha,\sum_{i=1}^k\rho_i)$-RDP.
\end{enumerate}
\end{lemma}

\begin{lemma}[One-step privacy of the uniformly subsampled Gaussian mechanism]
\label{lem:one-step-sgm-lambda0}
Fix an iteration $t$ and condition on the past iterate $\bW_{t-1}$. 
Let $\widetilde g_{t,i} := g_{t,i} \cdot \min\big\{ 1,\frac{C_{\mathrm{clip}}}{\|g_{t,i}\|_2} \big\}.$  
Consider the mechanism
\[
\mathcal M_t(S)
=
\frac1B\sum_{i\in \mathcal B_t}\widetilde g_{t,i}
+
\frac{C_{\mathrm{clip}} }{B}\kappa Z_t,
\qquad
Z_t\sim \mathcal N(\mathbf 0,\mathbf I),
\]
where $\mathcal B_t$ is sampled uniformly without replacement from $[n]$ with $|\mathcal B_t|=B$.
Then, in the standard low-sampling regime, there exists an absolute constant $c_{\rm sgm}>0$ such that $M_t$ satisfies $\big(\alpha,c_{\rm sgm}\frac{\alpha B^2}{n^2\kappa^2}\big)$-RDP. 
\end{lemma}
\begin{proof}
To ensure the consistency of the paper, we provide proof here. 
Condition on the past iterate $\bW_{t-1}$.
Then the clipped gradients $\{\widetilde g_{t,i}\}_{i=1}^n$ are deterministic vectors and satisfy $\|\widetilde g_{t,i}\|_2\le C_{\mathrm{clip}}$ for all  $i\in[n]$.
Define $U=(u_1,\dots,u_B)$ by
\[
M(U)
  =
\frac1B\sum_{i=1}^B u_i
+
\frac{C_{\mathrm{clip}}}{B}\kappa Z,
\qquad
Z\sim\mathcal N(\mathbf 0,\mathbf I).
\]
If $U$ and $U'$ differ in exactly one entry, then
\[
\Big\|
\frac1B\sum_{i=1}^B u_i
-
\frac1B\sum_{i=1}^B u_i'
\Big\|_2
\le
\frac{2C_{\mathrm{clip}}}{B}.
\] 

Therefore, $M$ is a Gaussian mechanism with $\ell_2$-sensitivity
$\Delta = \frac{2C_{\mathrm{clip}}}{B}$ and noise standard deviation
$\sigma = \frac{C_{\mathrm{clip}}}{B}\kappa$.
By Lemma~\ref{lem:gaussian-rdp}, we know $M$ satisfies $(\alpha,\varepsilon_M(\alpha))$-RDP  with 
$\varepsilon_M(\alpha)
=
\frac{\alpha\Delta^2}{2\sigma^2}
=
\frac{2\alpha}{\kappa^2}.$
Note that $\M_t$ is exactly the mechanism  $\M_t = M\circ \mathrm{subsample}$ with subsampling ratio $q=\frac{B}{n}$.

Applying Theorem~9 of \cite{wang2019subsampled} to $\M_t$, for every integer $\alpha\ge 2$,
\begin{align}
\varepsilon_t(\alpha)
\le
& \, \frac{1}{\alpha-1}
\log \Big(
1
+
q^2\binom{\alpha}{2}
\min \Bigl\{
4\bigl(e^{\varepsilon_M(2)}-1\bigr),\;
e^{\varepsilon_M(2)}\min\{2,(e^{\varepsilon_M(\infty)}-1)^2\}
\Bigr\}\nonumber\\& 
+
2\sum_{j=3}^{\alpha}
q^j\binom{\alpha}{j}
e^{(j-1)\varepsilon_M(j)}
\min\{2,(e^{\varepsilon_M(\infty)}-1)^j\}
\Big).
\end{align}
Since  $\varepsilon_M(\infty)=\infty$ and $\varepsilon_M(j)=\frac{2j}{\kappa^2}$,
then 
\[
\varepsilon_t(\alpha)
\le
\frac{1}{\alpha-1}
\log\Big(
1
+
q^2\binom{\alpha}{2}
\min\!\Bigl\{
4\bigl(e^{4/\kappa^2}-1\bigr),\;
2e^{4/\kappa^2}
\Bigr\}
+
2\sum_{j=3}^{\alpha}
q^j\binom{\alpha}{j}
e^{2j(j-1)/\kappa^2}
\Big).
\]

In the standard low-sampling/high-noise regime, the $j=2$ term dominates the above expression.
Accordingly, as noted in \cite{wang2019subsampled}, the RDP of the subsampled Gaussian mechanism simplifies to
\[
\varepsilon_t(\alpha)
=
\O \Big(\frac{\alpha q^2}{\kappa^2} \Big)
=
\O \Big(\frac{\alpha B^2}{n^2\kappa^2} \Big).
\]
Hence there exists an absolute constant $c_{\mathrm{sgm}}>0$ such that
$\varepsilon_t(\alpha)
\le
c_{\mathrm{sgm}}\frac{\alpha B^2}{n^2\kappa^2}.$
The proof is completed. 
\end{proof}

Based on the above lemmas, we now prove Theorem~\ref{thm:DPlambda0}. We note the argument below is intended for the subsampled mini-batch regime \(B<n\), where privacy amplification by subsampling is relevant. In the full-batch case \(B=n\), amplification is not needed, and the privacy guarantee follows directly from the standard Gaussian mechanism (equivalently, from the full-batch DP-GD result in \cite{wang2026optimization}).

\begin{theorem}[Privacy guarantee for DP-SGD with independent noise]\label{thm:DPlambda0} 
Assume $\delta>0$, $\lambda=0$, and that the subsampled Gaussian mechanism operates in the regime of Lemma~\ref{lem:one-step-sgm-lambda0}.
If
\[
\kappa^2 \ge  
\frac{2c_{\rm sgm}B^2T}{n^2\epsilon}
+
\frac{4c_{\rm sgm}B^2T\log(1/\delta)}{n^2\epsilon^2},
\]
then Algorithm~\ref{alg:dp-lambda-minibatch} satisfies $(\epsilon,\delta)$-DP.
For simplicity, we set $\kappa^2 =  
\frac{2c_{\rm sgm}B^2T}{n^2\epsilon}
+
\frac{4c_{\rm sgm}B^2T\log(1/\delta)}{n^2\epsilon^2}.$
\end{theorem}
\begin{proof}
At iteration $t$, the released noisy mini-batch gradient is
\[
\hat v_t
=
\frac1B\sum_{i\in\mathcal B_t}\widetilde g_{t,i}
+
\frac{C_{\mathrm{clip}}}{B}\kappa Z_t,
\qquad
Z_t\sim\mathcal N(\mathbf 0,\mathbf I_{mdp}),
\]
since $\lambda=0$ implies $\xi_t=\kappa Z_t$.

Condition on the past iterate $\bW_{t-1}$.
Then all clipped gradients $\{\widetilde g_{t,i}\}_{i=1}^n$ are deterministic vectors satisfying $\|\widetilde g_{t,i}\|_2\le C_{\mathrm{clip}}$.
Therefore, by Lemma~\ref{lem:one-step-sgm-lambda0}, the conditional mechanism that outputs $\hat v_t$ satisfies
\[
\left(\alpha,\frac{c_{\rm sgm} \alpha B^2}{n^2\kappa^2}\right)\text{-RDP}.
\]
Note the update $\bW_t
=
\Pi_{\mathcal K}\!\bigl(\bW_{t-1}-\eta \hat v_t\bigr)$ 
is a deterministic function of $\widehat G_t$ and $\bW_{t-1}$.
Hence, by Lemma~\ref{lem:post-processing-lambda0}, the conditional mechanism that outputs $\bW_t$ also satisfies $\bigl(\alpha,\; \frac{c_{\rm sgm}\alpha  B^2}{n^2\kappa^2} \bigr)$-RDP.

Applying Lemma~\ref{lem:composition_RDP} over $t=0,\dots,T-1$, the full transcript $(\bW_0,\dots,\bW_{T-1})$ satisfies
$\bigl(\alpha,\; \frac{c_{\rm sgm}\alpha TB^2}{n^2\kappa^2}\bigr)$-RDP.
By further using post-processing property (Lemma~\ref{lem:post-processing-lambda0}), the final iterate $\bW_{T-1}$ also satisfies
$\bigl(\alpha,\; \frac{c_{\rm sgm}\alpha TB^2}{n^2\kappa^2}\bigr)$-RDP.

Choosing $\alpha = 1+\frac{2\log(1/\delta)}{\epsilon}$. 
Then Lemma~\ref{lem:rdp-to-dp-lambda0} gives that $\bW_{T-1}$ satisfies
$\big( c_{\mathrm{sgm}}\frac{\alpha T B^2}{n^2 \kappa^2}
+\frac{\epsilon}{2}, \, \delta \big)$-DP.
Note that $\kappa^2 \ge \frac{2c_{\mathrm{sgm}}\alpha T B^2}{n^2 \epsilon}$ ensures 
\[
c_{\mathrm{sgm}}\frac{\alpha T B^2}{n^2\kappa^2}\le \frac{\epsilon}{2} .
\]
Hence,  $\bW_{T-1}$ satisfies $(\epsilon,\delta)$-DP. 
Substituting $\alpha=1+\frac{2\log(1/\delta)}{\epsilon}$ yields
\[ \kappa^2 \ge \frac{2c_{\mathrm{sgm}}B^2T}{n^2\epsilon}
+ \frac{4c_{\mathrm{sgm}}B^2T\log(1/\delta)}{n^2\epsilon^2}, \]
which completes the proof.
\end{proof}

\subsection{Proofs for optimization of  DP-SGD with independent noise}\label{appen:lambda0-opt}

Recall
\[
\mathcal E_\delta
 =
\Big\{
\|\bc\|_2 \le 4\sqrt{pm} + 2\sqrt{\log(2/\delta)},
\quad
\max_{j\in[m]}\|\bc_j\|_2 \le 4\sqrt p + 2\sqrt{\log(2m/\delta)}
\Big\} 
\]
and  
\[
\mathcal F_{t-1}
 =
\sigma \bigl(
\bW_0,\mathcal B_1,\dots,\mathcal B_{t-1},Z_1,\dots,Z_{t-1}
\bigr).
\]
By Lemma~\ref{lem:bound-c}, it holds that $\mathbb P(\mathcal E_\delta)\ge 1-\delta.$
For any \(t\in[T]\), define
\begin{align}\label{eq:Delta}
    \Delta_t
 =
\frac1B\sum_{i\in\mathcal B_t} g_{t,i}
-
\nabla\mathcal L_S(\bW_{t-1}) 
\quad \text{with} \quad 
g_{t,i} =\nabla \ell \left(y_i f_{\bW_{t-1}}(\bx_i)\right).
\end{align}
\begin{lemma}[Unbiasedness and variance when \(\lambda=0\)]
\label{lem:lambda0-reduction}
Suppose that \(\lambda=0\) and \(C_{\mathrm{clip}}\ge G_\delta\), where \(G_\delta\) is defined in \eqref{eq:Gdelta}. Let \(\{\bW_t\}_{t=0}^T\) be generated by Algorithm~\ref{alg:dp-lambda-minibatch}. 
Under the event \(\mathcal E_\delta\), for all \(t\in[T]\),
\[
\mathbb E_{\mathcal B_t}[\Delta_t\mid \mathcal F_{t-1}]=0 \ \ \text{ and } \ \ \mathbb E_{\mathcal B_t}\big[\|\Delta_t\|_2^2\mid \mathcal F_{t-1}\big] \le \frac{G_\delta^2}{B}.
\]
Here the expectation is taken with respect to the uniform sampling of the mini-batch \(\mathcal B_t\).
\end{lemma}

\begin{proof}
Assume the event \(\mathcal E_\delta\) holds. For any fixed \(t\in[T]\), the uniform gradient bound \eqref{eq:Gdelta} implies
\[
\|g_{t,i}\|_2
=
\big\|\nabla \ell\!\left(y_i f_{\bW_{t-1}}(\bx_i)\right)\big\|_2
\le G_\delta
\le C_{\mathrm{clip}},
\qquad \forall i\in[n].
\]
Hence clipping is inactive, and since \(\lambda=0\), the update becomes
\[
\bW_t
=
\Pi_{\mathcal K}\Big(
\bW_{t-1}
-
\eta\Big(
\frac1B\sum_{i\in\mathcal B_t}g_{t,i}
+
\frac{C_{\mathrm{clip}}\kappa}{B}Z_t
\Big)
\Big).
\]
Conditional on \(\mathcal F_{t-1}\), the iterate \(\bW_{t-1}\) is fixed, and thus the vectors \(g_{t,1},\dots,g_{t,n}\) are deterministic. Since \(\mathcal B_t\) is sampled uniformly without replacement from \([n]\), the mini-batch average is an unbiased estimator of the full empirical gradient:
\[
\mathbb E_{\mathcal B_t}\Big[
\frac1B\sum_{i\in\mathcal B_t}g_{t,i}
\,\Big|\,
\mathcal F_{t-1}
\Big]
=
\frac1n\sum_{i=1}^n g_{t,i}
=
\nabla \mathcal L_S(\bW_{t-1}).
\]
This proves the first claim.

Moreover, since \(\|g_{t,i}\|_2\le G_\delta\) for all \(i\), Lemma~\ref{lem:finite-pop-var} gives
\[
\mathbb E_{\mathcal B_t}\big[\|\Delta_t\|_2^2\mid \mathcal F_{t-1}\big]
\le \frac{G_\delta^2}{B},
\]
which completes the proof.
\end{proof}

\begin{lemma}
\label{lem:lambda0-pathwise-recursion}
Let \(\delta\in(0,1)\). Suppose Assumption~\ref{ass:sigma} holds, \eqref{eq:good-event} holds, \(\lambda=0\) and \(C_{\mathrm{clip}}\ge G_\delta\), where \(G_\delta\) is defined in \eqref{eq:Gdelta}. Let \(\{\bW_t\}_{t=0}^{T-1}\) be generated by Algorithm~\ref{alg:dp-lambda-minibatch} with $\eta\le \frac{1}{12C_{\sigma,b}p^3}$. Assume
\begin{equation}
\label{eq:m-cond-lambda0}
m\ge 64 C_{\sigma,b}^2 p^3\bigl(\sqrt{\log(m/\delta)}+\sqrt p\bigr)^2R_*^4.
\end{equation}
Then, for any comparator \(\bW^*\in\mathcal K\) and all \(t\in[T]\), it holds on \(\mathcal E_\delta\) that
\begin{align*}
\mathcal L_S(\bW_{t-1})
\le\;&
4\mathcal L_S(\bW^*)
+
\frac{3}{2\eta}\Big(
\|\bW_{t-1}-\bW^*\|_2^2-\|\bW_t-\bW^*\|_2^2
\Big) \\
&\quad
+6\eta\|\Delta_t\|_2^2
+
6\eta \Big(\frac{C_{\mathrm{clip}}\kappa}{B}\Big)^2\|Z_t\|_2^2
-
3\Big\langle
\Delta_t+\frac{C_{\mathrm{clip}}\kappa}{B}Z_t,\,
\bW_{t-1}-\bW^*
\Big\rangle .
\end{align*}
\end{lemma}

\begin{proof}
Assume \(\mathcal E_\delta\) holds. Let $g_t:=\nabla\mathcal L_S(\bW_{t-1}) $.
Since clipping is inactive and \(\lambda=0\), the update can be written as
\[
\bW_t
=
\Pi_{\mathcal K}\Big(
\bW_{t-1}-\eta\big(g_t+\Delta_t+c_{\rm priv}Z_t\big)
\Big) 
\quad \text{with} \quad 
c_{\rm priv} =\frac{C_{\mathrm{clip}}\kappa}{B}.
\]
Applying Lemma~\ref{lem:MD} with \(g=g_t+\Delta_t+c_{\rm priv}Z_t\), \(\bW=\bW_{t-1}\), and \(\bW^+=\bW_t\), we obtain
\[
\langle g_t+\Delta_t+c_{\rm priv}Z_t,\,\bW_{t-1}-\bW^*\rangle
\le
\frac{1}{2\eta}\Big(\|\bW_{t-1}-\bW^*\|_2^2-\|\bW_t-\bW^*\|_2^2\Big)
+
\frac{\eta}{2}\|g_t+\Delta_t+c_{\rm priv}Z_t\|_2^2.
\]
Hence,
\begin{align}
\langle g_t,\bW_{t-1}-\bW^*\rangle
\le\;&
\frac{1}{2\eta}\Big(\|\bW_{t-1}-\bW^*\|_2^2-\|\bW_t-\bW^*\|_2^2\Big)
+
\frac{\eta}{2}\|g_t+\Delta_t+c_{\rm priv}Z_t\|_2^2\nonumber\\
&  -
\langle \Delta_t+c_{\rm priv}Z_t,\bW_{t-1}-\bW^*\rangle.
\label{eq:lambda0-md-pathwise}
\end{align}
On the other hand, by Lemma~\ref{pro:smooth}, we know
\[
\lambda_{\min}\big(\nabla^2 \mathcal L_S(\bW)\big)
\ge
-\frac{C_{\sigma,b}\, p^{3/2}\bigl(\sqrt{\log(m/\delta)}+\sqrt p\bigr)}{\sqrt m}\,\mathcal L_S(\bW) .
\]
Since \(\bW_{t-1},\bW^*\in\mathcal K=\mathcal B(\bW_0,R_*)\), it holds that $\|\bW_{t-1}-\bW^*\|_2\le 2R_*.$
Combining this with the condition \eqref{eq:m-cond-lambda0} ensures that 
\[
\frac{C_{\sigma,b}\, p^{3/2}\bigl(\sqrt{\log(m/\delta)}+\sqrt p\bigr)}{\sqrt m} (2R_*)^2 \le \frac12.
\]
Therefore, by Lemma~\ref{lem:quasi-convexity} and the same argument as in the proof of the correlated case (see proof of Lemma~\ref{lem:shifted-localized-comparator}), we have
\[
\mathcal L_S(\bW_{t-1})
\le
\mathcal L_S(\bW^*)
+
\langle g_t,\bW_{t-1}-\bW^*\rangle
+
\frac13\Big(
\mathcal L_S(\bW_{t-1})+\mathcal L_S(\bW^*)
\Big).
\]
Rearranging implies
\[
\frac23\,\mathcal L_S(\bW_{t-1})
\le
\frac43\,\mathcal L_S(\bW^*)
+
\langle g_t,\bW_{t-1}-\bW^*\rangle.
\]
Substituting \eqref{eq:lambda0-md-pathwise} in to the above observation yields
\begin{align*}
\frac23\,\mathcal L_S(\bW_{t-1})
\le\;&
\frac43\,\mathcal L_S(\bW^*)
+
\frac{1}{2\eta}\Big(\|\bW_{t-1}-\bW^*\|_2^2-\|\bW_t-\bW^*\|_2^2\Big)
+
\frac{\eta}{2}\|g_t+\Delta_t+c_{\rm priv}Z_t\|_2^2\\
& -
\langle \Delta_t+c_{\rm priv}Z_t,\bW_{t-1}-\bW^*\rangle.
\end{align*}

Note that 
\[
\|g_t+\Delta_t+c_{\rm priv}Z_t\|_2^2
\!\le\!
2\|g_t\|_2^2+4\|\Delta_t\|_2^2+4c_{\rm priv}^2\|Z_t\|_2^2 \!\le\!  8C_{\sigma,b}p^3\,\mathcal L_S(\bW_{t-1})+4\|\Delta_t\|_2^2+4c_{\rm priv}^2\|Z_t\|_2^2, 
\]
where we have used Lemma~\ref{pro:smooth}. By further using  $4C_{\sigma,b}p^3\eta\le \frac13$ implied by \(\eta\le (12C_{\sigma,b}p^3)^{-1}\), it holds 
\[
\frac{\eta}{2}\|g_t+\Delta_t+c_{\rm priv}Z_t\|_2^2
\le \frac 13 \mathcal L_S(\bW_{t-1})
+
2\eta\|\Delta_t\|_2^2
+
2\eta c_{\rm priv}^2\|Z_t\|_2^2.
\]
Therefore, 
\begin{align*}
    \frac13\,\mathcal L_S(\bW_{t-1})
\le\, &
\frac43\,\mathcal L_S(\bW^*)
\!+\!
\frac{1}{2\eta}\Big(\|\bW_{t-1}-\bW^*\|_2^2-\|\bW_t-\bW^*\|_2^2\Big)
\!+\!
2\eta\|\Delta_t\|_2^2
\!+\!
2\eta c_{\rm priv}^2\|Z_t\|_2^2\\ 
& -
\langle \Delta_t+c_{\rm priv}Z_t,\bW_{t-1}-\bW^*\rangle.
\end{align*}
Multiplying both sides by \(3\) and plugging $c_{\rm priv}$ back complete the proof.
\end{proof}

We now estimate the perturbation terms in Lemma~\ref{lem:lambda0-pathwise-recursion}.
\begin{lemma}
\label{lem:lambda0-delta-square}
Suppose the assumptions of Lemma~\ref{lem:lambda0-reduction} hold. Let \(C>0\) be an absolute constant. Then, conditional on the dataset and the initialization satisfying \(\mathcal E_\delta\) (see \eqref{eq:good-event}), the following statements hold true.
\begin{enumerate}[label=({\alph*}),leftmargin=*]
\item For any \(\delta_{\Delta^2}\in(0,1)\),
\[
\mathbb P_{\A}\Big(
\sum_{t=1}^T\|\Delta_t\|_2^2
\le
\frac{2TG_\delta^2}{B}
+
8G_\delta^2\log\!\Big(\frac{1}{\delta_{\Delta^2}}\Big)
\Big)
\ge
1-\delta_{\Delta^2}.
\]
\item For any \(\delta_{Z^2}\in(0,1)\), 
\[
\mathbb P\Big(
\sum_{t=1}^T \|Z_t\|_2^2
\le
Tmdp
+
2\sqrt{Tmdp\log (1/{\delta_{Z^2}} )}
+
2\log\big(\frac{1}{\delta_{Z^2}}\big)
\Big)
\ge
1-\delta_{Z^2}.
\]
\item Let \(\bW^*\in\mathcal K\), and define $\mathfrak M_T^{(\Delta)}
:=
-3\sum_{t=1}^T \langle \Delta_t,\bW_{t-1}-\bW^*\rangle.$ 
Then, for any \(\delta_{\Delta,\mathrm{lin}}\in(0,1)\),
\[
\mathbb P_{\A}\Big(
|\mathfrak M_T^{(\Delta)}|
\le
C\,G_\delta R_*
\sqrt{\frac{T\log(1/\delta_{\Delta,\mathrm{lin}})}{B}}
\Big)
\ge
1-\delta_{\Delta,\mathrm{lin}}.
\] 
\item Let \(\bW^*\in\mathcal K\). Define $\mathfrak M_T^{(Z)}
:=
-3\frac{C_{\mathrm{clip}}\kappa}{B}
\sum_{t=1}^T
\langle Z_t,\bW_{t-1}-\bW^*\rangle.$
Then,  for any \(\delta_{Z,\mathrm{lin}}\in(0,1)\),
\[
\mathbb P_{\A}\Big(
|\mathfrak M_T^{(Z)}|
\le
C\,\frac{C_{\mathrm{clip}}\kappa}{B}\,R_*
\sqrt{T\log(1/\delta_{Z,\mathrm{lin}})}
\Big)
\ge
1-\delta_{Z,\mathrm{lin}}.
\]
\end{enumerate}
\end{lemma}

\begin{proof}
 Part (a) is proved exactly as Lemma~\ref{lem:good-event-delta-square}, but without the stopping indicator. And part (b) directly from Lemma~\ref{lem:laurent-massart-gaussian}.

For the part (c), conditional on the pre-sampling filtration, \(\bW_{t-1}-\bW^*\) is fixed and has norm at most \(2R_*\), while \(\Delta_t\) is centered. The proof is the same as that of Lemma~\ref{lem:pot-fluctuation-1}, with \(\bar R\) replaced by \(2R_*\) and without the stopping indicator.

Finally, conditional on the filtration up to time \(t-1\), the vector \(\bW_{t-1}-\bW^*\) is fixed and has norm at most \(2R_*\), while \(Z_t\sim \mathcal N(0,\mathbf{I}_{mdp})\) is independent of the past. Hence each summand is conditionally Gaussian with variance bounded by \(C (C_{\mathrm{clip}}\kappa R_*/B)^2\). The claim then follows from the standard maximal inequality for Gaussian martingale differences, exactly as in Lemma~\ref{lem:pot-fluctuation-2}.
The proof is complete. 
\end{proof}

\begin{theorem}[General optimization risk bound for standard DP-SGD]
\label{thm:lambda0-opt}
Let \(\delta\in(0,1)\). Suppose Assumption~\ref{ass:sigma} holds. Let \(\{\bW_t\}_{t=0}^T\) be generated by Algorithm~\ref{alg:dp-lambda-minibatch} with \(\lambda=0\), $\eta\le \frac{1}{12C_{\sigma,b}p^3}$ and  
$C_{\mathrm{clip}}\ge G_\delta$. 
Assume 
$m\ge 64 C_{\sigma,b}^2p^3\bigl(\sqrt{\log(m/\delta)}+\sqrt p\bigr)^2R_*^4.$ 
Then, conditioned on the dataset and the initialization satisfying \(\mathcal E_\delta\), with probability at least \(1-\delta\) over the algorithmic randomness,
\begin{align*}
\frac1T\sum_{t=0}^{T-1}\mathcal L_S(\bW_t)
\lesssim\;&
\mathcal L_S(\bW^*)
+
\frac{\|\bW_0-\bW^*\|_2^2}{\eta T}
+
G_\delta R_*
\sqrt{\frac{\log(1/\delta)}{BT}} 
+
\frac{C_{\mathrm{clip}}\kappa}{B}\,R_*
\sqrt{\frac{\log(1/\delta)}{T}}
\\
& +
\eta G_\delta^2
\Big(
\frac{1}{B}+\frac{\log(1/\delta)}{T}
\Big)
+
\eta \frac{C_{\mathrm{clip}}^2\kappa^2}{B^2}
\Big(
md +\frac{\log(1/\delta)}{T}
\Big).
\end{align*}
\end{theorem}

\begin{proof}
Fix auxiliary failure probabilities $\delta_{\Delta^2}=\delta_{\Delta,\mathrm{lin}}=\delta_{Z^2}=\delta_{Z,\mathrm{lin}}=\frac{\delta}{4}.$ 
Since replacing \(\delta\) by a constant fraction only affects logarithmic terms by absolute constants, we suppress this distinction below.

Assume the event \(\mathcal E_\delta\) defined in \eqref{eq:good-event} holds. Summing the inequality in Lemma~\ref{lem:lambda0-pathwise-recursion} over \(t=1,\dots,T\) yields
\begin{align*}
\sum_{t=0}^{T-1}\mathcal L_S(\bW_t)
\le\;&
4T\mathcal L_S(\bW^*)
+
\frac{3}{2\eta}\|\bW_0-\bW^*\|_2^2
+
6\eta\sum_{t=1}^T\|\Delta_t\|_2^2
+
6\eta \Big(\frac{C_{\mathrm{clip}}\kappa}{B}\Big)^2\sum_{t=1}^T\|Z_t\|_2^2
\\
&\quad
+
\mathfrak M_T^{(\Delta)}
+
\mathfrak M_T^{(Z)}.
\end{align*}

Plugging the results of Lemma~\ref{lem:lambda0-delta-square} into the above inequality  and dividing by \(T\) gives
\begin{align*}
\frac1T\sum_{t=0}^{T-1}\mathcal L_S(\bW_t)
\lesssim\;&
\mathcal L_S(\bW^*)
+
\frac{\|\bW_0-\bW^*\|_2^2}{\eta T}
+
G_\delta R_*
\sqrt{\frac{\log(1/\delta)}{BT}}
+
\frac{C_{\mathrm{clip}}\kappa}{B}\,R_*
\sqrt{\frac{\log(1/\delta)}{T}}
\\
&\quad + 
\eta G_\delta^2
\Big(
\frac{1}{B}+\frac{\log(1/\delta)}{T}
\Big)
+ \eta \frac{C_{\mathrm{clip}}^2\kappa^2}{B^2}
\Big(
md +\frac{\log(1/\delta)}{T}
\Big).
\end{align*}
The stated probability follows from a union bound over the four good events.
\end{proof}

Now, we present our optimization risk bound under the NTK separability assumption.
\begin{theorem}[Optimization Risk Bound]\label{thm:lambda0-ntk-rate}
Let $\delta\in(0,1)$.
Suppose Assumptions~\ref{ass:sigma} and \ref{ass:ntk} hold.  Let \(\{\bW_t\}_{t=0}^T\) be generated by Algorithm~\ref{alg:dp-lambda-minibatch} with $\eta\le \frac{1}{12C_{\sigma,b}p^3}$ and \(C_{\mathrm{clip}}= G_\delta\).   
Assume  
$m\gtrsim \max\big\{ \frac{\log(\frac{m}{\delta}) (\log^4(T) + \log^2(\frac{n}{\delta}) )}{\gamma^4}, \frac{\log(nT/\delta)}{Td} \big\}$ and $R_*
\asymp
\frac{1}{\gamma} (\log (T)+\sqrt{\log(n/\delta)} ).$ 
Then, with probability at least $1-\delta$ over the randomness of the initialization and the algorithmic randomness, 
\begin{align*}
\frac1T\sum_{t=0}^{T-1}\mathcal L_S(\bW_t)
\lesssim\;&
\frac{\log^2(T)+\log(n/\delta)}{\gamma^2\eta T}
+
\frac{\log(1/\delta)\bigl(\log (T)+\sqrt{\log(n/\delta)}\bigr)}{\gamma\sqrt{BT}}
\\
&\quad
+\frac{\log(1/\delta)\bigl(\log (T)+\sqrt{\log(n/\delta)}\bigr) \kappa }{\gamma B\sqrt{T}}
+
\eta \log(1/\delta)
\Big(
\frac{1}{B}+\frac{\log(1/\delta)}{T}
\Big)
\\
&\quad
+
 \frac{\eta md \log^2(1/\delta)\kappa^2}{B^2}=:A_{\mathrm{van}}.
\end{align*}
\end{theorem}
\begin{proof}
Fix $\delta_{\mathrm{init}}=\delta_{\mathrm{ntk}}=\delta_{\mathrm{alg}}=\frac{\delta}{3}.$ 
Take
\[
R_*
\asymp
\frac{1}{\gamma}\bigl(\log (T)+\sqrt{\log(n/\delta)}\bigr).
\]
Under the corresponding width conditions, Lemma~\ref{lem:ntk-comparator-lambda0} implies that, with probability at least \(1-\delta_{\mathrm{ntk}}\) over the initialization, there exists a comparator \(\bW^*\in\mathcal K\) such that
\[
\mathcal L_S(\bW^*)\le \frac1T,
\qquad
\|\bW^*-\bW_0\|_2^2
\lesssim
\frac{\log^2(T)+\log(n/\delta)}{\gamma^2}.
\]
Intersecting this event with \(\mathcal E_{\delta_{\mathrm{init}}}\) and then applying Theorem~\ref{thm:lambda0-opt} with failure probability \(\delta_{\mathrm{alg}}\), we obtain, with probability at least \(1-\delta\),
\begin{align*}
\frac1T\sum_{t=0}^{T-1}\mathcal L_S(\bW_t)
\lesssim\;&
\frac1T
+
\frac{\log^2(T)+\log(n/\delta)}{\gamma^2\eta T}
+
G_\delta R_*\sqrt{\frac{\log(1/\delta)}{BT}}
+
\frac{C_{\mathrm{clip}}\kappa}{B}\,R_*
\sqrt{\frac{\log(1/\delta)}{T}}
\\
&\quad +
\eta G_\delta^2
\Big(
\frac{1}{B}+\frac{\log(1/\delta)}{T}
\Big)
+
\eta \frac{C_{\mathrm{clip}}^2\kappa^2}{B^2}
\Big(
md+\frac{\log(1/\delta)}{T}
\Big).
\end{align*}
Recall that $C_{\mathrm{clip}} = G_\delta \lesssim  \sqrt{ \log(1 /\delta)},$ $\kappa^2\asymp \frac{ B^2 T \log(1/\delta)}{ n^2 \epsilon^2}$ and $R_*
\asymp
\frac{1}{\gamma}\bigl(\log (T)+\sqrt{\log(n/\delta)}\bigr)$.
\begin{align*}
\frac1T\sum_{t=0}^{T-1}\mathcal L_S(\bW_t)
\lesssim\;&
\frac{\log^2(T)+\log(n/\delta)}{\gamma^2\eta T}
+
\frac{\log(1/\delta)\bigl(\log (T)+\sqrt{\log(n/\delta)}\bigr)}{\gamma\sqrt{BT}}
\\
&\quad
+\frac{\log(1/\delta)\bigl(\log (T)+\sqrt{\log(n/\delta)}\bigr) \kappa }{\gamma B\sqrt{T}}
+
\eta \log(1/\delta)
\Big(
\frac{1}{B}+\frac{\log(1/\delta)}{T}
\Big)
\\
&\quad
+
 \frac{\eta md \log^2(1/\delta)\kappa^2}{B^2}.
\end{align*}
This proves the theorem.
\end{proof}

\subsection{Proofs for  population of  DP-SGD with independent noise}\label{appen:lambda0-gen}
The stability and generalization analysis developed for the correlated-noise case is noise-agnostic once the corresponding optimization good event is available. Therefore, when $\lambda=0$, the same argument applies verbatim after replacing $A_{\mathrm{corr}} $ by $A_{\mathrm{van}} $.  
Recall that $U_{\delta,\delta_f} 
\lesssim  \sqrt{\log(n/\delta)} $  and $b_{\delta,m}\lesssim 
 \sqrt{ \log(1/\delta) } $. 
\begin{theorem}[On-average argument stability for standard DP-SGD]
\label{thm:on-average-stability-lambda0}
Let \(\delta\in(0,1)\).  
Suppose Assumption~\ref{ass:sigma} holds and the parameter conditions of Theorem~\ref{thm:lambda0-ntk-rate} are satisfied. Assume further that, conditional on the initialization, an i.i.d.\ sample $S\sim\mathcal P^n$ satisfies Assumption~\ref{ass:ntk} with margin $\gamma$ with probability at least $1-\delta_{\rm sep}$, where $\delta_{\rm sep}\le\delta/8$. Assume 
\[
m
\gtrsim
 \log( \frac{m}{\delta}) 
\eta^2
\Big(
T^2 A^2_{\mathrm{van}} 
+
 \log( {n}/{\delta}) \log^2( {1}/{\delta })
\Bigl(
 \frac{T }{B} 
+1 
\Bigr)
\Big) .
\]
Then, with probability at least $1-\delta$ over initialization,
\[
\mathbb E_{S,\widetilde S,\A}\Big[
\frac1n\sum_{i=1}^n
\|\bW_T-\bW_T^{(i)}\|_2
\Big]
\lesssim 
\frac{ \eta \sqrt{\log(1/\delta)}}{n}
\sum_{t=0}^{T-1}\mathbb E_{S,\A}\bigl[\mathcal L_S(\bW_t)\bigr]
+
\frac{(\log (T)+\sqrt{\log(n/\delta)} )  \delta}{\gamma}.
\]
\end{theorem}
\begin{proof}
The proof follows the stability argument in the proof of Theorem~\ref{thm:lambda1-gen}, including the conditionalization and sample-separability arguments for $S$ and $S^{(i)}$. The only difference is that the optimization good event \(\mathcal E_{\mathrm{opt}}^{(i)}\) is now supplied by Theorem~\ref{thm:lambda0-ntk-rate} instead of Theorem~\ref{thm:lambda-positive-ntk-opt}. Under the stated lower bound on \(m\), Lemma~\ref{lem:control-M-corr} applies with \(A_{\mathrm{corr}}\) replaced by \(A_{\mathrm{van}}\), and the rest of the argument is unchanged.
\end{proof}

\begin{theorem}[Population risk bound]\label{thm:DPSGD-gen} 
Suppose Assumption~\ref{ass:sigma} holds and the parameter conditions of Theorem~\ref{thm:lambda0-ntk-rate} are satisfied. Assume further that, conditional on the initialization, an i.i.d.\ sample $S\sim\mathcal P^n$ satisfies Assumption~\ref{ass:ntk} with margin $\gamma$ with probability at least $1-\delta_{\rm sep}$, where $\delta_{\rm sep}
\le
\frac{\delta}{8  (1+\log(\frac n\delta)(\sqrt m+R_*) )}$. 
Assume  \[
m \gtrsim \log( \frac{m}{\delta}) 
\eta^2 \Big( T^2 A^2_{\mathrm{van}} 
+  \log( {n}/{\delta}) \log^2( {1}/{\delta })
\Bigl(  \frac{T }{B}  +1  \Bigr) \Big) .
\]
Then, with probability at least $1-\delta$ over initialization,
\[
\frac1T\!\sum_{t=0}^{T-1}\!\mathbb E_{S,\A}\!\bigl[\mathcal L(\bW_t)\bigr]
\!\lesssim\!
\Bigl(1+\frac{\eta  T\log(\frac{1}{\delta})}{n}\Bigr)
\bigl(A_{\mathrm{van}} + \delta \sqrt{\log(\frac{n}{\delta})} \bigr)
+ \sqrt{\log(\frac{1}{\delta})}
\frac{(\log (T)\!+\!\!\sqrt{\log(\frac{n}{\delta})} )\delta}{\gamma} .
\]
\end{theorem}
\begin{proof}
By Lemma~\ref{lem:hessian} and \(|\ell'(u)|\le 1\) for the logistic loss, for every \(\bW\in\mathcal K\) and every \((\bx,y)\in\mathcal Z\),
\[
\|\nabla_{\bW}\ell(yf_{\bW}(\bx))\|_2
=
|\ell'(yf_{\bW}(\bx))|\,\|\nabla f_{\bW}(\bx)\|_2
\le
b_{\delta_0,m}.
\]
Hence the loss is \(b_{\delta_0,m}\)-Lipschitz with respect to \(\bW\). Applying Lemma~\ref{lem:connection} together with Theorem~\ref{thm:on-average-stability-lambda0} yields 
\begin{equation}\label{eq:gen-inde}
    \mathbb E_{S,\A}\bigl[\mathcal L(\bW_T)-\mathcal L_S(\bW_T)\bigr]
\lesssim
\frac{\eta  \log(1/\delta)}{n}
\sum_{t=0}^{T-1}\mathbb E_{S,\A}\bigl[\mathcal L_S(\bW_t)\bigr]
+
 \sqrt{\log(\frac{1}{\delta})}\frac{ (\log (T)+\sqrt{\log(n/\delta)} )\delta}{\gamma}.
\end{equation}

Now, we consider the population risk bound.
Since replacing \(\delta\) by a smaller auxiliary failure probability only changes logarithmic factors, we suppress this distinction below.

On the initialization event \(\mathcal E_{\delta_0}\), for any \(\bx\in \mathcal X\), any \(y\in\{-1,+1\}\), and any \(\bW\in\mathcal K=B(\bW_0,R_*)\), it holds that
\[
|f_{\bW}(\bx)|
\le
|f_{\bW_0}(\bx)|+\|\nabla f_{\widetilde{\bW}}(\bx)\|_2\,\|\bW-\bW_0\|_2
\le
B_b\sqrt p\,\|c\|_2+b_{\delta_0,m}R_*
\lesssim
\sqrt m+\log(\tfrac n\delta)R_*,
\]
where we used \(\|c\|_2\lesssim \log(1/\delta_0)\sqrt m\) on \(\mathcal E_{\delta_0}\), the bound
\(\|\nabla f_{\widetilde{\bW}}(\bx)\|_2\le b_{\delta_0,m}\) from Lemma~\ref{lem:hessian}, and \(\|\bW-\bW_0\|_2\le R_*\).
Consequently,
\[
\ell(yf_{\bW}(\bx))
=
\log\bigl(1+e^{-y f_{\bW}(\bx)}\bigr)
\le
\log 2+|f_{\bW}(\bx)|
\lesssim
\sqrt m+\log(\tfrac n\delta)R_*.
\]
Since \(\bW_t\in\mathcal K\) for all \(t\) by construction of Algorithm~\ref{alg:dp-lambda-minibatch}, we obtain the uniform bound
\[
\mathcal L_S(\bW_t)\lesssim \log(\tfrac n\delta)(\sqrt m+R_*)
\qquad\forall t.
\]

Now choose an auxiliary optimization failure probability $\delta_{\mathrm{opt}}'
 =
\frac{\delta}{1+\log(\tfrac n\delta)(\sqrt m+R_*)}$. 
Applying Theorem~\ref{thm:lambda0-ntk-rate} with a sufficiently smaller auxiliary failure probability and Lemma~\ref{lem:conditional-failure}, there exists an initialization event $\mathcal{E}_{\mathrm{cond}}'$ with probability at least $1-\delta/8$ such that, on $\mathcal E_{\mathrm{cond}}'$, the conditional optimization failure probability is at most $\delta_{\mathrm{opt}}'$. Here the smaller auxiliary failure probability only changes logarithmic factors.
Let \(\mathcal E_{\mathrm{opt}}\) be the event on which $S$ satisfies Assumption~\ref{ass:ntk} with margin $\gamma$ and
\[
\frac1T\sum_{t=0}^{T-1}\mathcal L_S(\bW_t)\le A_{\mathrm{van}}.
\]
By the sample-separability assumption, on $\mathcal E_{\mathrm{cond}}'$,
$\mathbb P(\mathcal E_{\mathrm{opt}}^c)\le \delta_{\rm sep}+\delta_{\mathrm{opt}}'$. 
Therefore, conditioning on \(\mathcal E_{\delta_0}\cap\mathcal E_{\mathrm{cond}}'\), we have
\begin{align*}
\mathbb E_{S,\A}\Big[\frac1T\sum_{t=0}^{T-1}\mathcal L_S(\bW_t)\Big]
&=
\mathbb E_{S,\A}\Big[\frac1T\sum_{t=0}^{T-1}\mathcal L_S(\bW_t)\mathbf 1_{\mathcal E_{\mathrm{opt}}}\Big]
+
\mathbb E_{S,\A}\Big[\frac1T\sum_{t=0}^{T-1}\mathcal L_S(\bW_t)\mathbf 1_{\mathcal E_{\mathrm{opt}}^c}\Big] \\
&\lesssim
A_{\mathrm{van}}
+
\log(\tfrac n\delta)(\sqrt m+R_*)\,\mathbb P(\mathcal E_{\mathrm{opt}}^c) \\
&\lesssim
A_{\mathrm{van}}+\delta
\lesssim
A_{\mathrm{van}}+\delta\sqrt{\log(\tfrac n\delta)}.
\end{align*}

Combining the above observation with the first part of the theorem gives
\begin{align*}
\frac1T\sum_{t=0}^{T-1}\mathbb E_{S,\A}[\mathcal L(\bW_t)]
\lesssim\;&
A_{\mathrm{van}} + \delta \sqrt{\log(\tfrac{n}{\delta})}
+
\frac{\eta \log(\tfrac{1}{\delta})}{n}
\sum_{t=0}^{T-1}\mathbb E_{S,\A}[\mathcal L_S(\bW_t)]
\\
&\quad
+
\sqrt{\log(\tfrac{1}{\delta})}\frac{(\log (T)+\sqrt{\log(\tfrac{n}{\delta})} )\delta}{\gamma} .
\end{align*}
Substituting
\[
\sum_{t=0}^{T-1}\mathbb E_{S,\A}[\mathcal L_S(\bW_t)]
\lesssim
T\bigl(A_{\mathrm{van}} + \delta\sqrt{\log(\tfrac{n}{\delta})}\bigr)
\]
yields the claimed bound.
\end{proof}

The result below shows that independent noise DP-SGD attains the rate \(\O(n^{-1/2}+\sqrt d/(n\epsilon))\).
\begin{corollary}[Risk rates under representative parameter regime]
\label{cor:DPSGD-rates}
Suppose Assumptions~\ref{ass:sigma} and \ref{ass:ntk} hold.  Assume $\delta \le \frac{\gamma}{n}$.   Let \(\{\bW_t\}_{t=0}^{T-1}\) be generated by Algorithm~\ref{alg:dp-lambda-minibatch} with $\eta\le \frac{1}{12C_{\sigma,b}p^3}$ and \(C_{\mathrm{clip}}= G_\delta\).  If $m\asymp \frac{\mathrm{polylog}(n/\delta)}{\gamma^4}$, set $B\asymp \gamma\sqrt n$,
$\eta\asymp \min\big\{1,\frac{\gamma^2\sqrt n\,\epsilon}{\sqrt d}\big\}$ and 
$T\asymp \frac{\sqrt n}{\gamma}$.
With probability at least \(1-\delta\) over the initialization and the algorithmic randomness,  \[\frac1T\sum_{t=0}^{T-1}\mathcal L_S(\bW_t)
\lesssim 
\frac{1}{\gamma\sqrt n}
+
\frac{\sqrt d}{\gamma^3 n\epsilon}.\] 
For the population risk statement, assume further the sample-separability condition of Theorem~\ref{thm:DPSGD-gen}. Moreover, with probability at least \(1-\delta\) over the initialization,   \[\frac1T\sum_{t=0}^{T-1}\mathbb E_{S,\A}\bigl[\mathcal L(\bW_t)\bigr]
\lesssim 
\frac{1}{\gamma\sqrt n}
+
\frac{\sqrt d}{\gamma^3 n\epsilon}.\]
\end{corollary}
\begin{proof} 
We first check the optimization bound in Theorem~\ref{thm:lambda0-ntk-rate}. Under the stated choice of parameters, $\frac{1}{\gamma\sqrt{BT}}
\asymp
\frac{1}{\gamma\sqrt{\gamma\sqrt n\cdot \sqrt n/\gamma}}
\asymp
\frac{1}{\gamma\sqrt n}$. 
Moreover, $\frac{1}{\gamma^2\eta T}
\asymp
\frac{1}{\gamma\eta\sqrt n}$. 
If \(\eta\asymp 1\), then $\frac{1}{\gamma^2\eta T}\asymp \frac{1}{\gamma\sqrt n}$.
If \(\eta\asymp \frac{\gamma^2\sqrt n\,\epsilon}{\sqrt d}\), then $\frac{1}{\gamma^2\eta T}
\asymp
\frac{\sqrt d}{\gamma^3 n\epsilon}$.
Hence, in all cases, $\frac{1}{\gamma^2\eta T}
\lesssim
\frac{1}{\gamma\sqrt n}
+
\frac{\sqrt d}{\gamma^3 n\epsilon}$. 

For the term \(\eta/B\), if \(\eta\asymp 1\), then $\frac{\eta}{B}\asymp \frac{1}{\gamma\sqrt n}$.
If \(\eta\asymp \frac{\gamma^2\sqrt n\,\epsilon}{\sqrt d}\), then 
$\frac{\eta}{B}
\asymp
\frac{\gamma\epsilon}{\sqrt d}
\lesssim
\frac{\sqrt d}{\gamma^3 n\epsilon}$, 
where the last inequality follows from the branch condition
\(\gamma^2\sqrt n\,\epsilon\le \sqrt d\).
Therefore, in all cases, $\frac{\eta}{B}
\lesssim
\frac{1}{\gamma\sqrt n}
+
\frac{\sqrt d}{\gamma^3 n\epsilon}$. 

By Theorem~\ref{thm:DPlambda0},  $\kappa^2 \asymp \frac{B^2T}{n^2\epsilon^2}$. Since \(B\asymp \gamma\sqrt n\) and \(T\asymp \sqrt n/\gamma\), this gives
$\kappa^2 \lesssim \frac{\gamma}{\sqrt n\,\epsilon^2}$,
and 
$\frac{\kappa}{\gamma B\sqrt T}
\asymp \frac{1}{\gamma n\epsilon}
\lesssim
\frac{\sqrt d}{\gamma^3 n\epsilon}$.

For the quadratic noise term, using \(\kappa^2/B^2\asymp T/(n^2\epsilon^2)\), we have $\frac{\eta m d\kappa^2}{B^2}
\asymp
\frac{\eta m dT}{n^2\epsilon^2}$. 
If \(\eta\asymp 1\), then
$\frac{\eta m d\kappa^2}{B^2}
\asymp
\frac{m d}{\gamma n^{3/2}\epsilon^2}
\lesssim
\frac{1}{\gamma\sqrt n}$, where we used the branch condition \(\epsilon\ge \sqrt d/(\gamma^2\sqrt n)\) and \(m\asymp \polylog(n/\delta)/\gamma^4\). If \(\eta\asymp \frac{\gamma^2\sqrt n\,\epsilon}{\sqrt d}\), then
 $\frac{\eta m d\kappa^2}{B^2}
\asymp
\frac{m\gamma\sqrt d}{n\epsilon}
\lesssim
\frac{\sqrt d}{\gamma^3 n\epsilon}$. 
Hence, $\frac{\eta m d\kappa^2}{B^2}
\lesssim
\frac{1}{\gamma\sqrt n}
+
\frac{\sqrt d}{\gamma^3 n\epsilon}$. 

Combining the above estimates, Theorem~\ref{thm:lambda0-ntk-rate} yields
\[
\frac1T\sum_{t=0}^{T-1}\mathcal L_S(\bW_t)
\lesssim
\frac{1}{\gamma\sqrt n}
+
\frac{\sqrt d}{\gamma^3 n\epsilon}
\]
with probability at least \(1-\delta\).

Now we show the generalization and population risk rates. By \eqref{eq:gen-inde} and the above optimization bound,
\[
\mathbb E_{S,\A}\bigl[\mathcal L(\bW_T)-\mathcal L_S(\bW_T)\bigr]
\lesssim
\frac{\eta \log(1/\delta)}{n}
\sum_{t=0}^{T-1}\mathbb E_{S,\A}\bigl[\mathcal L_S(\bW_t)\bigr]
+
\frac{\delta}{\gamma},
\]
where we suppress logarithmic factors in the first term. Therefore,
\[
\mathbb E_{S,\A}\bigl[\mathcal L(\bW_T)-\mathcal L_S(\bW_T)\bigr]
\lesssim
\frac{\eta T}{n}
\Big(
\frac{1}{\gamma\sqrt n}
+
\frac{\sqrt d}{\gamma^3 n\epsilon}
\Big)
+
\frac{\delta}{\gamma}.
\]
Since $\frac{\eta T}{n}\asymp \min\bigl\{\frac{1}{\gamma\sqrt n},\,\frac{\gamma\epsilon}{\sqrt d}\bigr\},$ 
we have $\frac{\eta T}{n}\cdot \frac{1}{\gamma\sqrt n}
\lesssim \frac{1}{\gamma^2 n}
\qquad\text{and}\qquad
\frac{\eta T}{n}\cdot \frac{\sqrt d}{\gamma^3 n\epsilon}
\lesssim \frac{1}{\gamma^2 n}$. 
Because \(\delta\lesssim \gamma/n\), the residual term \(\delta/\gamma\) is also of order at most \(1/n\), and hence
\[
\mathbb E_{S,\A}\bigl[\mathcal L(\bW_T)-\mathcal L_S(\bW_T)\bigr]
\lesssim \frac{1}{\gamma^2 n}.
\]

Further, Theorem~\ref{thm:DPSGD-gen} implies
\[
\frac1T\sum_{t=0}^{T-1}\mathbb E_{S,\A}\bigl[\mathcal L(\bW_t)\bigr]
\lesssim
\Bigl(1+\frac{\eta T}{n}\Bigr)\bigl(A_{\mathrm{van}} +\delta\bigr)+\frac{\delta}{\gamma},
\]
where logarithmic factors are suppressed. Since \(\eta T/n\lesssim 1\), \(\delta\lesssim \gamma/n\), and $A_{\mathrm{van}}
\lesssim
\frac{1}{\gamma\sqrt n}
+
\frac{\sqrt d}{\gamma^3 n\epsilon}$, 
we obtain
\[
\frac1T\sum_{t=0}^{T-1}\mathbb E_{S,\A}\bigl[\mathcal L(\bW_t)\bigr]
\lesssim
\frac{1}{\gamma\sqrt n}
+
\frac{\sqrt d}{\gamma^3 n\epsilon}.
\]

Finally, we verify that the width conditions in Theorems~\ref{thm:lambda0-ntk-rate} and~\ref{thm:DPSGD-gen} are satisfied. Since $m\asymp \frac{\polylog(n/\delta)}{\gamma^4}$,
it suffices to check that $\eta^2\Bigl(T^2A_{\mathrm{van}}^2 +\frac{T}{B}+1\Bigr)
\lesssim \frac{1}{\gamma^4}$.
Since \(T/B\asymp 1/\gamma^2\), then \(\eta^2(T/B+1)\lesssim 1/\gamma^4\). Also,
$A_{\mathrm{van}}
\lesssim
\frac{1}{\gamma\sqrt n}
+
\frac{\sqrt d}{\gamma^3 n\epsilon}$,
and $\frac{\eta T}{\gamma\sqrt n}\lesssim \frac{1}{\gamma^2},
\qquad
\frac{\eta T \sqrt d}{\gamma^3 n\epsilon}\lesssim \frac{1}{\gamma^2}$
in both cases \(\eta\asymp 1\) and \(\eta\asymp \frac{\gamma^2\sqrt n\,\epsilon}{\sqrt d}\). Therefore
$(\eta T)^2A_{\mathrm{van}}^2 \lesssim \frac{1}{\gamma^4}$, 
and Theorems~\ref{thm:lambda0-ntk-rate} and~\ref{thm:DPSGD-gen} apply.
This completes the proof.
\end{proof}

\paragraph{Recovery of Full-batch DP-GD.}

Since $B=n$, the mini-batch gradient coincides with the full empirical gradient, so $\Delta_t=0$ for all $t$ (see \eqref{eq:Delta}). The sampling fluctuation term disappears and the recursion reduces to full-batch DP-GD. We therefore specialize the bounds in Theorems~\ref{thm:lambda0-ntk-rate} and~\ref{thm:DPSGD-gen} to this regime. For the generalization and population risk statements below, assume further the sample-separability condition of Theorem~\ref{thm:DPSGD-gen}.

Assume $m\asymp \frac{\polylog(n/\delta)}{\gamma^4}, $ $
T\asymp n, $ $
\eta\asymp \min\Bigl\{1,\frac{\epsilon}{\gamma\sqrt{md}}\Bigr\}, $ $
C_{\rm clip}=G_\delta$, 
and $\delta\lesssim\gamma/n$. For $B=n$, the privacy guarantee follows directly from the standard Gaussian mechanism. Up to logarithmic factors, we have $\kappa^2\asymp \frac{T}{\epsilon^2}\asymp\frac{n}{\epsilon^2} $ and $
\kappa\asymp\frac{\sqrt n}{\epsilon}$. 

Plugging these parameter choices into the optimization bound in Theorem~\ref{thm:lambda0-ntk-rate}, we first note that $\frac{1}{\gamma n}\lesssim\frac{1}{\gamma^2n}, $ $
\frac{\eta}{n}\lesssim\frac{1}{\gamma^2n} $ and $
\frac{\kappa}{\gamma n\sqrt n}\asymp\frac{1}{\gamma n\epsilon}\lesssim\frac{\sqrt d}{\gamma^3n\epsilon}$. 
Moreover, $\frac{1}{\gamma^2\eta n}
\lesssim
\frac{1}{\gamma^2n}+\frac{\sqrt{md}}{\gamma n\epsilon}
\lesssim
\frac{1}{\gamma^2n}+\frac{\sqrt d}{\gamma^3n\epsilon}$, 
where we used $m\asymp\gamma^{-4}\polylog(n/\delta)$. Similarly, $\frac{\eta md\kappa^2}{n^2}
\asymp
\frac{\eta md}{n\epsilon^2}
\lesssim
\frac{1}{\gamma^2n}+\frac{\sqrt d}{\gamma^3n\epsilon}$. 
Therefore,
\[
\frac1T\sum_{t=0}^{T-1}\mathcal L_S(\bW_t)
\lesssim
\frac{1}{\gamma^2n}+\frac{\sqrt d}{\gamma^3n\epsilon}
\]
with probability at least $1-\delta$.

For the generalization risk, Theorem~\ref{thm:DPSGD-gen} gives
\[
\mathbb E_{S,\A}\bigl[\mathcal L(\bW_T)-\mathcal L_S(\bW_T)\bigr]
\lesssim
\frac{\eta}{n}\sum_{t=0}^{T-1}\mathbb E_{S,\A}\bigl[\mathcal L_S(\bW_t)\bigr]
+\frac{\delta}{\gamma}.
\] . Using $T\asymp n$ and the optimization bound above,
\[
\mathbb E_{S,\A}\bigl[\mathcal L(\bW_T)-\mathcal L_S(\bW_T)\bigr]
\lesssim
\eta\Bigl(\frac{1}{\gamma^2n}+\frac{\sqrt d}{\gamma^3n\epsilon}\Bigr)+\frac{\delta}{\gamma}.
\]
By the choice $\eta\asymp\min\{1,\epsilon/(\gamma\sqrt{md})\}$ and $m\asymp\gamma^{-4}\polylog(n/\delta)$, $\eta\frac{\sqrt d}{\gamma^3n\epsilon}\lesssim\frac{1}{\gamma^2n}$.  
Together with $\delta\lesssim\gamma/n$, this yields
\[
\mathbb E_{S,\A}\bigl[\mathcal L(\bW_T)-\mathcal L_S(\bW_T)\bigr]
\lesssim
\frac{1}{\gamma^2n}.
\]

Furthermore, Theorem~\ref{thm:DPSGD-gen} shows
\[
\frac1T\sum_{t=0}^{T-1}\mathbb E_{S,\A}\bigl[\mathcal L(\bW_t)\bigr]
\lesssim
\Bigl(1+\frac{\eta T}{n}\Bigr)\bigl(A_{\mathrm{van}}+\delta\bigr)+\frac{\delta}{\gamma}.
\]
Since $T\asymp n$, $\eta\lesssim1$, $\delta\lesssim\gamma/n$, and $A_{\mathrm{van}}
\lesssim
\frac{1}{\gamma^2n}+\frac{\sqrt d}{\gamma^3n\epsilon}$, 
we obtain
\[
\frac1T\sum_{t=0}^{T-1}\mathbb E_{S,\A}\bigl[\mathcal L(\bW_t)\bigr]
\lesssim
\frac{1}{\gamma^2n}+\frac{\sqrt d}{\gamma^3n\epsilon}.
\]

Finally, we verify the width condition in Theorem~\ref{thm:DPSGD-gen}. Since $T/B=1$ and $\eta T A_{\mathrm{van}}
\lesssim
\frac{\eta}{\gamma^2}
+\frac{\eta\sqrt d}{\gamma^3\epsilon}
\lesssim
\frac{1}{\gamma^2}$, 
we have $\eta^2\bigl(T^2A_{\mathrm{van}}^2+\frac{T}{B}+1\bigr)\lesssim\frac{1}{\gamma^4}$. 
Hence the required width condition is satisfied by $m\asymp\gamma^{-4}\polylog(n/\delta)$. This completes the proof.

\section{Proofs for Mini-batch SGD}\label{appen:SGD}

We now specialize the preceding analysis to non-private mini-batch SGD by setting the noise level $\kappa=0$. For the representative rates below, we take $C_{\mathrm{clip}}=G_\delta$, so that clipping is inactive.

\begin{theorem}[Optimization, generalization and population risk bounds]
\label{thm:SGD-opt}
Let $\delta\in(0,1)$. Suppose Assumptions~\ref{ass:sigma} and~\ref{ass:ntk} hold. Let $\{\bW_t\}_{t=0}^T$ be generated by mini-batch SGD with $\eta\le 1/(12C_{\sigma,b}p^3)$ and $C_{\mathrm{clip}}=G_\delta$. Assume $m\gtrsim \frac{\log(m/\delta) (\log^4(T)+\log^2(n/\delta) )}{\gamma^4}$. 
Then, with probability at least $1-\delta$ over the initialization and the algorithmic randomness,
\[
\frac1T\sum_{t=0}^{T-1}\mathcal L_S(\bW_t)
\lesssim
\frac{1}{\gamma^2\eta T}
+\frac{1}{\gamma\sqrt{BT}}
+\frac{\eta}{B}
=:A_{\mathrm{non}}.
\]
For the generalization and population risk statements below, assume further the sample-separability condition of Theorem~\ref{thm:DPSGD-gen}. Further assume $m\gtrsim \log(m/\delta)\eta^2\bigl(T^2A_{\mathrm{non}}^2+\log(n/\delta)\log^2(1/\delta)\bigl(\frac{T}{B}+1\bigr)\bigr)$. 
Then, with probability at least $1-\delta$ over the initialization,
\[
\mathbb E_{S,\A}\bigl[\mathcal L(\bW_T)-\mathcal L_S(\bW_T)\bigr]
\lesssim
\frac{\eta}{n}\sum_{t=0}^{T-1}\mathbb E_{S,\A}\bigl[\mathcal L_S(\bW_t)\bigr]
+\frac{\delta}{\gamma},
\]
and
\[
\frac1T\sum_{t=0}^{T-1}\mathbb E_{S,\A}\bigl[\mathcal L(\bW_t)\bigr]
\lesssim
\Bigl(1+\frac{\eta T}{n}\Bigr)\bigl(A_{\mathrm{non}}+\delta\bigr)
+\frac{\delta}{\gamma}.
\]
\end{theorem}

\begin{proof}
The non-private mini-batch SGD bounds are obtained as direct special cases of the standard DP-SGD analysis by setting $\kappa=0$, including the sample-separability condition required for the generalization and population risk statements. All privacy-noise terms then disappear, and the corresponding optimization, generalization, and population risk bounds follow immediately.
\end{proof}

\begin{corollary}[Risk rates under representative parameter regime]\label{cor:SGD-optimal-rate}
Suppose Assumptions~\ref{ass:sigma} and~\ref{ass:ntk} hold. Let $\{\bW_t\}_{t=0}^{T-1}$ be generated by mini-batch SGD with $C_{\mathrm{clip}}=G_\delta$. Assume $\delta\le 1/(\gamma n)$ and $m\gtrsim \frac{\polylog(n/\delta)}{\gamma^4}$. 
Set $B\lesssim \gamma\sqrt n,$ $
\eta\asymp\frac{B}{n},$ and $
T\gtrsim\frac{n^2}{\gamma^2B}$. 
Then, with probability at least $1-\delta$ over the initialization and the algorithmic randomness,
\[
\frac1T\sum_{t=0}^{T-1}\mathcal L_S(\bW_t)\lesssim\frac1n.
\]
For the population risk statement, assume further that the sample-separability condition in Theorem~\ref{thm:DPSGD-gen} holds. If $T\asymp n^2/(\gamma^2B)$, then with probability at least $1-\delta$ over the initialization,
\[
\frac1T\sum_{t=0}^{T-1}\mathbb E_{S,\A}\bigl[\mathcal L(\bW_t)\bigr]
\lesssim
\frac{1}{\gamma^2n}.
\]
\end{corollary}
\begin{proof}
Set $B\lesssim\gamma\sqrt n$, $\eta\asymp B/n$, and $T\gtrsim n^2/(\gamma^2B)$. Under these choices, $\frac{\eta}{B}\asymp\frac1n$. 
Moreover,$\frac{1}{\gamma^2\eta T}
\asymp
\frac{n}{\gamma^2BT}
\lesssim
\frac1n,$
and $
\frac{1}{\gamma\sqrt{BT}}
\lesssim
\frac{1}{\gamma\sqrt{B\cdot n^2/(\gamma^2B)}}
=
\frac1n.$
Hence Theorem~\ref{thm:SGD-opt} gives
$A_{\mathrm{non}}\lesssim\frac1n$
and therefore
\[
\frac1T\sum_{t=0}^{T-1}\mathcal L_S(\bW_t)\lesssim\frac1n
\]
with probability at least $1-\delta$ over the initialization and the algorithmic randomness.

Now further take $T\asymp n^2/(\gamma^2B)$. Then
\[
\sum_{t=0}^{T-1}\mathbb E_{S,\A}\bigl[\mathcal L_S(\bW_t)\bigr]
\lesssim
\frac{T}{n}.
\]
By Theorem~\ref{thm:SGD-opt},
\[
\mathbb E_{S,\A}\bigl[\mathcal L(\bW_T)-\mathcal L_S(\bW_T)\bigr]
\lesssim
\frac{\eta}{n}\sum_{t=0}^{T-1}\mathbb E_{S,\A}\bigl[\mathcal L_S(\bW_t)\bigr]
+\frac{\delta}{\gamma}.
\]
Since
$\eta T\asymp\frac{B}{n}\cdot\frac{n^2}{\gamma^2B}=\frac{n}{\gamma^2},$
and $\delta\le1/(\gamma n)$, we obtain
\[
\mathbb E_{S,\A}\bigl[\mathcal L(\bW_T)-\mathcal L_S(\bW_T)\bigr]
\lesssim
\frac{1}{\gamma^2n}.
\]

For the averaged population risk, Theorem~\ref{thm:SGD-opt} gives
\[
\frac1T\sum_{t=0}^{T-1}\mathbb E_{S,\A}\bigl[\mathcal L(\bW_t)\bigr]
\lesssim
\Bigl(1+\frac{\eta T}{n}\Bigr)\bigl(A_{\mathrm{non}}+\delta\bigr)
+\frac{\delta}{\gamma}.
\]
Using
$A_{\mathrm{non}}\lesssim\frac1n,$ $
\frac{\eta T}{n}\asymp\frac1{\gamma^2},$ and $
\delta\le\frac1{\gamma n},$
we obtain
\[
\frac1T\sum_{t=0}^{T-1}\mathbb E_{S,\A}\bigl[\mathcal L(\bW_t)\bigr]
\lesssim
\frac{1}{\gamma^2n}.
\]

Finally, we verify the additional width condition in Theorem~\ref{thm:SGD-opt}. Since $A_{\mathrm{non}}\lesssim1/n$ and $T\asymp n^2/(\gamma^2B)$,
$\eta^2T^2A_{\mathrm{non}}^2
\lesssim
\frac1{\gamma^4},$
while
$\eta^2\Bigl(\frac{T}{B}+1\Bigr)
\lesssim
\frac1{\gamma^4}.$
Hence $m\gtrsim\gamma^{-4}\polylog(n/\delta)$ is sufficient. This completes the proof.
\end{proof}

\paragraph{Recovery of full-batch GD.}

Take
$m\gtrsim\frac{\polylog(n/\delta)}{\gamma^4},\qquad
B=n,$ $
\eta\asymp1,$ and $
C_{\mathrm{clip}}=G_\delta.$ 
Since $B=n$, the mini-batch gradient coincides with the full empirical gradient, so $\Delta_t=0$ for all $t$ (see \eqref{eq:Delta}). Hence the sampling fluctuation terms disappear from the underlying recursion, and the bound reduces to
\[
\frac1T\sum_{t=0}^{T-1}\mathcal L_S(\bW_t)
\lesssim
\frac{1}{\gamma^2T}
\]
with probability at least $1-\delta$ over the initialization.

Now set
$T\asymp\frac{n}{\gamma^2}.
$
Then
$\frac1T\sum_{t=0}^{T-1}\mathcal L_S(\bW_t)\lesssim\frac1n.$
For the generalization and population risk statements below, assume further the sample-separability condition of Theorem~\ref{thm:SGD-opt}. Then
\[
\sum_{t=0}^{T-1}\mathbb E_{S,\A}\bigl[\mathcal L_S(\bW_t)\bigr]
\lesssim
\frac{T}{n}
\asymp
\frac1{\gamma^2}.
\]
Therefore, by the generalization bound in Theorem~\ref{thm:SGD-opt},
\[
\mathbb E_{S,\A}\bigl[\mathcal L(\bW_T)-\mathcal L_S(\bW_T)\bigr]
\lesssim
\frac{1}{n}\sum_{t=0}^{T-1}\mathbb E_{S,\A}\bigl[\mathcal L_S(\bW_t)\bigr]
+\frac{\delta}{\gamma}
\lesssim
\frac{1}{\gamma^2n},
\]
where we used $\delta\le1/(\gamma n)$.

Finally, the averaged population-risk bound in Theorem~\ref{thm:SGD-opt} yields
\[
\frac1T\sum_{t=0}^{T-1}\mathbb E_{S,\A}\bigl[\mathcal L(\bW_t)\bigr]
\lesssim
\Bigl(1+\frac{T}{n}\Bigr)\Bigl(\frac1n+\delta\Bigr)+\frac{\delta}{\gamma}
\lesssim
\frac{1}{\gamma^2n}.
\]
The corresponding width condition is satisfied since
$T^2 (\frac1n )^2\asymp\frac1{\gamma^4},$
and therefore $m\gtrsim\gamma^{-4}\polylog(n/\delta)$ is sufficient. This completes the proof.

\section{Further Experiments}\label{app:experiments}

\subsection{Fixed-Privacy Experiments}
Figure~\ref{fig:synthetic-fixed-privacy} (Left) follows the experimental setup of \cite{kalinin2026dp}. The main modification is that we replace the CNN used in their experiments with the two-layer KAN considered in this paper, while studying the effect of temporal correlation under a fixed privacy budget. We additionally conduct the same experiment on synthetic logistic data generated following the setup of \cite{wang2026optimization}.
Unlike the fixed-noise experiments in Section~\ref{sec:main_opt}, these experiments evaluate held-out predictive performance under a common privacy budget, and thus provide an empirical population-level comparison across different correlation levels. 

\begin{figure}[t]
    \centering
    \begin{minipage}{0.49\textwidth}
        \centering
        \includegraphics[width=0.98\textwidth]{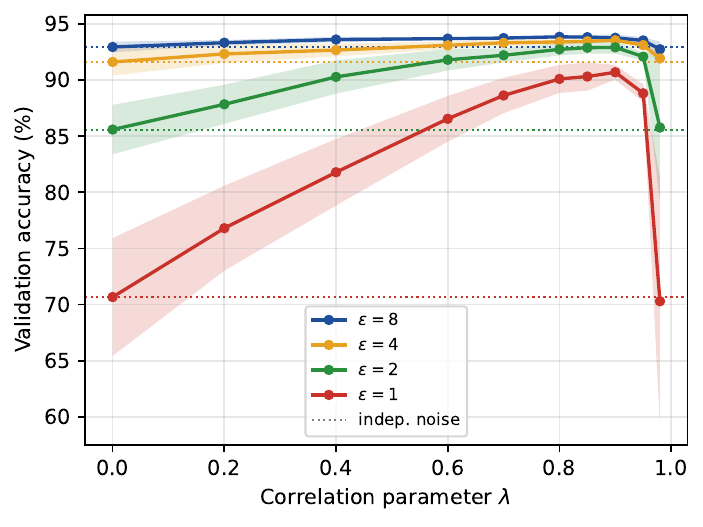}
    \end{minipage}
    \hfill
    \begin{minipage}{0.49\textwidth}
        \centering
        \includegraphics[width=0.98\textwidth]{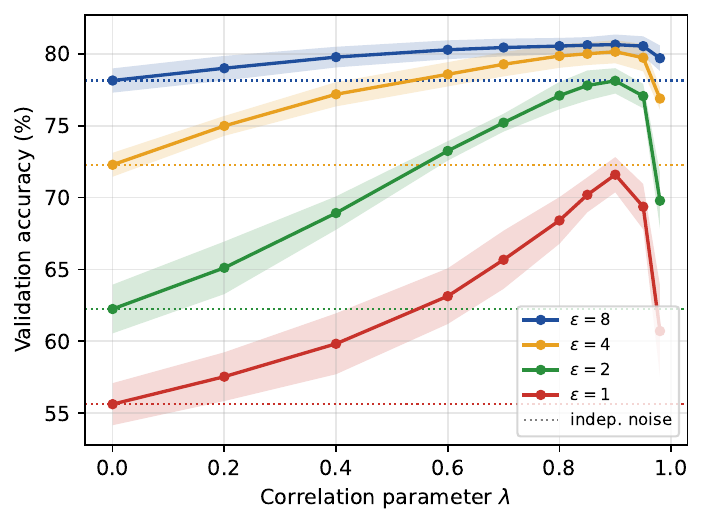}
    \end{minipage}
    \vspace{-3mm}
   \caption{Effect of temporal correlation on two-layer KANs trained under a fixed privacy budget on MNIST  and  synthetic logistic data. The noise multiplier is recalibrated for each value of the correlation parameter $\lambda$ to maintain the same privacy budget.}
    \label{fig:synthetic-fixed-privacy}  
\end{figure}

\subsection{Further Results for Section~\ref{sec:main_opt}}

\paragraph{Held-out accuracy for $\eta=1$.}
We first report the held-out accuracy corresponding to the $\eta=1$ experiments in Figure~\ref{fig:fixed-kappa-eta1}. Although Theorem~\ref{thm:lambda-positive-ntk-opt} directly controls the empirical loss rather than the prediction accuracy, Figure~\ref{fig:fixed-kappa-eta1-acc} shows that the reduction in training loss is also accompanied by improved predictive performance. On both synthetic logistic data and MNIST 3 vs.\ 8, the held-out accuracy generally increases as $\lambda$ increases, with the improvement being more pronounced at the larger noise multiplier $\kappa=11.8$. Small fluctuations under strong correlation are within the variation across random seeds. Thus, the improvement observed in the training objective is also reflected in held-out predictive performance.

\begin{figure}[t]
    \centering
    \begin{subfigure}[t]{0.49\linewidth}
        \centering
        \includegraphics[width=\linewidth]{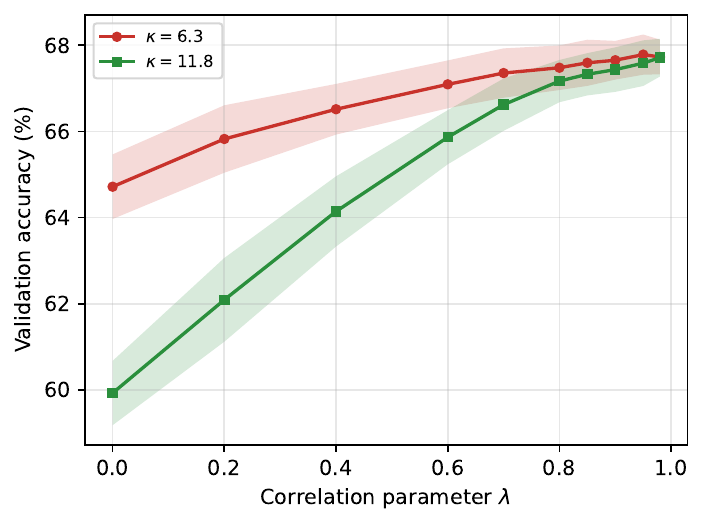}
        \caption{Synthetic logistic data.}
    \end{subfigure}
    \hfill
    \begin{subfigure}[t]{0.49\linewidth}
        \centering
        \includegraphics[width=\linewidth]{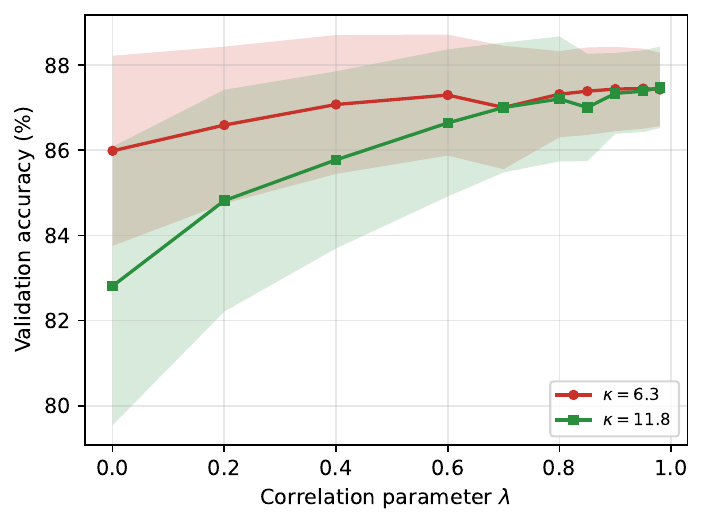}
        \caption{MNIST.}
    \end{subfigure}
    \caption{Held-out accuracy for the fixed-noise-multiplier experiments with $\eta=1$. The noise multiplier $\kappa$ is held fixed while the correlation parameter $\lambda$ varies. Curves show the mean over five random seeds, with shaded regions indicating one standard deviation.}
    \label{fig:fixed-kappa-eta1-acc}
\end{figure}

\paragraph{Results with a larger nominal step size.}
We further repeat the fixed-noise-multiplier experiments with the larger nominal step size $\eta=5$. We emphasize that Theorem~\ref{thm:lambda-positive-ntk-opt} constrains the effective step size $\eta_{\rm eff}=\rho_{\rm clip}\eta$, rather than the nominal step size $\eta$ itself. Hence, choosing $\eta=5$ is compatible with the theorem whenever the corresponding $\eta_{\rm eff}$ satisfies the stated condition.

Figure~\ref{fig:fixed-kappa-eta5} shows that the same qualitative behavior persists. For both datasets and both noise multipliers, increasing temporal correlation reduces the trajectory-averaged empirical loss, and the final-iterate loss follows the same trend. Compared with $\eta=1$, the dependence on $\lambda$ is more pronounced, particularly at the larger noise multiplier. The corresponding held-out accuracies exhibit a consistent improvement. These results provide additional evidence that the correlation-dependent reduction observed in the main text persists across different nominal step sizes.

\begin{figure*}[t]
    \centering

    \begin{subfigure}[t]{0.66\linewidth}
        \centering
        \includegraphics[width=\linewidth]{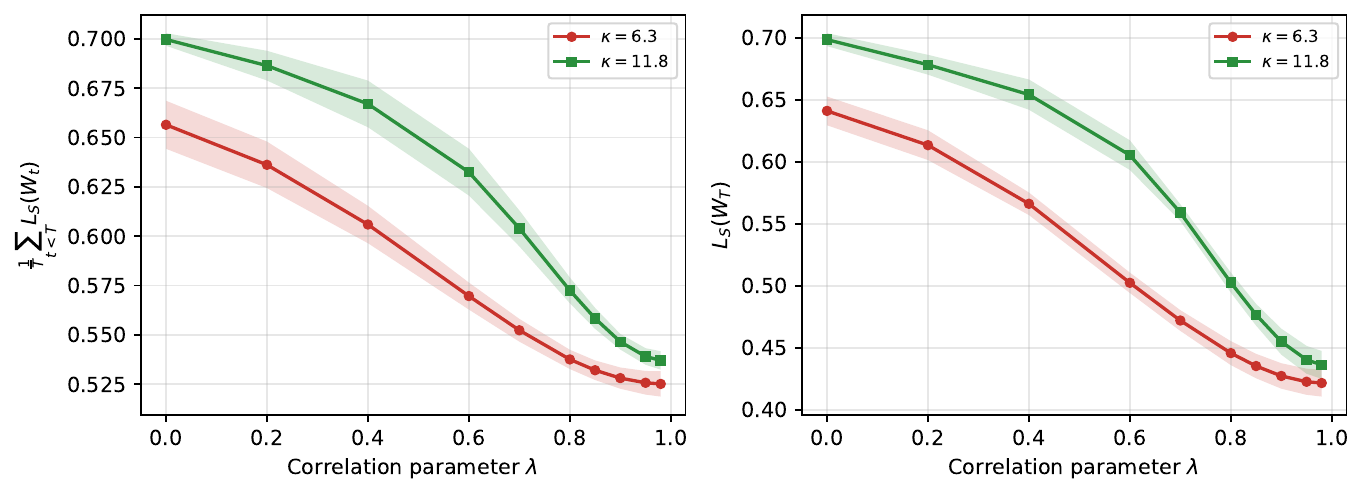}
    \end{subfigure}
    \hfill
    \begin{subfigure}[t]{0.32\linewidth}
        \centering
        \includegraphics[width=\linewidth]{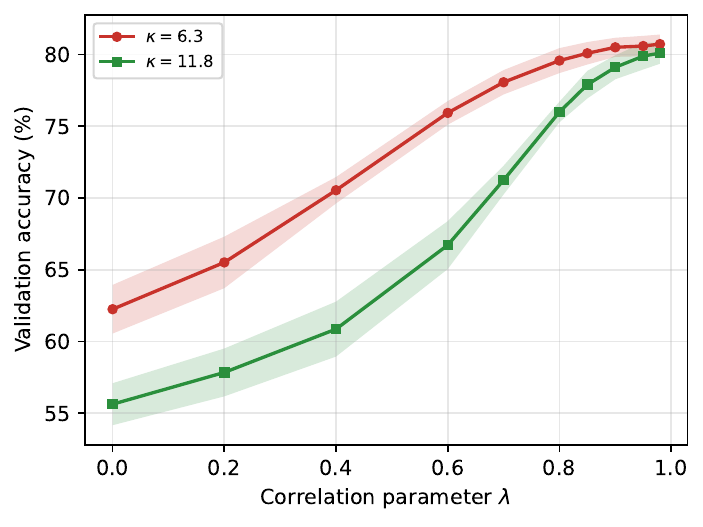}
    \end{subfigure}

    \vspace{-1mm}
    
    \centerline{\textbf{(a) Synthetic logistic data.}}

    \vspace{2mm}

    \begin{subfigure}[t]{0.66\linewidth}
        \centering
        \includegraphics[width=\linewidth]{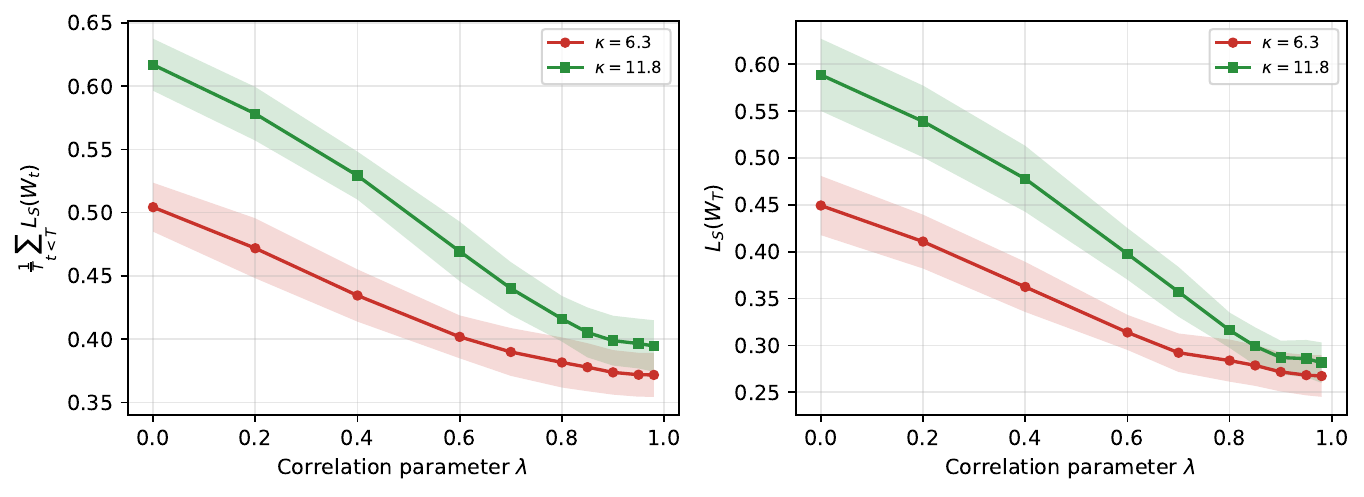}
    \end{subfigure}
    \hfill
    \begin{subfigure}[t]{0.32\linewidth}
        \centering
        \includegraphics[width=\linewidth]{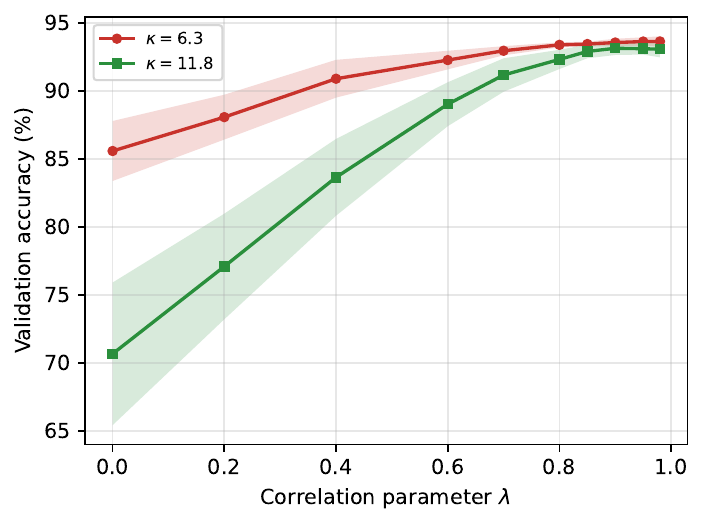}
    \end{subfigure}

    \vspace{-1mm}

    \centerline{\textbf{(b) MNIST.}}

    \caption{Additional fixed-noise-multiplier experiments with the larger nominal step size $\eta=5$. For each dataset, the first two panels report the trajectory-averaged empirical loss and the empirical loss of the final iterate, respectively, while the third panel reports the held-out accuracy. The noise multiplier $\kappa$ is held fixed while the correlation parameter $\lambda$ varies. Curves show the mean over five random seeds, with shaded regions indicating one standard deviation.}
    \label{fig:fixed-kappa-eta5}
\end{figure*}
\end{document}